\documentclass[sigconf]{acmart} 
\usepackage{algorithm}
\usepackage{algpseudocode}
\usepackage{multirow}
\usepackage{subfigure}
\usepackage{tikz}
\usepackage{amsmath}
\usepackage{graphicx} 

\newtheorem{theorem}{Theorem}[section]

\newtheorem{lemma}[theorem]{Lemma}
\newtheorem{definition}[theorem]{Definition}
\newtheorem{assumption}[theorem]{Assumption}

\makeatletter 
\renewcommand{\@thesubfigure}{\hskip\subfiglabelskip} % cancel the number of subfigures

\AtBeginDocument{%
  \providecommand\BibTeX{{%
    \normalfont B\kern-0.5em{\scshape i\kern-0.25em b}\kern-0.8em\TeX}}}

\setcopyright{none}
\copyrightyear{2025}
\acmYear{2025} 
 
\acmConference[ACM SIGKDD 2025]{31st ACM SIGKDD Conference on Knowledge Discovery and Data Mining}{2025}{Toronto, Canada}

 \acmConference{}
\begin{document}

% \title{A Similarity Matrix specified Low-Rank Matrix Completion}
%\title{Tailored Low-Rank Matrix Completion for Similarity Matrices}
\title{Tailored Low-Rank Matrix Factorization for Similarity Matrix Completion}

\author{Changyi Ma$^\dagger$}
\affiliation{%
  \institution{Centre for Artificial Intelligence and Robotics, Hong Kong Institute of Science \& Innovation, CAS,} 
  \city{Hong Kong SAR}
  \country{China}}
\email{changyi.ma@cair-cas.org.hk}

\author{Runsheng Yu$^\dagger$}
\affiliation{%
  %\institution{Department of Computer Science and Engineering,  Hong Kong University of Science and Technology}
  \institution{Independent Researcher}
  \city{}
  \country{}
}
\email{runshengyu@gmail.com}

\author{Xiao Chen}
\affiliation{%
  \institution{The Hong Kong Polytechnic University}
  \city{Hong Kong SAR}
  \country{China}
}
\email{shawn.chen@connect.polyu.hk}

\author{Youzhi Zhang$^*$}
\affiliation{%
 \institution{Centre for Artificial Intelligence and Robotics, Hong Kong Institute of Science \& Innovation, CAS,} 
  \city{Hong Kong SAR}
  \country{China}
 }  
\email{youzhi.zhang@cair-cas.org.hk}
\renewcommand{\shortauthors}{Changyi Ma, Runsheng Yu, Xiao Chen, \& Youzhi Zhang}

\renewcommand{\thefootnote}{}

\begin{abstract}   
Similarity matrix serves as a fundamental tool at the core of numerous downstream machine-learning tasks. However, missing data is inevitable and often results in an inaccurate similarity matrix. To address this issue, Similarity Matrix Completion (SMC) methods have been proposed, but they suffer from high computation complexity due to the Singular Value Decomposition (SVD) operation. To reduce the computation complexity, Matrix Factorization (MF) techniques are more explicit and frequently applied to provide a low-rank solution, but the exact low-rank optimal solution can not be guaranteed since it suffers from a non-convex structure. In this paper, we introduce a novel SMC framework that offers a more reliable and efficient solution. Specifically, beyond simply utilizing the unique Positive Semi-definiteness (PSD) property to guide the completion process, our approach further complements a carefully designed rank-minimization regularizer, aiming to achieve an optimal and low-rank solution. Based on the key insights that the underlying PSD property and Low-Rank property improve the SMC performance, we present two novel, scalable, and effective algorithms, SMCNN and SMCNmF, which investigate the PSD property to guide the estimation process and incorporate nonconvex low-rank regularizer to ensure the low-rank solution. Theoretical analysis ensures better estimation performance and convergence speed. Empirical results on real-world datasets demonstrate the superiority and efficiency of our proposed methods compared to various baseline methods. \footnotetext{$\dagger$ indicates these authors have contributed equally to the work. }
\footnotetext{* indicates the corresponding author.}

\end{abstract}

\begin{CCSXML}
<ccs2012>
   <concept>
       <concept_id>10010147.10010257</concept_id>
       <concept_desc>Computing methodologies~Machine learning</concept_desc>
       <concept_significance>500</concept_significance>
       </concept>
 </ccs2012>
\end{CCSXML}

\ccsdesc[500]{Computing methodologies~Machine learning} 
\keywords{Similarity Matrix Completion, Matrix Factorization, Positive Semi-definiteness (PSD), Low Rank}
 
\maketitle 

%matrix learning-- matrix completion---similarity matrix completion---previous cons+our pros
% matrix learning 

\section{Introduction}
\label{sec:introduction}    
%matrix->similarity matrix
%In numerous practical applications, data samples are commonly represented as a matrix, covering areas such as images in computer vision \cite{chaudhuri2023data,chen2021dair}, documents in natural language processing \cite{li2023text,he2023clur}, and user-item ratings matrix in recommender systems \cite{li2023text,wang2022spatial}. 
In numerous practical applications, data samples are commonly represented as a matrix and applied as a fundamental tool in a broad range of downstream tasks extensively, such as image retrieval \cite{lee2023revisiting,chaudhuri2023data}, document clustering \cite{he2023clur,ahmadi2022unsupervised}, and recommender system \cite{li2023text,wang2022spatial}. Among these applications, the essence is to apply an appropriate similarity metric \cite{hamers1989similarity} to construct the similarity matrix \cite{manz1995analysis,dehak2010cosine} where each entry represents the pairwise similarities between the given data samples.

%data missing
However, data missing is inevitable in practical scenarios, resulting in an incomplete data matrix, which in turn leads to inaccurate similarity scores (i.e., the entries of the similarity matrix) and further affect the performance of downstream search tasks \cite{ankerst1998improving,xie2022fast,cheng2013searching}. For instance, an image might be incompletely observed due to partial occlusion or corruption \cite{lee2023revisiting}, which leads to inaccurate image retrieval results. In collaborative filtering recommender systems \cite{schafer2007collaborative}, users typically provide limited ratings for movies, resulting in a user-movie rating matrix with missing entries \cite{li2023text,wang2022spatial}, which further leads to inaccurate recommendation results.

% MC + MF-MC
% To address this problem, the matrix completion (MC) methods are proposed to recover a low-rank matrix from a subset of the observed entries \cite{ankerst1998improving,xie2022fast,cheng2013searching}. 
% The main strategy are based on nuclear-norm minimization which requires computing singular value decompositions (SVD) \cite{stewart1993early}, which is increasingly costly as matrix sizes and ranks increase. To improve the capacity to solve large-scale problems, the Matrix Factorization (MF) \cite{ding2008convex,gillis2010using} technique has become popular. Specifically, instead of estimating the entire matrix, MF-MC aims to recover the matrix by estimating factorized submatrices whose rank is smaller than the original matrix \cite{srebro2004maximum,mnih2007probabilistic}. However, MF-MC approaches rely on non-convex optimization, which leads to the multiple local optima \cite{ding2008convex,gillis2010using}. In detail, MC methods approximate the missing entries based on the observed entries rather than fully conforming to the true underlying data properties. Moreover, accurately capturing the low-rank structure of the original matrix becomes more challenging when the observed entries are limited, as insufficient information can easily lead to higher-rank components in the estimated results \cite{yalccin2022factorization}.  

% SMC 
% In contrast, practical applications \cite{lee2023revisiting,li2023text,wang2022spatial} mainly focus on providing high-quality performance, which is core to preserving pairwise similarity even with missing data.

To alleviate the data missing problem and enhance performance in downstream tasks, Similarity Matrix Completion (SMC) methods \cite{li2020scalable,li2015estimating,ma2021fast,wong2017matrix} are proposed to estimate the similarity matrix, approximating the unknown similarity matrix calculated from the fully observed data samples. Previous SMC methods \cite{li2015estimating,li2020scalable} leverage the Positive Semi-definite (PSD) property and utilize singular value decomposition (SVD) \cite{stewart1993early} to extract underlying information through non-negative eigenvalues, which ensures the PSD property and further guides the SMC process. However, these methods may encounter computational bottlenecks when performing SVD on large-scale datasets. Furthermore, they ignore the low-rank property of the similarity matrix, which arises from the high correlation between rows or columns \cite{zhang2003optimal}. It may hinder the estimation process of SMC and limit the performance of the downstream tasks.  

To improve the computation efficiency, Matrix Factorization (MF) \cite{ding2008convex,gillis2010using} technique has become popular in the general matrix completion problem \cite{mnih2007probabilistic,fan2019factor} and offers a straightforward way of reducing the computation cost by reducing the rank of the matrix. Specifically, instead of estimating the missing values in the entire data matrix, the MF technique aims to reconstruct the matrix by approximating the factorized submatrices whose rank is smaller than the original matrix \cite{srebro2004maximum,mnih2007probabilistic}. However, MF techniques usually rely on non-convex optimization, which leads to multiple local optima \cite{ding2008convex,gillis2010using}. In detail, MF-based MC methods estimate the missing entries on the basis of the observed entries, rather than fully conforming to the true underlying data properties.

 % While MF techniques can streamline computation and facilitate matrix recovery, the reliance on non-convex optimization may introduce challenges related to convergence and the potential for suboptimal solutions due to the presence of multiple local optima.

%Moreover, accurately capturing the low-rank structure of the original matrix becomes more challenging when the observed entries are limited, as insufficient information can easily lead to higher-rank components in the estimated results \cite{yalccin2022factorization}.  }

% BMF
% The other approach to the MC problem can be formulated as semidefinite programming (SDP) \cite{nakata2003exploiting}. However, both the computational complexity and storage cost are high when the data size is extremely large. Then, Burer-Monteiro factorization (BMF) \cite{burer2003nonlinear} is proposed to reduce the computational cost, but the non-convexity nature may lead to spurious suboptimal solutions. Moreover, existing theoretical analysis supports the optimality guarantee of BMF, that is, the low-rank property can be achieved by selecting a sufficiently large rank $r$. However, it may falsely suggest that these methods can achieve the optimal low-rank solution with comparable performances in practice.  

%our methods
To provide a high-quality SMC estimation solution, we propose a novel framework for Similarity Matrix Completion (SMC). Specifically, we fully utilize the unique positive semi-definiteness (PSD) property of the similarity matrix \cite{aronszajn1950theory} thus providing effective estimation results with a theoretical guarantee in Theorem \ref{theorem:guarantee}. Moreover, we integrate the low-rank property \cite{hamers1989similarity} as the regularizer to further enforce the low-rank property, which can be further used to enhance validity in Theorem \ref{theorem:RMSE}. Subsequently, we introduce two novel, scalable, and effective algorithms, SMCNN and SMCNmF, to solve the SMC problem effectively and efficiently. Our contributions are three-fold:

$\bullet$ We propose a novel framework by leveraging the PSD and Low-Rank property to solve the SMC problem, which provides accurate and feasible estimation results with high efficiency. Specifically, we utilize the unique PSD property to ensure the effectiveness of the estimation process with reduced computation cost, which is accomplished by adopting the MF technique to achieve a low-rank solution. Additionally, beyond the low-rank property generated from MF technique, we formulate the low-rank property as the regularizer to further enhance the lower-rank property in our estimation, setting our SMC framework apart from previous methods.

$\bullet$ We provide a theoretical guarantee of accuracy and reliability on the estimation and convergence properties of our proposed novel SMC framework. On the one hand, Theorem \ref{theorem:RMSE} shows that the estimated similarity matrix is close to the unknown ground truth similarity matrix when the carefully designed PSD property is imposed as the constraint. On the other, Theorem \ref{theorem:guarantee} verifies the solution is optimal and to be lower-rank through a well-designed rank-minimization regularizer compared with the conventional MF techniques.

$\bullet$ We further conduct extensive experiments on various real-world datasets. Comparative analyses against existing methods consistently show that our SMCNN/SMCNmf algorithms outperform in similarity matrix completion performance while demonstrating improved efficiency.

\textbf{Notations}: Lowercase letters denote vectors, uppercase letters denote matrices, and $(.)^\top$ denotes the transpose operation. $I$ denotes the identity matrix. For a square matrix $S \in \mathbb{R}^{n \times n}$, $\text{tr}(S)$ is its trace, $\|S\|_F = \sqrt{(\text{tr}(S^\top S))}$ is its Frobenius norm. Consider a rank-$r$ matrix $S$ with SVD given by $U \text{Diag}(\sigma(S))V^\top$, where $U\in \mathbb{R}^{n\times r}$, $V\in \mathbb{R}^{n \times r}$, and $\sigma_i(S)$ denotes the $i$th vector of singular values in decreasing order, with $\sigma_i(S)$ denoting the $i$th singular value of $S$. The nuclear norm is defined as the sum of all singular values, $\|S\|_*=\sum_{i=1}^{n}\sigma_i(S)$. Without loss of generality, we assume $\sigma_1(S) \geq \sigma_2(S) \geq ... \geq \sigma_n(S)\geq 0$ denotes $S \succeq 0$ is a positive semi-definite matrix.  

%-----------------------------------------------------------------------------------------------
\section{Background}
\label{sec:background} 
To facilitate understanding the Similarity Matrix Completion (SMC) problem, we first formulate a motivating scenario. Subsequently, we present the Direct Missing Value Imputation (DMVI) method, which imputes the missing values in the data vector. We also introduce the Direct Incomplete Matrix Completion (DIMC) method, which fills the missing values in the entire incomplete observed matrix. Finally, an in-depth overview of various SMC methods is provided, along with the prerequisites for addressing the SMC problem.% Finally, we summarize the differences and highlight the novelty of our methods. The various MC methods are described in detail next, along with the prerequisites for the Similarity Matrix Completion (SMC) problem. We briefly discuss the matrix completion (MC) problem followed by an introduction to various methods for solving the MC problems. Then, we introduce the preliminaries of the similarity matrix completion (SMC) problem and illustrate the inspiration, motivation, and novelty of our proposed methods.  

\subsection{Motivating Scenario}
\label{subsec:pb}  
Incomplete images (i.e., partially occluded or corrupted images) are common in the real world. For instance, in 2003, all the Earth images of NASA contained missing data due to the sensor failure\footnote{$^1$ \url{https://earthobservatory.nasa.gov/features/GlobalLandSurvey/page2.php}}. Now, suppose you obtain an incomplete image $q'$ from $Q' \in \mathbb{R}^{n_\text{q}\times d}$ on Mars, and you want to find the most similar image $p$ from a massive observed search database $P\in \mathbb{R}^{n_\text{p} \times d}$ that is completely observed on Earth. However, directly calculating and sorting the similarity score $S_{\text{pq}}$ between the complete search candidate $p$ and incomplete query sample $q'$ \cite{kyselak2011stabilizing,zezula2006similarity} may lead to inaccurate similarity scores due to missing data. It motivates us to investigate and exploit similarity matrix properties, i.e., PSD and Low-Rank, to estimate and recover pairwise similarities to better accomplish the downstream tasks. This can be formulated as the similarity matrix completion (SMC) problem: 

\begin{equation}
\label{eq:obj}
    \min_{S} \| S- S^0\|_F^2 \quad  s.t.~~ S \succeq 0,~~\text{rank}(S)\leq r,
\end{equation}
where $S^0$ denotes the initial inaccurate similarity matrix, $S \succeq 0$ denotes PSD property, and $\text{rank}(S)$ denotes the expected rank of $S$ is $r$ that ensures the low-rank property. In general, SMC problem aims to find the $\hat{S}$ to approximate the unknown ground truth $S^*$ by using the PSD property $S \succeq 0$ and Low-Rank property $\text{rank}(S)\leq r$ to guide the optimization process.

Specifically, we adopt the widely used symmetric similarity metric in various studies \cite{melville2010recommender,deng2011hierarchical,amer2021enhancing,huang2008similarity}, Cosine Similarity \cite{hamers1989similarity}, to calculate the similarity matrix $S$. Each entry $S_{ij}$ is calculated as: \[ S_{ij}  = \frac{p_i^{\top} \cdot q_j}{\|p_i\| \cdot \|q_j\|},\] where the column vectors $p_i, q_j\in \mathbb{R}^{d}$ denotes the data vectors in $d$-dimensional space and $\|.\|$ denotes $\ell^2$-norm of a vector. When the data vector is complete, $S_{ij}$ is calculated directly. When the data vector is incomplete, the initial inaccurate $S^0_{ij}$ is calculated by using the overlapped features observed in $p_i$ and $q_j$ with common indices. (Detailed description in Appendix \ref{app:detailedSS})

%The entry $S_{ij}$ in the $i$-th row $j$-th column of the similarity matrix $S\in \mathbb{R}^{n\times n}$ denotes the pairwise similarity score between search candidate $p_i$ and query item $q_j$, where $n$ denotes the total number of data samples.

% \textcolor{red}{Therefore, we are motivated to estimate the incomplete observations. i) Estimating missing values and filling the original data matrix can be formulated as the matrix completion (MC) problem \cite{ankerst1998improving,xie2022fast,cheng2013searching}. ii) Estimating pairwise similarity scores and recovering the similarity matrix can be formulated as the similarity matrix completion (SMC) problem\cite{ma2021fast,wong2017matrix}. (refer to Appendix \ref{app:dpf})}

%Therefore, we are motivated to find an approach to estimate the similarity matrix, which can be formulated as the similarity matrix completion (SMC) problem \cite{ma2021fast,wong2017matrix}. 

\subsection{Direct Missing Value Imputation (DMVI)} 

Direct Missing Value Imputation (DMVI) methods \cite{lin2020missing,graham1997analysis} handle the missing values in raw data vectors, exploiting statistical principles to impute missing values in the raw data vectors by replacing them with surrogate values. For example, Mean/Zero imputation \cite{kim2004reuse} replaces the missing values with the mean value/zero for the continuous variables without considering the domain knowledge. Then, the model-based missing imputation methods \cite{tsuda2003algorithm,seber2012linear,pigott2001review}, such as Expectation-Maximization (EM) and Liner Regression (LR), are proposed to further improve the performance. In particular, these methods assume the variables with missing values are correlated and predict the unknown variables using other observed variables. Though the DMVI method is easy to implement and highly efficient, accuracy is the main bottleneck in practical applications \cite{acuna2004treatment}. 
 
%-------------------------------------------------------------------- 
\subsection{Direct Incomplete Matrix Completion (DIMC)} 
\label{subsec:mfmc}
%\footnote{*** plz rewrite this sec.}  

%\hl{A straightforward method to address this issue is to calibrate the}
 
%in various real-world applications, such as image retrieval \cite{cabral2013unifying,li2023lsdir}, natural language processing \cite{pfahler2017learning,pennington2014glove}, and recommender systems \cite{wang2015collaborative}, data samples are represented as a matrix $X\in \mathbb{R}^{n \times d}$ denoting $n$ data samples in $d$-dimensional space. However, the incomplete observations are ubiquitous \cite{ankerst1998improving,xie2022fast}, resulting in partially observed representations and potentially reducing the accuracy of downstream tasks.  
 
%\textcolor{blue}{can we follow the motivation scenario, and use $Q'$ as the incomplete data matrix?} 
Direct Incomplete Matrix Completion (DIMC) methods handle the missing values in the raw incomplete data matrix $X^0$, which is also known as general Matrix Completion (MC) \cite{ankerst1998improving,xie2022fast,cheng2013searching}. A widely used variation of the DIMC problem is to estimate $X$ with the smallest rank $r$ that matches the unknown complete matrix $X^*$ \cite{hu2021low,xu2016dynamic,wang2017unified,kuang2015symnmf}. The mathematical formulation is as follows:
\begin{equation}
\label{eq:MCprob}
    \min_{X} \| X - X^0\|_F^2 \quad s.t. \text{rank}(X)\leq r,
\end{equation}
where $X \in \mathbb{R}^{n \times d}$ denotes the data matrix with $n$ data samples in $d$-dimension space, $X^0$ denotes the incomplete observed data matrix, and $\text{rank}(.)$ is the rank constraint function with $\text{rank}(X)$ denotes the expected rank of the recovered matrix $X$ is less than or equal to $r$.

%\footnote{*** discuss the intuition behind this obj.}

%where MC aims to fill in the missing values and recover an optimal $\hat{X}$ based on the incomplete observed data matrix $X^0$.
%\footnote{***define $rank(X)\leq r$, and justify why. }. 

\subsubsection{Low-Rank Solution to MC} 
Problem \eqref{eq:MCprob} is generally NP-hard due to the combinational nature that requires finding a near-optimal solution with the rank constraint function \text{rank($\cdot$)}. However, under reasonable conditions \cite{recht2010guaranteed,udell2016generalized,zhou2014low}, the problem can be solved by solving a convex optimization problem via minimizing the nuclear norm $\|X\|_*$ to alternate the $\text{rank}(X)$ function, where the nuclear norm is the summation of the singular values of $X$. Though the nuclear norm is the only surrogate of matrix rank, it enjoys the strong theoretical guarantee that the underlying matrix can be exactly recovered \cite{candes2012exact}. One of the widely used approaches is the fixed-point shrinkage algorithm \cite{ma2011fixed} which solves the regularized linear least problem. The other popular method is truncated-SVD \cite{xu1998truncated,falini2022review}, which involves a selection of singular values of the incomplete observed matrix at each iteration. However, both approaches involve singular value decomposition (SVD), which incurs a high computational cost and becomes increasingly expensive as the size or rank of the underlying matrix increases. Therefore, it is desirable to develop an alternative approach that is better suited to solving large-scale problems.

\subsubsection{More Efficient Technique: Matrix Factorization (MF)}
To improve computational efficiency, the Matrix Factorization (MF) \cite{mnih2007probabilistic,fan2019factor} technique has received a significant amount of attention and is widely used to solve the low-rank MC problem. Specifically, it factorizes the entire matrix into two low-rank sub-matrices by investigating the latent factors in the original matrix and capturing the underlying structure to provide reasonable estimation results more efficiently. 

%\footnote{*** plz discuss the $UV^\top$.} change to WH, conflict to S=VV'
\begin{equation}
\label{eq:problem_def_mfmc}
    \begin{aligned}
        &\min_{X} \|WH^\top-X^0\|_F^2,
    \end{aligned}
\end{equation}
where $X=WH^\top$, $W\in \mathbb{R}^{n\times r}$, $H\in \mathbb{R}^{d\times r}$, and $r < \min(n,d)$. Usually, $r$ represents the total number of latent factors, which can also be regarded as a smaller rank to reduce computational cost.

It is obvious MF technique reduces the computational complexity, but the non-convex nature of MF poses challenges in ensuring the convergence of the optimal and unique solution \cite{ding2008convex,gillis2010using}, which means it may easily fall into a local minimum as the global optimal solution. Various optimization techniques, stochastic gradient descent (SGD) \cite{gemulla2011large,he2016fast} and non-negative matrix factorization (NMF) \cite{ding2008convex,gillis2020nonnegative}, are proposed to address the non-convexity and try to find better solutions. However, due to the inherent difficulty of non-convex optimization problems, these methods may not always guarantee a globally optimal solution \cite{yao2018large, hu2021low}. It motivates us to further study the non-convexity property along with the internal matrix property to find a more stable and efficient solution.

\subsection{Similarity Matrix Completion (SMC)}
\label{subsec:SMC}
% In practical applications \cite{lee2023revisiting,li2023text,wang2022spatial}, the similarity matrix is the core of providing high-quality search results and usually relies on carefully designed similarity metrics to calculate pairwise similarities.
%\footnote{*** i rewrite, but i think it could be better.}

Different from the DIMC methods imputing the missing values in the raw data matrix, similarity matrix completion (SMC) \cite{ma2021fast,ma2024fast,wong2017matrix} provides a more comprehensive approach working on the similarity matrix. SMC methods study the properties of similarity matrices that preserve pairwise similarities in the presence of incomplete information and are more suitable for downstream search tasks. As a consequence, SMC is commonly expressed as: 

\begin{equation}
\label{eq:problem_def_smc}
    \begin{aligned}
         \min_{S} \|S-S^0\|_F^2 \quad\quad s.t.~~S \succeq 0.  %\text{rank}(S)\leq r,
    \end{aligned}
\end{equation} 
%where SMC aims to find an estimation similarity matrix $\hat{S}$ based on the initial inaccurate similarity matrix $S^0$. 

Here, similarity matrix $S\in \mathbb{R}^{n \times n}$ represents the pairwise similarities, where $n$ denotes the number of data samples, $S_{ij}$ denotes the similarity score of pairwise data samples, and $\|.\|_F^2$ denotes the squared Frobenius norm. $S \succeq 0$ indicates that the real symmetric similarity matrix $S$ has the PSD property, which always holds when any symmetric similarity measure (i.e., cosine similarity) is adopted. Given an initial inaccurate similarity matrix $S^0$, SMC makes use of the PSD property ($S \succeq 0$) to find the optimal estimated $\hat{S}\in \mathbb{R}^{n\times n}$ to approximate the unknown ground truth $S^*$.

%  Note that, $S \succeq 0$ indicates that the real symmetric similarity matrix $S$ lies on PSD property, which always holds when cosine similarity is set as the similarity metric. Meanwhile, each entry $S^0_{ij}$ in the initial similarity matrix $S^0$ is calculated by the overlapped features observed in two data samples $i$ and $j$ with common indices. 

\subsubsection{Main Solution to SMC}
A similarity matrix $S$ is symmetric (i.e., $S=S^\top$) if the similarity metric is symmetric \cite{li2004similarity,blondel2004measure}, i.e., Cosine Similarity \cite{hamers1989similarity}. Besides, the similarity matrix further owns the Positive Semi-Definiteness (PSD) property \cite{vapnik1999nature,bouchard2013proof,higham1988computing} (i.e., $S \succeq 0$), where all the eigenvalues are non-negative \cite{hill1987cone,higham1988computing}. Therefore, the previous work of SMC \cite{li2020scalable,li2015estimating,ma2021fast,wong2017matrix} mostly regards the similarity matrix calculated by a symmetric similarity metric that satisfies both the symmetric and PSD algebraic variety. Specifically, a series of SMC methods \cite{li2015estimating,li2020scalable} leverage the PSD property and employ singular value decomposition (SVD) \cite{stewart1993early} to extract underlying information through non-negative eigenvalues and further ensure the PSD property. However, these methods may encounter computational bottlenecks when performing SVD on large-scale datasets. 

\subsubsection{More Scalable Approach}
To reduce the computational cost of SVD while ensuring the SMC performance \cite{ma2024fast}, the learning-based method is proposed to approximate the SVD operation while reducing the computational complexity. Though these methods fully exploit the PSD property and propose a series of algorithms to provide an effective estimation, the internal low-rank properties arising from the high correlation between rows or columns \cite{zhang2003optimal} are still rarely studied.

%Similarity matrix completion (SMC) problem \cite{li2020scalable,li2015estimating,ma2021fast,wong2017matrix} is proposed to estimate and calibrate the initial similarity matrix to approximate the unknown ground-truth similarity matrix. Previous SMC methods \cite{li2015estimating,li2020scalable} leverage the PSD property and employ singular value decomposition (SVD) \cite{stewart1993early} to extract underlying information through non-negative eigenvalues and further ensure the PSD property. However, these methods may encounter computational bottlenecks when performing SVD on large-scale datasets. Moreover, these methods ignore the low-rank property of the similarity matrix, which arises from the high correlation between rows or columns \cite{zhang2003optimal}. 

\section{Methodology}
\label{sec:method}

\subsection{Novel SMC Framework}
%\hl{move it to methodology. this is our framework not bg.} 
Following the previous SMC work, we not only investigate the PSD property but also exploit the low-rank property of the similarity matrix, which can both improve the SMC performance. Unlike the low-rank property in DIMC which comes from partially observed data samples, the low-rank property in SMC arises from the high correlation of rows/columns. It motivates us to delve deeper into the problem \textit{Is it possible to complete an inaccurate similarity matrix if its rank, i.e., its maximum number of linearly independent row/column vectors, is small? } Therefore, the estimated similarity matrix is expected to belong to the symmetric and PSD algebraic variety, and further lies on the Low-Rank property, which calls back to the motivating scenario in subsection \ref{subsec:pb}.

\subsection{Low-Rank via Cholesky Factorization} 
\label{subsec:MFSM}
As discussed in subsection \ref{subsec:SMC}, the similarity matrix is symmetric, i.e., $S=S^\top$, and enjoys PSD property, i.e., $S \succeq 0$. Therefore, the Cholesky Factorization (CF) \cite{higham1990analysis,johnson1985matrix,schabauer2010toward}, $S = VV^\top$, can be adopted by constraining the search space to factorized sub-matrix of rank at most $r$. It enforces PSD properties mechanically: 

\begin{equation}
\label{eq:problem_def_PSD}
    \min_{V}  \|VV^\top - S^0\|_F^2. %R+ \lambda %R(VV^\top).
\end{equation} 

This is equivalent to the Problem \eqref{eq:problem_def_smc} with the additional underlying constraint $\text{rank}(S) \leq r$, $V\in \mathbb{R}^{n \times r}$ is a real factorized sub-matrix, and $r<n$ is given as the small rank, which reduces the number of variables from $O(n^2)$ to $O(nr)$.  
%Typically, $S = VV^\top$, $V\in \mathbb{R}^{n\times r}$ is a real lower factorized sub-matrix with positive diagonal entries, and $r<n$ is given as the smaller rank and designed to reduce the computational cost. %.\footnote{*** $k< n$ is better.}. 

Different from the previous work that uses SVD on the entire similarity matrix to learn the underlying similarity matrix features with high computational complexity \cite{li2020scalable,li2015estimating}. Specifically, we fully utilize the PSD property of the symmetric similarity matrix $S$ and adopt Cholesky Factorization \cite{higham1990analysis}, a typical example of matrix factorization (MF) \cite{mnih2007probabilistic}, and accomplish the SMC problem by estimating the factorized matrices. Then, a series of learning-based algorithms \cite{razin2020implicit,wang2017unified} can be adopted to estimate the factorized matrices to achieve effective and efficient results. However, the MF techniques involve the non-convex optimization problems and often find a local minimum, i.e., the entire matrix estimated is not guaranteed to be optimal, which is comprehensively discussed in previous work \cite{gillis2020nonnegative,sun2016guaranteed,xu2017globally}.

\subsection{Lower-Rank via Nuclear Norm as Regularizer}
\subsubsection{Limitation/disadvantage of MF technique}
% \footnote{*** make the sec better.}
% \footnote{***Discuss how to above therom helps to guide as add lower rank..}
% \textcolor{red}{I prefer to put the limitations/disadvantages as subsubsec into this subsec. Otherwise, the subsec is short and will repeat again and again. How do you think?}
% \hl{agreed.}

As discussed in subsection \ref{subsec:mfmc}, MF is a typical non-convex optimization problem and easily gets stuck in sub-optimal \cite{gillis2020nonnegative,sun2016guaranteed,xu2017globally}. We note that the SMC problem lies in a similar property:
%note that this is also true for the SMC problem:  \footnote{*** the logic is strange. rank(s)>r implies global opt, which does not mean that rank(s)=r can not be global opt...  }

% Though the MF-MC helps to find the low-rank factorized matrix $\hat{V}$, the low-rank property \footnote{*** should be optimal condition} of the estimated similarity matrix $\hat{S}$ cannot be guaranteed. 

\begin{theorem}
\label{theorem:guarantee} 
     If the matrix factorization of the similarity matrix converges to the solution $S$ such that $\text{rank}(S) < r$, then it is also the optimal solution for (\ref{eq:problem_def_PSD}).\footnote{$^2$ All proofs are in the Appendix \ref{App:proof}.} 
\end{theorem} 
 
Theorem \ref{theorem:guarantee} shows that if the rank of $S$ is strictly less than $r$, all the local minimum is the global minimum. We define this as the \textit{Lower-Rank Matrix Property}. Note that, though the Problem \eqref{eq:problem_def_PSD} lies in the framework of Burer-Monteiro factorization (BMF) \cite{boumal2016non,journee2010low,cifuentes2019burer}, the optimality can be ensured by the \textit{Lower-Rank Matrix Property}.

\begin{definition}
\label{definition:lowerrank}
(Lower-Rank Matrix Property) 
When %performing Cholesky factorization, a special case of matrix factorization, on a positive semi-definite matrix 
$S=VV^\top$, where $S\in \mathbb{R}^{n \times n}$, $V\in \mathbb{R}^{n\times r}$, and $r<n$. The rank of $S$ is strictly smaller than $r$, i.e., $\text{rank}(S) < r$.
\end{definition}

%\textcolor{red}{?Theorem 3.1 here or after problem (8)? \hl{definitely here. we have no guarantee (8) has lower rank theoretically. }}

%Inspired by Theorem \ref{theorem:guarantee}, the lower-rank property \textcolor{red}{of the solution of problem \eqref{eq:problem_def_PSD}} can be directly guaranteed by solving the following objective: 

%we add a constraint term to ensure the lower-rank property. The approach involves the addition of a rank-minimization term $R(S)$ as the regularizer to achieve the low-rank solutions \cite{candes2012exact}. 

% \begin{equation}
% \label{eq:problem_def_smc_lower_rank}
%     \begin{aligned}
%          \min_{S} \|S-S^0\|_F^2 \quad\quad s.t.~~ rank(S)< r, 
%     \end{aligned}
% \end{equation} 

%\textcolor{red}{To ensure it is followed by (5)}, Can i use (7) not (6)? \hl{yes} 
Given the Definition \ref{definition:lowerrank} of \textit{Lower-Rank Matrix Property}, we design a novel objective to ensure the lower-rank property of the estimated similarity matrix $\hat{S}$ as the solution to problem \eqref{eq:problem_def_PSD}:

\begin{equation} 
    \begin{aligned}
         \min_{V} \|VV^\top-S^0\|_F^2 \quad\quad s.t.~~ \text{rank}(VV^\top)< r, 
    \end{aligned}
\end{equation}  
where the constraint $\text{rank}(VV^\top) < r$ aims to ensure the \textit{Lower-Rank Matrix Property} of the estimated similarity matrix $\hat{S}=\hat{VV^\top}$ and further guarantee the optimality by Theorem \ref{theorem:guarantee}. However, directly solving the rank minimization function is a computationally challenging NP-hard problem, and for most practical problems there is no efficient algorithm that yields an exact solution \cite{ma2011fixed,lu2015nonconvex}. To solve the problem smoothly \cite{bhojanapalli2018smoothed,recht2010guaranteed,chandrasekaran2011rank,koltchinskii2011nuclear,lu2015nonconvex}, we incorporate the rank constraint into the objective function as a penalty term, weighted by the weighted parameter $\lambda > 0$. Now, the objective is converted to: %\footnote{*** too vague. plz discuss why it is hard or add some refs....}  

\begin{equation}
\label{eq:problem_def_smc_lower_rank}
    \begin{aligned}
         \min_{V} \|VV^\top-S^0\|_F^2 + \lambda \text{rank}(VV^\top) \quad  s.t.~~ V\in \mathbb{R}^{n\times r},
    \end{aligned}
\end{equation} 
%\footnote{\textcolor{red}{ \hl{i slight change the eq, plz discuss why the conversion is true!!!! in details. add refs to illustrate this is a common relaxation.}}} 
where the lower-rank matrix property of $\hat{S}=\hat{V}\hat{V}^\top$ is ensured by the ensuring $\text{rank}(\hat{V})\leq r$. In other words, Problem \eqref{eq:problem_def_smc_lower_rank} aims to balance the goodness of estimation via the squared loss expressed as the first term and the lower-rank property of the rank function as the second term.

%\textcolor{red}{Penalty term: why not kkt? kkt cannot ensure the optimal, and gap , linear independence . why relax ? and why converge. why solution is consistent , more [10] reference }

However, solving Problem (\ref{eq:problem_def_smc_lower_rank}) is still computationally intractable with $\text{rank}(\cdot)$ function as the regularizer \cite{srebro2004maximum,lu2015nonconvex,wang2017unified}, which means few efficient algorithms can optimally solve all instances of the problem in polynomial time. A popular approach adopts a convex relaxation of the nuclear norm to obtain a computationally feasible solution, thereby achieving a lower-rank property \cite{candes2012exact}. A widely used convex surrogate is the nuclear norm \cite{lu2015nonconvex,wang2017unified}, which is shown in the following convex objective: 
%\textcolor{red}{Should $ s.t.~~ V\in \mathbb{R}^{n\times r},$ exist?}
\begin{equation}
\label{eq:problem_def_PSDLoR}
    \begin{aligned}
         \min_{V} \|VV^\top - S^0\|_F^2 + \lambda \|VV^\top\|_*  \quad\quad s.t.~~ V\in \mathbb{R}^{n\times r},
    \end{aligned}
\end{equation} 
where the regularizer $R(V) = \|VV^\top\|_*$ denotes the nuclear norm of matrix $VV^\top$, i.e., the sum of singular values \cite{recht2010guaranteed,jaggi2010simple} given by the SVD \cite{stewart1993early}, which is the convex surrogate for the rank minimization function.

\subsection{Similarity Matrix Completion with Frobenius Norm (SMCNN)}
Since calculating the nuclear norm of a matrix is the sum of its singular values by taking the SVD of the matrix and summing the absolute values of its singular values \cite{recht2010guaranteed,jaggi2010simple}, leading to a high computational cost. To avoid this, Lemma \ref{lemma:lemma1} provides the tight upper bound of the nuclear norm to obtain an efficient and comparable result \cite{srebro2004maximum}:  
\begin{lemma}
\label{lemma:lemma1} 
    $\|V V^\top \|_* \leq \|V\|_F^2$. 
\end{lemma}

Lemma \ref{lemma:lemma1} tells that the nuclear norm $\|VV^\top\|_*$ is upper bounded by $ \|V\|_F^2$.
Thus, we can use $ \|V\|_F^2$ as a surrogate regularizer to avoid the high computation of the nuclear norm. As a result, Problem \eqref{eq:problem_def_PSDLoR} can be reformulated as: 
% Recall that the calculating the nuclear norm requires the SVD operation, which is computationally expensive \footnote{*** cite the sec you mentioned it}. }. Since the F-norm.

% As can be seen, by using Lemma \ref{lemma:lemma1}, the SVD operation of summing the singular values of the nuclear norm $\|VV^\top\|_*$ is reduced to the Frobenius norm of the matrix $\|V\|^2_F$. %This calculation is performed without explicitly calculating the singular values. 

\begin{equation}
\label{eq:problem_def_NN1}
\begin{aligned}
    \min_{V} \|V V^\top -S^0\|_F^2 + \lambda \|V\|_F^2  \quad\quad s.t.~~ V\in \mathbb{R}^{n\times r},
\end{aligned}
\end{equation}
where $S=VV^\top$, $S_{ij}=\{VV^\top\}_{ij}$, regularizer $R(V)=\|V\|_F^2$, $\lambda > 0$ is the weighted parameter. Note that though the MF technique  \cite{sun2016guaranteed,chi2019nonconvex,jain2013low,hardt2014understanding} claims that the weighted parameters $\lambda$ of the factorized sub-matrix can be set to $0$, which is not suitable for the SMC problem. The former adds the regularization term to achieve a smooth function with simple gradients, whereas the latter aims to ensure a lower-rank solution than conventional MF-based MC methods, which is also the clear difference between conventional MC and SMC. 

\begin{algorithm}[!htb]
\caption{Similarity Matrix Completion with Nuclear Norm (\textbf{SMCNN})}
\label{alg:alg1} 
  \begin{algorithmic}[1] 
    \Require 
        $V^0 \in \mathbb{R}^{n\times r}$: Randomly initialized factorized matrix $V^0$ with rank $r$;  
        $\gamma$: Stepsize; 
        $\lambda$: Weighted parameter. 
        $T$: Total iterations.
    \Ensure 
        $\hat{V}\in \mathbb{R}^{n \times r}$: Estimated factorized matrix;
        $\hat{S}\in \mathbb{R}^{n \times n}$: Estimated similarity matrix $\hat{S}=\hat{V}\hat{V}^\top$. 
    \State Define $\min_V F_1(V) \equiv \min_{V} \|V V^\top -S^0\|_F^2 + \lambda \|V \|_F^2$;
%    \State Calculate the gradient w.r.t. $V$ $\nabla_V F_1(V)$;
     \For {t = 1, ..., T}
     \State Calculate gradient $\nabla_V F_1(V)$;
     \State Update $V^t = V^{t-1} - \gamma \nabla_V F_1(V^{t-1})$;
     \EndFor
    \State Calculate $\hat{S} = \hat{V} \hat{V}^\top$; 
     \State \Return $ \hat{S}$. 
  \end{algorithmic} 
\end{algorithm}

%we estimate the $\hat{V}$ to generate the similarity matrix $\hat{S}$, we calculate the gradient w.r.t $V$, where the gradient $\nabla F_1(V)$ is $\nabla_V F_1(V) = 4(VV^\top - S^0) V+2\lambda V$.

%\begin{theorem}
%\label{corollary:samesolution}
    % If Algorithm 1 converges to a solution  $V^*$ such that $\text{rank}(V)^* < r$, then it is also the optimal solution for (\ref{eq:problem_def_NN}).
%        If Algorithm 1 converges to a solution  $S'$ such that $\text{rank}(S') < r$, then it is also the optimal solution for (\ref{eq:problem_def_NN}).
%\end{theorem}

%Theorem \ref{corollary:samesolution} ensures the convergence of Algorithm.\ref{alg:alg1} as the condition for optimality guarantee.

%The Theorem \ref{corollary:samesolution} assures the solution to problem \eqref{eq:problem_def_NN} is equal to  problem \eqref{eq:problem_def_NN1}.

We further provide a general and efficient gradient-based algorithm to solve Problem \eqref{eq:problem_def_NN1} and summarize the entire process in Algorithm \ref{alg:alg1}. 

Line 1 shows the objective function $F_1(V)$. Lines 2-5 show the learning process of estimating the low-rank factorized sub-matrix $\hat{V}$ factorized on similarity matrix $S$. Lines 6-7 show the calculation and return the estimated similarity matrix $\hat{S}$. 

Recall that in \cite{stewart1993early}, the main computation complexity $O(n^3)$ is generated by the SVD operation. The SMCNN only needs to estimate the factorized matrix $\hat{V} \in \mathbb{R}^{n \times r}$ by a general solver, where the computation cost is $O(n^2r)$ instead of $O(n^3$) in each iteration, as well as the storage cost is $O(nr)$ instead of $O(n^2)$. As a result, both the computation cost and storage cost are significantly reduced since $r<n$.

%adopts the Cholesky decomposition \cite{gill1996stability} to optimize the factorized matrix $\hat{V} \in \mathbb{R}^{n \times r}$ by a general solver, where the computational cost is $O((n^2r)T)$ instead of $O(n^3$) as well as the storage cost is $O(nr)$ instead of $O(n^2)$. Both the computational cost and storage cost are significantly reduced since $r<n$.

%To satisfy the constraint that $0< \{VV^\top\}_{ij} \leq 1$, Line 6 shows the clip operation to fix all the elements in (0,1], with a relatively small $\epsilon$. To avoid changing the similarity score calculated by the complete observations, i.e., the similarity matrix between all the search candidates, we further utilize the projection function $P_\Omega$ in Line 7. To be more concise, a projection function $P_\Omega$ is proposed, i.e., $P_\Omega(\hat{S}-S^0)$. Here, $\Omega \in \{0,1\}^{n\times n}$ recorded positions of the observations, with $\Omega_{ij}=1$ if $S_{ij}$ is calculated by complete observations, and $\Omega_{ij}=0$ is calculated by the incomplete observations. $P_{\Omega}(.)$ is a projection operator such that $[P_{\Omega}(S)]_{ij}= S^0_{ij}$ if $\Omega_{ij} = 1$ and $[P_{\Omega}(S)]_{ij}=\hat{S}_{ij}$ otherwise $\Omega_{ij} = 0$. 

\subsection{Similarity Matrix Completion with Nuclear Norm minus Frobenius Norm (SMCNmF)}     
\label{sec:smcnmf}
% \hl{perhaps we need to delete that....} 
When conducting the vanilla nuclear norm regularizer, $R(V) = \|VV^\top\|_*$, it penalizes each singular value by $\lambda$ until it reaches 0, which performs relatively badly in empirical studies \cite{yao2018large,lu2015generalized,lu2015nonconvex}. Instead, the recent work \cite{wang2021scalable} reveals that if a regularizer can satisfy \textbf{i) (adaptivity)} the singular value can be penalized adaptively w.r.t the value, i.e., larger singular values penalized less, and \textbf{ii) (shrinkage)} the singular value is always larger than the corresponding singular value implied by the nuclear norm regularizer, it can achieve better performance.

Inspired by these analytic observations \cite{wang2021scalable}, we introduce a new regularizer $R(V) = \|VV^\top\|_* - \|VV^\top\|_F$ that satisfies both adaptivity and shrinkage. Specifically, the nuclear norm regularizer $\|VV^\top\|_*$ penalizes all singular values as the same shrinkage, whereas the $R(V)$ enforces different amounts of shrinkage depending on the magnitude of the Frobenius norm $\|VV^\top\|_F$. Then, we introduce an algorithm to obtain the estimated similarity matrix efficiently, where this regularizer can also be relaxed by the Lemma \ref{lemma:lemma1} with a sound theoretical guarantee. Now, the Problem \eqref{eq:problem_def_PSDLoR} can be transferred to:
\begin{equation}
\label{eq:problem_def_NF1}
\begin{aligned}
    \min_{V} \|V V^\top -S^0\|_F^2 + \lambda (\|V\|_F^2 - \|VV^\top\|_F) \quad s.t.~~ V\in \mathbb{R}^{n\times r},
\end{aligned}
\end{equation} 
where $S=VV^\top$, $S_{ij}=\{VV^\top\}_{ij}$, regularizer $R(V)=\|V\|_F^2 - \|VV^\top\|_F$, $\lambda > 0$ is the weighted parameter. Note that, $\lambda$ is the weighted parameter of low-rank regularizer $R(V)$, which is set the same as in SMCNN to ensure low-rank property and no need to be different for $\|V\|_F^2$ or $\|VV^\top\|_F$.

%\textcolor{red}{should we also highlight the single $\lambda$? \hl{yes, it is a single $\lambda$, so i think we do not need to adjust multi $\lambda$s.}} 

%Similar to problem \eqref{eq:problem_def_NN}, problem \eqref{eq:problem_def_NF1} can be reformulated as:
%\begin{equation}
%\label{eq:problem_def_NF2}
%\begin{aligned}
%    \min_V F_2(V) = & \min_{V} \|V V^\top -S^0\|_F^2 + \lambda \|V \|_F^2 - \lambda \|VV^\top\|_F, \\
   % &s.t.~~ 0< \{V V^\top\}_{ij} \leq 1. 
%\end{aligned}
%\end{equation}

%we estimated the $\hat{V}$ to generate similarity matrix $\hat{S}$, we calculated the gradient w.r.t $V$, where the gradient $ \nabla_V F_2(V) = 4(VV^\top - S^0) V + 2\lambda V - 2 \lambda \frac{V(V^\top V)}{\|VV^\top\|_F}$. 
Algorithm \ref{alg:alg2} shows the entire process to solve the Problem \eqref{eq:problem_def_NF1} via a general and efficient gradient-based algorithm.

\begin{algorithm}[!htb]
\caption{Similarity Matrix Completion with Nuclear Norm minus Frobenius Norm (\textbf{SMCNmF})} 
\label{alg:alg2} 
  \begin{algorithmic}[1] 
    \Require 
        $V^0 \in \mathbb{R}^{n \times r}$ : Randomly initialized factorized matrix $V^0$ with rank $r$; 
        $\gamma$: Stepsize; 
        $\lambda$: Weighted parameter;
        $T$: Total iterations. 
    \Ensure 
        $\hat{V} \in \mathbb{R}^{n \times r}$: Estimated factorized matrix;
        $\hat{S}\in \mathbb{R}^{n \times n}$: Estimated similarity matrix $\hat{S}=\hat{V}\hat{V}^\top$. 
    \State Define $\min_V F_2(V) \equiv \min_{V} \|V V^\top -S^0\|_F^2 + \lambda (\|V\|_F^2 - \|VV^\top\|_F)$;
     \For {t = 1,..,T}
     \State Calculate gradient $\nabla_V F_2(V)$;
     \State Update $V^t = V^{t-1}-\gamma \nabla_V F_2(V^{t-1})$; 
     \EndFor
    \State Calculate $\hat{S} = \hat{V} \hat{V}^\top $;
     \State \Return $ \hat{S}$.
%     \State Restrict $0< \{V V^\top\}_{ij} \leq 1$ via $\hat{S} = \min(\max(\hat{S}, \epsilon), 1)$;
%    \State \Return $ \hat{S} = P_{\Omega}(\hat{S}-S^0)$;  
  \end{algorithmic} 
\end{algorithm} 

Line 1 shows objective function $F_2(V)$. Lines 2-5 show the learning process of estimating the low-rank factorized sub-matrix $\hat{V}$. Lines 6-7 show the calculation and return the estimated similarity matrix $\hat{S}$. Similar to Algorithm \ref{alg:alg1}, in each iteration, the computation cost is $O(n^2r)$ and storage cost is $O(nr)$, which are also reduced significantly by optimizing the factorized submatrix $\hat{V}$.\footnote{$^3$ Detailed description of algorithm and complexity analysis in Appendix \ref{app:complexity}.}

%-----------------------------------------------------------------------------------------
\section{Theoretical Analysis}
\subsection{Complexity Analysis} 
\label{sec:analysis}  
Table \ref{tab:Complexity} shows the complexity analysis of the computation complexity and storage complexity of various matrix completion methods. % Recall the problem formulation in Section \ref{subsec:pb}, $\|\Omega\|_0 = 2n_pn_q+n_q^2$, $n^2=(n_p+n_q)^2$. Obviously, $\frac{\|\Omega\|_0}{n^2} < 1$. 
 
\begin{table}[!htb]
\centering
\huge 
\caption{Complexity Analysis on Matrix Completion Methods, $r_k$ (usually $r_k \geq r$) denotes an estimated rank at the $k$th iteration of SVD, $y$ denotes the number of subsets of $S$, $n_\text{p}$ denotes the number of search candidates, $n_\text{q}$ denotes the number of query items, $n=n_\text{p}+n_\text{q}$ denotes the total number of data samples, $T$ denotes the total number of iterations.}
\label{tab:Complexity} 
\resizebox{1\columnwidth}{!}{  
\begin{tabular}{l|l|l|l}
\toprule 
\hline
\textbf{Property} & \textbf{Method}                        & \textbf{\begin{tabular}[c]{@{}l@{}}Computation\\ Complexity\end{tabular}} & \textbf{\begin{tabular}[c]{@{}l@{}}Storage \\ Complexity\end{tabular}} \\ \hline
\multicolumn{1}{l|}{\multirow{2}{*}{\textbf{MF}}}   & \textbf{FGSR}\cite{fan2019factor}         & $O((n^2r)T)$ & $O(2nr)$  \\
\multicolumn{1}{l|}{}                               & \textbf{BPMF}\cite{mnih2007probabilistic} & $O((n^2r^2 + nr^3)T)$   & $O(2nr)$  \\  \hline
                                                
\multicolumn{1}{l|}{\multirow{3}{*}{\textbf{PSD}}}  & \textbf{CMC}\cite{li2020scalable}         & $O((n^3/y^2)T)$                                                           & $O(n^2)$  \\
\multicolumn{1}{l|}{}                               & \textbf{DMC}\cite{li2015estimating}       & $O((n^3)T)$  & $O(n^2)$  \\
\multicolumn{1}{l|}{}                               & \textbf{FMC}\cite{ma2024fast}             & $O(n_\text{p}^3+(n^2+\log n)T)$                                           & $O(n^2)$  \\ \hline
\multicolumn{1}{l|}{\multirow{2}{*}{\textbf{LoR}}}  & \textbf{NNFN}\cite{wang2021scalable}      & $O((n^2r_k)T)$                                                            & $O(n^2)$  \\
\multicolumn{1}{l|}{}                               & \textbf{facNNFN}\cite{wang2021scalable}   & 
%$O((n^2r+nr^2)T)$
$O(n^2rT)$
& $O(2nr)$  \\ \hline
\multirow{2}{*}{\textbf{\begin{tabular}[c]{@{}l@{}}PSD + \\ LoR\end{tabular}}} & \textbf{SMCNN(OURS)} & $O((n^2r)T)$ & $O(nr)$   \\
\multicolumn{1}{l|}{}                               & \textbf{SMCNmF(OURS)}                                                             & %$O((n^2r+nr^2)T) $  
$O((n^2r)T) $  
& $O(nr)$   \\ \hline
\bottomrule
\end{tabular} 
}
\end{table}
  
\noindent
\textbf{Computation Complexity}. PSD-MC methods (CMC, DMC, FMC) exhibit the highest computation complexity and storage complexity. For CMC and DMC, it is primarily due to the SVD operations with computation complexity $O(n^3)$ in each iteration \cite{stewart1993early}, which can be particularly burdensome for large similarity matrices. For FMC, it involves SVD operation once on the sub-similarity matrix between search samples. For MF-MC methods (BPMF, FGSR), in each iteration, BPMF shows lower computation complexity than PSD-MC methods, specifically on a large-scale dataset with $r< n$. For LoR-MC (NNFN, facNNFN), NNFN needs $O(n^2r_k)$ in each iteration (usually $n\geq r_k\geq r$), and facNNFN gains the same computation complexity as our SMCNmF and SMCNN. 
%Meanwhile, SMCNN enjoys the lowest computation complexity. 
Hence, DMC > CMC > FMC > BPMF > facNNFN > NNFN > FGSR = \textbf{SMCNN} = \textbf{SMCNmF}. 
%(Detailed comparison in Appendix \ref{app:}) \footnote{*** any ref. if not, i think we should conduct exp on $r_k$ or at least you should show that $r_k>r$ and the performance is not good.})

%\footnote{***plz define $\Omega$ first. otherwise, directly use $ 2n_\text{p}n_\text{q}+n_\text{q}^2$}
\noindent
\textbf{Storage Complexity}. %Recall the motivating scenario (Section \ref{subsec:pb}, the incomplete observations only occurred in the query database. We define $\Omega \in \{0,1\}^{n\times n}$ recorded positions of the observations, with $\Omega_{ij}=1$ if $S_{ij}$ is calculated by complete observations, and $\Omega_{ij}=0$ is calculated by the incomplete observations. $\Omega_{ij}=1$ only occurs in sub-matrices $S_{\text{pq}}$ and $S_{\text{qq}}$. Therefore, $\|\Omega\|_0 = 2n_\text{p}n_\text{q}+n_\text{q}^2$.)
CMC, DMC, FMC, and NNFN need $O(n^2)$ to store the estimated similarity matrix $\hat{S}\in \mathbb{R}^{n \times n}$. BPMF, FGSR, and facNNFN need $O(2nr)$ to store the two estimated factorized matrices. SMCNN and SMCNmF only need $O(nr)$ to store the factorized matrix $\hat{V}$. For the storage complexity, CMC = DMC = FMC = NNFN > BPMF = FGSR = facNNFN > \textbf{SMCNN} = \textbf{SMCNmF}.

Overall, our proposed SMCNN and SMCNmF exhibit the lowest computation complexity and storage complexity. %Though SMCNmF gains the same computational complexity as facNNFN, it provides a strong motivation for the nuclear norm regularizer can be penalized adaptively and shrinkagely. 
Meanwhile, the performance is also verified in empirical parts. (Detailed analysis shown in Appendix \ref{app:complexity})

\subsection{Better Estimation Guarantee}
%\subsection{Theoretical Guarantee for better estimation}
%\hl{moved 3.5 to 5, plz make the theorem better suit this sec.}
%\footnote{*** I think this theorem may help. though have not been proven. but it really similar to what we have proven. indeed it is exactly the same by taking $\gamma = 0$.} 
\begin{theorem} 
\label{theorem:RMSE}
$||S^*-\hat{S}||_F^2 \leq ||S^*-S^0||_F^2$. The equality holds if and only if $ \hat{S} = S^0$. 
\end{theorem}

Theorem \ref{theorem:RMSE} provides a theoretical guarantee that the optimal estimated similarity matrix $\hat{S}$ is closer to the fully observed ground-truth similarity matrix $S^*$ than the initial inaccurate similarity matrix $S^0$ which calculated by the incomplete data samples.

\subsection{Convergence Analysis} 
We prove that our SMCNN/SMCNmF can converge under a mild assumption:
\begin{assumption}
$\displaystyle ||V^{t} ||_{F} \leq G,\forall t\leq T$. 
\end{assumption}

This assumption indicates the norm of the matrix is bounded. Then, we can guarantee the convergence of our algorithm.

\begin{theorem}
\label{theorem:convergence}
Taking 
$\gamma =\frac{1}{ 6G^{2} + \lambda }$,  Algorithm \ref{alg:alg1} has:
    $$\min_{t} ||\nabla _{V^t} F_{1}( V^t) ||^{2} \leq O\left(\frac{1}{T}\right)$$
\end{theorem} 

Theorem \ref{theorem:convergence} provides theoretical evidence that Algorithm \ref{alg:alg1} can effectively minimize the objective function $F_ 1(V)$ and converge to a desirable solution efficiently. It states that the convergence rate achieved $O\left(\frac{1}{T}\right)$ with a fixed stepsize $\gamma=\frac{1}{6G^{2} +\lambda}$. With the increasing iteration $T$, the objective $F_1(V)$ approaches the minimal value. Meanwhile, the convergence rate $O\left(\frac{1}{T}\right)$ reflects that the objective value decreases at a rate inversely proportional to the number of iterations $T$.

%-----------------------------------------------------------------------------------------
\section{Experiment}
\label{sec:experiments}
\subsection{Experimental Settings} 
We perform Similarity Matrix Completion (SMC) experiments on five datasets (visual, textual, and bio-informatics datasets), using a PC with an Intel i9-11900 2.5GHz CPU and 32G memory. We report the performance by repeating the comparison methods for 5 runs. All experimental results are in the Appendix \ref{sec:exp}.
%\footnote{We will open source our code upon acceptance.}
% \footnote{*** why this is outside the para setting?} \textcolor{red}{overall exp setting here, each exp setting in para setting6.2}

\subsubsection{Datasets} % \footnote{*** do we really need to discuss categories? we do not use it here.}  \footnote{*** do we really need to discuss objects? we do not use it here.} 
\textbf{1) Visual datasets.} \textbf{1.1) ImageNet} \cite{deng2009imagenet}: an image dataset contains $12,000,000$ samples with $1,000$ dimensions. \textbf{1.2) MNIST} \cite{lecun1998gradient}: an image dataset of handwritten digits (e.g. 0-9) contains $ 60,000 $ samples with $784$ dimensions. \textbf{1.3) CIFAR10} \cite{krizhevsky2009learning}: an image dataset of $60,000$ samples with $1,024$ dimensions.
\textbf{2) Bio-informatics datasets.} \textbf{2.1) PROTEIN} \cite{wang2002application}: a standard structural bio-informatics dataset contains $ 23,000$ samples with $357$ dimensions. \textbf{3) Textual datasets.} \textbf{3.1) GoogleNews} \footnote{$^4$ \url{https://www.kaggle.com/datasets/adarshsng/googlenewsvectors}}: a textual dataset of $3$ million words in $300$ dimensions. 
 
 \begin{table}[!htb]
\centering
%\scriptsize
\huge
\caption{Comparisons of RMSE, Recall@top20\%, nDCG@top20\% with $m=1,000$ search candidates and $n=200$ query items, where missing ratio $\rho =\{0.7,0.8,0.9\}$, rank $r = 100$, $\lambda = 0.001$, $\gamma = 0.001$, and $T = 10,000$. The top 2 performances are highlighted in bold.} 
\label{tab:rmse:s1q2}  
\resizebox{1\columnwidth}{!}{   
\begin{tabular}{cc|c|ccc|ccc|ccc}
\toprule
\hline
\multicolumn{3}{c|}{\multirow{2}{*}{\textbf{Dataset/Method}}}                                                                      & \multicolumn{3}{c|}{\textbf{0.7}}                                     & \multicolumn{3}{c|}{\textbf{0.8}}                                     & \multicolumn{3}{c}{\textbf{0.9}}                                                \\ \cline{4-12} 
\multicolumn{3}{c|}{}                                                                                                              & \textbf{RMSE}         & \textbf{Recall}       & \textbf{nDCG}         & \textbf{RMSE}         & \textbf{Recall}       & \textbf{nDCG}         & \textbf{RMSE}              & \textbf{Recall}       & \textbf{nDCG}              \\ \hline
\multicolumn{1}{c|}{\multirow{12}{*}{\textbf{ImageNet}}}   & \multicolumn{1}{c|}{\multirow{4}{*}{\textbf{MVI}}} & \textbf{MEAN}\cite{kim2004reuse}   & 2.4314                & 0.6985                & 0.7821                & 2.1510                & 0.5420                & 0.7124                & 1.9898                     & 0.5715                & 0.7851                     \\
\multicolumn{1}{c|}{}                                      & \multicolumn{1}{c|}{}                              & \textbf{LR}\cite{seber2012linear}     & 1.2130                & 0.6985                & 0.7941                & 1.2145                & 0.5983                & 0.7154                & 1.6743                     & 0.5568                & 0.8012                     \\
\multicolumn{1}{c|}{}                                      & \multicolumn{1}{c|}{}                              & \textbf{GR}\cite{balzano2010online}     & 1.2140                & 0.6741                & 0.7214                & 1.7530                & 0.5012                & 0.7256                & 1.5438                     & 0.5939                & 0.7987                     \\
\hline
\bottomrule 
\end{tabular} 
}
\end{table}  
\begin{figure*}[!htb]
    \centering    
    \hspace{-5cm}
      \subfigure[(a) ImageNet, $r$=50]{
   \includegraphics[width =0.6\columnwidth, trim=15 0 30 10, clip]{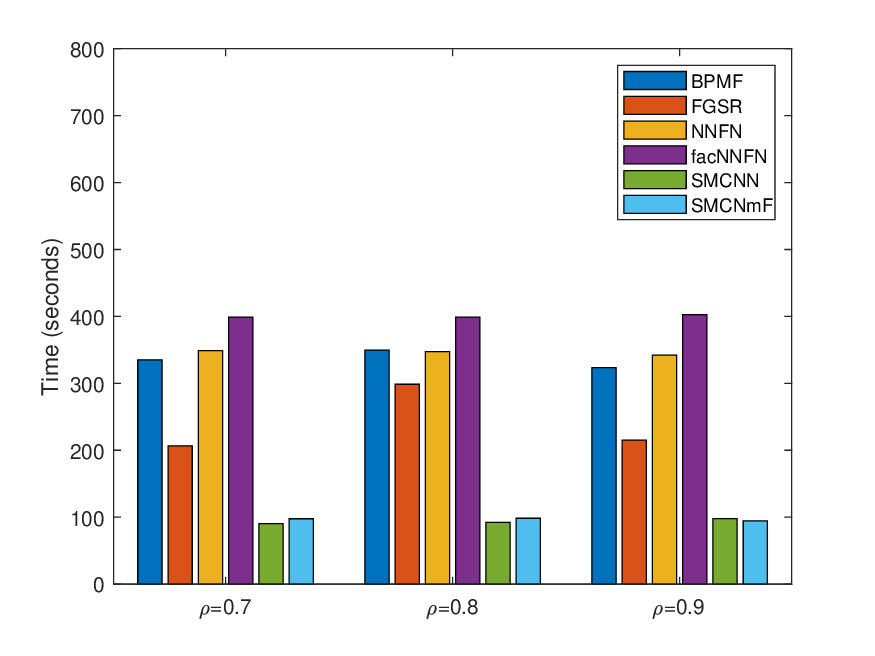}}
      \subfigure[(b) ImageNet, $r$=100]{
   \includegraphics[width =0.6\columnwidth, trim=15 0 30 10, clip]{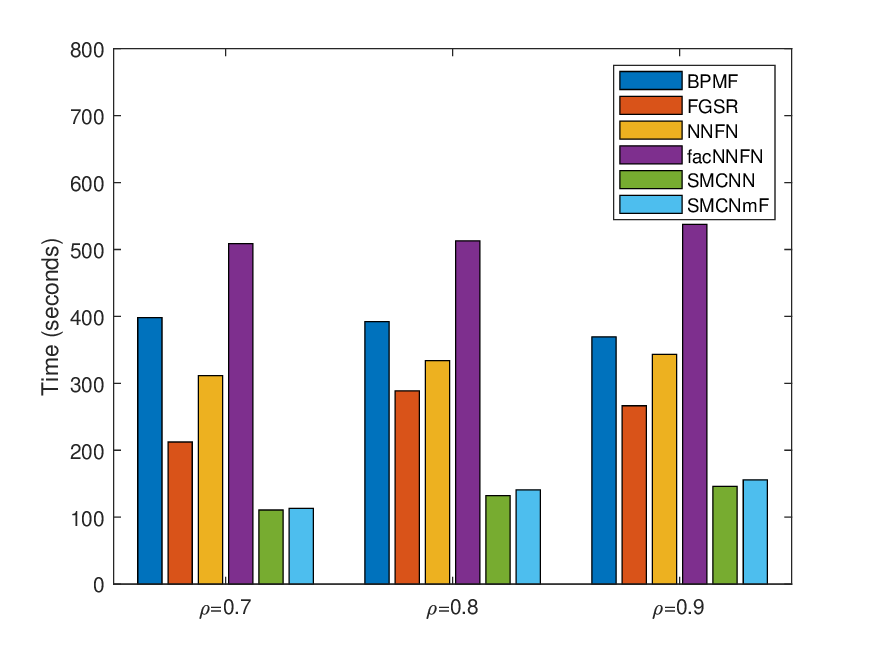}}
      \subfigure[(c) ImageNet, $r$=200]{
   \includegraphics[width =0.6\columnwidth, trim=15 0 30 10, clip]{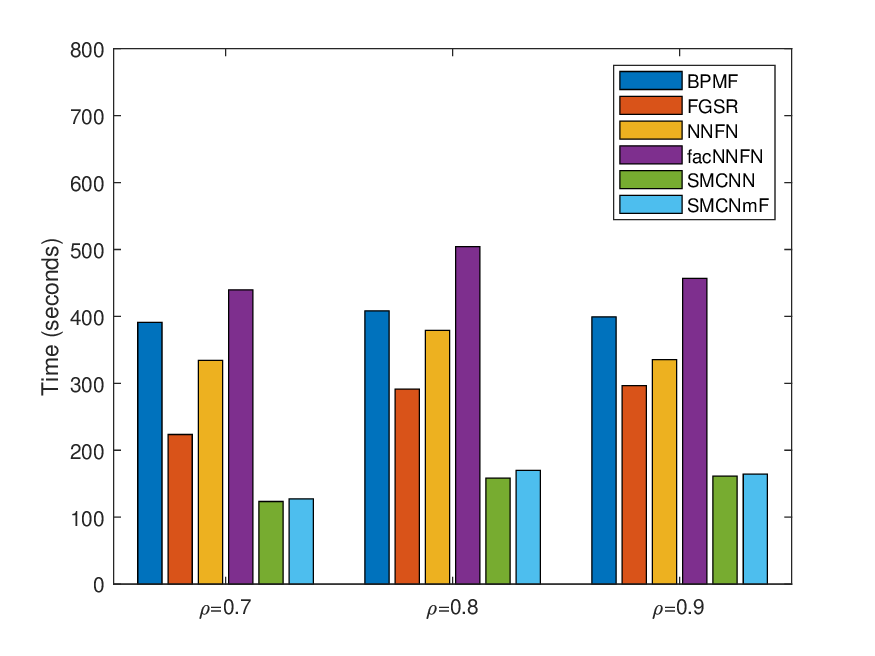}}
    \hspace{-5cm}
    \vspace{-0.5cm}
    \caption{Running Time versus $\rho$ on ImageNet dataset, with $m=1,000$ search candidates and $n=200$ query items, with rank $r=\{50, 100, 200\}$, $\lambda = 0.001$, $\gamma = 0.001$, $T=10,000$. (Best viewed in $\times$ 2 sized color pdf file)}
    \label{fig:time:ImageNet}
\end{figure*}  
 
\subsubsection{Baseline Methods} 
For all comparison methods in the experiments, we use public codes unless they are not available. \textbf{1) DMVI methods}. \textbf{1.1) Mean} \cite{kim2004reuse}: The missing values are imputed by the mean value in search data samples. \textbf{1.2) LR} \cite{seber2012linear}: The missing values are imputed by the multivariate linear regression between observed features and missing features. \textbf{1.3) GR} \cite{balzano2010online}: The missing values are imputed based on the low-rank assumption on the original dataset. \textbf{1.4) KFMC} \cite{fan2019online}: The missing values are imputed and optimized in online data based on high-rank matrix assumption. \textbf{2) DIMC/MC methods}. \textbf{2.1) MF-MC}: \textbf{2.1.1) BPMF}\cite{mnih2007probabilistic}: The incomplete matrix is recovered by investigating the hidden patterns in the data using Bayesian inference. \textbf{2.1.2) FGSR}\cite{fan2019factor}: The incomplete matrix is recovered using factored group-sparse regularizers without SVD operation. \textbf{2.2) LoR-MC}: \textbf{2.2.1) NNFN} \cite{wang2021scalable}: The incomplete matrix is recovered by a proximal algorithm based on non-negative SVD operation. \textbf{2.2.2) facNNFN} \cite{wang2021scalable}: The scalable variant facNNFN recovered the incomplete matrix with a gradient descent algorithm. \textbf{2.3) PSD-MC}: \textbf{2.3.1) DMC} \cite{li2015estimating}: The similarity matrix is estimated by searching the approximated matrix with PSD based on Dykstra's alternating projection method. \textbf{2.3.2) CMC} \cite{li2020scalable}: The estimated similarity matrix is approximated by the cyclic projection with PSD constraint. \textbf{2.3.3) FMC} \cite{ma2024fast}: The estimated similarity matrix is approximated by estimating the similarity vector sequentially with PSD constraint. %Note the PSD-MC methods focus on the typical similarity matrix completion problem by utilizing the unique PSD property.
%\footnote{*** is the last sentence really necessary.}

\subsubsection{Parameters setting}
All the algorithms are implemented in MATLAB R2022a. Each algorithm is stopped when the iterations $T$ = $10,000$. All hyperparameters are tuned by grid search. Specifically, step-size $\gamma$ is chosen from $\{10^{-3}, 10^{-2}, 10^{-1}, 1, 10\}$, weighted parameters $\lambda$ is chosen from $\{10^{-3}, 10^{-2}, 10^{-1}, 1, 10\}$, missing ratio $\rho$ is chosen from $\{0.7,0.8,0.9\}$, rank $r$ is chosen from $\{50, 100, 200\}$. For the baseline methods, we directly use the parameters mentioned in the original paper for all datasets.

\subsection{Evaluation Metric} 
We evaluate the performance of the comparison methods referring to the following metrics: 

\noindent
\textbf{\textit{RMSE}}: To evaluate the SMC performance, we adopt the widely used evaluation metric, Relative Mean Squared Error (RMSE) \cite{chai2014root}: \textbf{RMSE} = $\frac{\|\hat{S}-S^*\|_F^2}{\|S^0-S^*\|_F^2}$  Specifically, RMSE $< 1$ denotes $\hat{S}$ is close to the unknown complete similarity matrix $S^*$, which means the good similarity matrix completion performance. Otherwise, RMSE $\geq 1$ denotes $\hat{S}$ cannot approximate the $S^*$, which reveals the insufficient similarity matrix completion ability. 
%This is consistent with Theorem \ref{theorem:RMSE}. 
Overall, a smaller RMSE denotes better performance.

\noindent
\textbf{\textit{Recall}}: To evaluate pairwise and relevant similarity preservation capability, following the previous work \cite{kyselak2011stabilizing,zezula2006similarity}, we adopt Recall@top$k$ as the evaluation metric: \textbf{Recall}@top$k$ := $\frac{|I_k(\hat{S}) \cap I_k(S^*)|}{k}\cdot 100\% $, where $I_k(S)$ denotes the index of top$k$ items in matrix $S$ in decreasing order. Here, the top$k$ most similar items with the largest similarity scores in $S^*$ denotes the ground truth, and the top$k$ most similar items in $\hat{S}$ denotes the estimated results. As a result, Recall records the overlapped ratios in the top$k$ set, and varies in the range $[0,1]$, where $0$ means no overlap items and $1$ means that all overlapping similar indices are the same. A larger Recall denotes a better performance. 

\noindent
\textbf{\textit{nDCG}}: To evaluate the quality of a ranking system by comparing the positions of relevant items in its output to an ideal ordering, we calculate the Normalized Discounted Cumulative Gain (\textbf{nDCG}) \cite{jarvelin2002cumulated}: $\textbf{nDCG}_{k} = \frac{\text{DCG}_{k}(\hat{S})}{\text{IDCG}_{k}(S^*)}$, where $\text{DCG}_{k}=\sum_{i=1}^{k}\frac{2^{\text{rel}(\hat{S}_i)}-1}{\log_2(i+1)}$ with $\text{rel}(\hat{S}_i)$ is the relevance score of the $i$-th item in $\hat{S}$ based on the positions in $S^*$, and $\text{IDCG}_{k}=\sum_{i=1}^k \frac{2^{\text{rel}(S^*)}-1}{\log_2(i+1)}$ with $\text{rel}(S^*_i)$ is the relevance score of the $i$-th item in $S^*$. Note, when ${\text{IDCG}_{k}(S^*)}$ is $0$, then $\text{DCG}_{k}(\hat{S})$ is $0$, indicating non relevant items in top$k$ positions. \text{nDCG} varies in $[0,1]$, comparing the ranking quality of the estimated $\hat{S}$ against the ground truth $S^*$ up to a certain rank position $k$. A larger nDCG denotes better performance.  
%\textcolor{red}{ \textbf{Notation.} To provide a fair and meaningful evaluation, we design different strategies on various evaluation metrics in practice. \textbf{RMSE}. We calculate the RMSE as $\frac{||\hat{S}_{\text{-PP}}-S^*_{\text{-PP}}||_F^2}{||S^0_{\text{-PP}}-S^*_{\text{-PP}}||_F^2}$, where $S_{\text{-PP}}$ denotes the entire similarity matrix $S$ exclude the sub-matrix $S_{\text{PP}}$. Since the $S_{\text{PP}}$ denotes the similarity score of completed observed search candidates and is fixed in either $\hat{S}$, $S^0$, and $S^*$, it is no need to calculate the difference. \textbf{Recall \& nDCG}. We keep the entire similarity matrix $\hat{S}$ and $S^*$ and calculate the overlapping ratios of similar items since we aim to preserve the relative pairwise similarity score and evaluate the ranking/order of similarity.}

\subsection{Effectiveness Evaluation}
\label{subsec:effectiveness}
%\footnote{*** more details, some are worse than smc, plz mention them here.} 

Table \ref{tab:rmse:s1q2} shows the results of RMSE, Recall@top 20\%, and nDCG@top20\% of all comparison methods. In most cases, our SMCNN/SMCNmF always produces the best SMC performance compared with other baseline methods in terms of RMSE. Additionally, the estimated similarity matrix $\hat{S}$ also yields the best Recall and nDCG than the comparison methods, revealing that our methods can well preserve the pairwise similarity. (The remaining results are shown in Appendix \ref{App:RMSE}) 

\subsection{Efficiency Evaluation}
\label{subse:efficiency} 
Fig.\ref{fig:time:ImageNet} shows the running time varying with missing ratio $\rho$ on ImageNet dataset, with fixed rank $r=\{50, 100, 200\}$, weighted parameter $\lambda=0.001$, and stepsize $\gamma=0.001$. Since SVD is the main computation bottleneck, we only compare the running time of all baseline methods involving MF techniques. Among all the comparison methods, our SMCNN/SMCNmF achieves the fastest speed, which is consistent with the complexity analysis shown in Section \ref{sec:analysis}. Combining the effectiveness results in Section \ref{subsec:effectiveness}, our SMCNN and SMCNmF achieve the best estimation performance with the lowest running time, which reveals the low-rank setting is effective and efficient in most cases. Both the MF-MC methods (BFMF, FGSR) and LoR-MC methods (NNFN, facNNFN) take much more time under various settings, further revealing the unbalance between effectiveness and efficiency. (The remaining results are shown in Appendix \ref{App:Time})

% \footnote{*** four in one row, otherwise they are too small to see,} \footnote{*** enlarge the xlabel and ylabel, as well as caption in legend  } \footnote{*** smc is missing. \textcolor{red}{just compare the MF based methods (methods has $r$ parameter)}}  

\subsection{Ablation Study}  
\label{sec:exp:ab}
We further apply ablation study to better understand the importance of the rank-minimization regularizer of SMCNN/SMCNmF.  We conduct a series of ablation studies: \textbf{i) SMC Variants}: We adopt three variants of SMC problem, i.e., Problem \eqref{eq:problem_def_PSDLoR}, including adopting the Frobenius norm as the regularizer (\textbf{SMC\_F}), ignore the regularizer (\textbf{SMC\_NR}), and use truncated-SVD to solve nuclear norm (\textbf{SMC\_GD}). \textbf{ii) MC Variants}: We adopt two variants of general MC problem, i.e., Problem \eqref{eq:MCprob}, including direct summation of all eigenvalues \cite{cai2010singular} to calculate the nuclear norm (\textbf{MC\_ON}), directly use matrix factorization \cite{srebro2004maximum} to calculate the nuclear norm term (\textbf{MC\_fON}). 
Table \ref{tab:RMSE:abmethods} shows the comparison of various methods. 

\begin{table}[!htb]
\centering
\huge
\caption{Comparisons of Various Ablation Methods, FrNorm denotes Frobenius Norm, NuNorm denotes Nuclear Norm. } 
\label{tab:RMSE:abmethods} 
\vspace{-0.4cm}
\resizebox{\columnwidth}{!}{   
\begin{tabular}{ll|ll|l|l}
\toprule
\hline
\multicolumn{2}{c|}{\textbf{Method}}                                & \textbf{LoR} & \textbf{PSD} & \textbf{Regularizer} & \textbf{Solver}      \\ \hline
\multicolumn{1}{l|}{\multirow{3}{*}{\textbf{\begin{tabular}[c]{@{}l@{}}SMC \\ Variants\end{tabular}}}} & \textbf{SMC\_F} & ×            & $\checkmark$ & FrNorm               & Gradient Descent     \\
\multicolumn{1}{l|}{}                            & \textbf{SMC\_NR}                                                      & ×            & $\checkmark$ & No Regularizer       & Gradient Descent     \\
\multicolumn{1}{l|}{}                            & \textbf{SMC\_GD}                                                      & $\checkmark$ & $\checkmark$ & NuNorm               & Gradient Descent     \\ \hline
\multicolumn{1}{l|}{\multirow{2}{*}{\textbf{\begin{tabular}[c]{@{}l@{}}MC \\ Variants\end{tabular}}}}  & \textbf{MC\_ON} & $\checkmark$ & ×            & NuNorm               & SVD Operation        \\
\multicolumn{1}{l|}{}                            & \textbf{MC\_fON}                                                      & $\checkmark$ & ×            & NuNorm               & Gradient Descent     \\ \hline
\multicolumn{1}{l|}{\multirow{3}{*}{\textbf{\begin{tabular}[c]{@{}l@{}}PSD - \\ SMC \end{tabular}}}}                                               & \textbf{CMC}\cite{li2015estimating}   & ×            & $\checkmark$ & No Regularizer       & Cyclic Projection    \\
\multicolumn{1}{l|}{}                            & \textbf{DMC}\cite{li2020scalable}  & ×            & $\checkmark$ & No Regularizer       & Dykstra's Projection \\
\multicolumn{1}{l|}{}                            &\textbf{FMC}\cite{ma2024fast}        & ×            & $\checkmark$ & No Regularizer       & Lagrangian Method    \\ \hline
\multicolumn{1}{l|}{\multirow{2}{*}{\textbf{OURS}}}                                                    & \textbf{SMCNN}  & $\checkmark$ & $\checkmark$ & NuNorm               & Gradient Descent     \\
\multicolumn{1}{l|}{}                            & \textbf{SMCNmF} & $\checkmark$ & $\checkmark$ & NuNorm+FrNorm        & Gradient Descent     \\ \hline
\bottomrule
\end{tabular} 
}
\end{table}

%\textbf{ SMC\_F, SMC\_NR, SMC\_QN, SMC\_QN, SMCNN\_QN} and \textbf{SMCNmF\_QN} are a type of SMC, where the similarity matrix being PSD or Low-rank. \textbf{MC\_ON, MC\_fON} are a type of general MC, where the low-rank property of the similarity matrix can be ensured.  

\begin{table}[!htb]
\centering
\huge    
\caption{Ablation study on various variations of RMSE, Recall@top20\%, nDCG@top20\% with $m=1,000$ search candidates and $n=200$ query items, where missing ratio $\rho =\{0.7,0.8,0.9\}$, rank $r = 100$, $\lambda = 0.001$, $\gamma = 0.001$, and $T = 10,000$. The top 2 performances are highlighted in bold.} 
\label{tab:ab:Q2S1} 
\vspace{-0.4cm}
\resizebox{\columnwidth}{!}{   
\begin{tabular}{ll|lll|lll|lll}
\toprule
\hline
\multicolumn{2}{c|}{\multirow{2}{*}{\textbf{Datasets/Method}}}                                                     & \multicolumn{3}{c|}{\textbf{0.7}}                   & \multicolumn{3}{c|}{\textbf{0.8}}                   & \multicolumn{3}{c}{\textbf{0.9}}                    \\ \cline{3-11} 
\multicolumn{2}{c|}{}                                                                                             & \textbf{RMSE}   & \textbf{Recall} & \textbf{nDCG}   & \textbf{RMSE}   & \textbf{Recall} & \textbf{nDCG}   & \textbf{RMSE}   & \textbf{Recall} & \textbf{nDCG}   \\ \hline
\multicolumn{1}{l|}{\multirow{10}{*}{\textbf{ImageNet}}}   & \textbf{SMC\_F}                                      & 0.6496          & 0.5354          & 0.9602          & 0.5905          & 0.4521          & 0.6766          & 0.4586          & 0.7348          & 0.6750          \\
\multicolumn{1}{l|}{}                                      & \textbf{SMC\_NR}                                     & 0.6496          & 0.5375          & 0.9603          & 0.5904          & 0.4521          & 0.6745          & 0.4587          & 0.7348          & 0.6798          \\
\multicolumn{1}{l|}{}                                      & \textbf{SMC\_GD}                                     & 0.6122          & 0.6123          & 0.9469          & 0.4690          & 0.5170          & 0.6949          & 0.4574          & 0.7201          & 0.6329          \\ \cline{2-11} 
\multicolumn{1}{l|}{}                                      & \textbf{MC\_ON}                                      & 1.0120          & 0.5436          & 0.6475          & 1.5000          & 0.5663          & 0.6475          & 1.0830          & 0.4063          & 0.6476          \\
\multicolumn{1}{l|}{}                                      & \textbf{MC\_fON}                                     & 1.1602          & 0.6896          & 0.6553          & 0.6038          & 0.5217          & 0.6333          & 0.5169          & 0.5245          & 0.5672          \\ \cline{2-11} 
\multicolumn{1}{l|}{}                                      & \textbf{CMC}\cite{li2015estimating} & 0.6539          & 0.7276          & 0.6772          & 0.5018          & 0.5280          & 0.5739          & 0.5462          & 0.7004          & 0.7988          \\
\multicolumn{1}{l|}{}                                      & \textbf{DMC}\cite{li2020scalable}   & 0.5980          & 0.6192          & 0.6320          & 0.4742          & 0.5915          & 0.6320          & 0.7714          & 0.6771          & 0.8290          \\
\multicolumn{1}{l|}{}                                      &  \textbf{FMC}\cite{ma2024fast}       & 0.6269          & 0.6644          & 0.6984          & 0.7347          & 0.4688          & 0.5984          & 0.7365          & 0.6613          & 0.8430          \\ \cline{2-11} 
\multicolumn{1}{l|}{}                                      & \textbf{SMCNN(OURS)}                                 & \textbf{0.5689} & \textbf{0.8223} & \textbf{0.7866} & \textbf{0.4443} & \textbf{0.5951} & \textbf{0.6945} & \textbf{0.4546} & \textbf{0.7708} & \textbf{0.8951} \\
\multicolumn{1}{l|}{}                                      & \textbf{SMCNmF(OURS)}                                & \textbf{0.5658} & \textbf{0.8237} & \textbf{0.7992} & \textbf{0.4451} & \textbf{0.5950} & \textbf{0.6948} & \textbf{0.4457} & \textbf{0.7681} & \textbf{0.8952} \\ \hline
\multicolumn{1}{l|}{\multirow{10}{*}{\textbf{MNIST}}}      & \textbf{SMC\_F}                                      & 0.9721          & 0.6083          & 0.7628          & 0.8487          & 0.5396          & 0.7492          & 0.5682          & 0.4729          & 0.7346          \\
\multicolumn{1}{l|}{}                                      & \textbf{SMC\_NR}                                     & 0.9728          & 0.6083          & 0.7628          & 0.8488          & 0.5396          & 0.7490          & 0.5682          & 0.4729          & 0.7346          \\
\multicolumn{1}{l|}{}                                      & \textbf{SMC\_GD}                                     & 0.8380          & 0.7378          & 0.8155          & 0.7155          & 0.6156          & 0.7806          & 0.4913          & 0.6706          & 0.8429          \\ \cline{2-11} 
\multicolumn{1}{l|}{}                                      & \textbf{MC\_ON}                                      & 0.8878          & 0.6577          & 0.7707          & 0.7415          & 0.5440          & 0.7226          & 0.7377          & 0.5177          & 0.7707          \\
\multicolumn{1}{l|}{}                                      & \textbf{MC\_fON}                                     & 0.9593          & 0.6236          & 0.7780          & 0.7416          & 0.5440          & 0.7726          & 0.7377          & 0.5317          & 0.7707          \\ \cline{2-11} 
\multicolumn{1}{l|}{}                                      & \textbf{CMC}\cite{li2015estimating} & 1.0000          & 0.6285          & 0.7095          & 0.9962          & 0.5525          & 0.7291          & 0.7788          & 0.6604          & 0.8314          \\
\multicolumn{1}{l|}{}                                      & \textbf{DMC}\cite{li2020scalable}   & 1.0774          & 0.6025          & 0.7409          & 0.7513          & 0.5103          & 0.7409          & 0.5378          & 0.6617          & 0.8125          \\
\multicolumn{1}{l|}{}                                      &  \textbf{FMC}\cite{ma2024fast}       & 0.8491          & 0.6840          & 0.7998          & \textbf{0.6559} & 0.5645          & 0.7065          & 0.6573          & 0.6451          & 0.7977          \\ \cline{2-11} 
\multicolumn{1}{l|}{}                                      & \textbf{SMCNN(OURS)}                                 & \textbf{0.8233} & \textbf{0.7762} & \textbf{0.8488} & \textbf{0.7080} & \textbf{0.6432} & \textbf{0.7845} & \textbf{0.4893} & \textbf{0.6961} & \textbf{0.8854} \\
\multicolumn{1}{l|}{}                                      & \textbf{SMCNmF(OURS)}                                & \textbf{0.8236} & \textbf{0.7762} & \textbf{0.8542} & 0.7115          & \textbf{0.6323} & \textbf{0.7829} & \textbf{0.4893} & \textbf{0.6961} & \textbf{0.8814} \\ \hline
\multicolumn{1}{l|}{\multirow{10}{*}{\textbf{PROTEIN}}}    & \textbf{SMC\_F}                                      & 0.7829          & 0.6448          & 0.7970          & 0.7178          & 0.4437          & 0.6580          & 0.6078          & 0.5979          & 0.7914          \\
\multicolumn{1}{l|}{}                                      & \textbf{SMC\_NR}                                     & 0.7853          & 0.6586          & 0.7969          & 0.7191          & 0.4357          & 0.6548          & 0.6076          & 0.6021          & 0.7914          \\
\multicolumn{1}{l|}{}                                      & \textbf{SMC\_GD}                                     & 0.5922          & 0.7613          & 0.7775          & 0.5177          & 0.5178          & 0.7455          & 0.5291          & 0.7284          & 0.8864          \\ \cline{2-11}  
\multicolumn{1}{l|}{}                                      & \textbf{MC\_ON}                                      & 0.9221          & 0.5448          & 0.6605          & 0.7797          & 0.9304          & 0.6536          & 0.5730          & 0.7693          & 0.8233          \\
\multicolumn{1}{l|}{}                                      & \textbf{MC\_fON}                                     & 1.1450          & 0.4997          & 0.6705          & 0.7448          & 0.9244          & 0.6313          & 0.6780          & 0.7692          & 0.8332          \\ \cline{2-11} 
\multicolumn{1}{l|}{}                                      & \textbf{CMC}\cite{li2015estimating} & 0.7731          & 0.5980          & 0.7970          & 0.5946          & 0.4515          & 0.6436          & 0.5241          & 0.7614          & 0.8954          \\
\multicolumn{1}{l|}{}                                      & \textbf{DMC}\cite{li2020scalable}   & 0.8329          & 0.6029          & 0.7548          & 0.6585          & 0.4990          & 0.6548          & 0.5136          & 0.7455          & 0.8959          \\
\multicolumn{1}{l|}{}                                      &  \textbf{FMC}\cite{ma2024fast}       & 0.6313          & 0.6772          & 0.6995          & 0.5665          & 0.5618          & 0.6995          & 0.5664          & 0.6715          & 0.8949          \\ \cline{2-11} 
\multicolumn{1}{l|}{}                                      & \textbf{SMCNN(OURS)}                                 & \textbf{0.5861} & \textbf{0.7435} & \textbf{0.7974} & \textbf{0.4885} & \textbf{0.5160} & \textbf{0.6955} & \textbf{0.5049} & \textbf{0.7814} & \textbf{0.9255} \\
\multicolumn{1}{l|}{}                                      & \textbf{SMCNmF(OURS)}                                & \textbf{0.5832} & \textbf{0.7649} & \textbf{0.7974} & \textbf{0.4887} & \textbf{0.5240} & \textbf{0.6955} & \textbf{0.5049} & \textbf{0.7814} & \textbf{0.9436} \\ \hline
\multicolumn{1}{l|}{\multirow{10}{*}{\textbf{CIFAR10}}}      & \textbf{SMC\_F}                                      & 0.9999          & 0.5442          & 0.7783          & 0.9798          & 0.5479          & 0.6548          & 0.8367          & 0.4604          & 0.7614          \\
\multicolumn{1}{l|}{}                                      & \textbf{SMC\_NR}                                     & 1.0000          & 0.5463          & 0.7783          & 0.9801          & 0.5458          & 0.6458          & 0.8369          & 0.4604          & 0.7614          \\
\multicolumn{1}{l|}{}                                      & \textbf{SMC\_GD}                                     & 0.7380          & 0.6378          & 0.8594          & 0.5594          & 0.6595          & 0.6270          & 0.5698          & 0.7589          & 0.8234          \\ \cline{2-11}  
\multicolumn{1}{l|}{}                                      & \textbf{MC\_ON}                                      & 0.9226          & 0.5358          & 0.6530          & 0.7815          & 0.5693          & 0.6546          & 0.5729          & 0.5693          & 0.8543          \\
\multicolumn{1}{l|}{}                                      & \textbf{MC\_fON}                                     & 0.9226          & 0.5558          & 0.6510          & 0.7815          & 0.5930          & 0.6209          & 0.5729          & 0.5436          & 0.8435          \\ \cline{2-11} 
\multicolumn{1}{l|}{}                                      & \textbf{CMC}\cite{li2015estimating} & 0.9887          & 0.6284          & 0.8622          & 0.7288          & 0.6244          & 0.6123          & 0.5214          & 0.7314          & 0.8248          \\
\multicolumn{1}{l|}{}                                      & \textbf{DMC}\cite{li2020scalable}   & 1.0430          & 0.5893          & 0.8379          & 0.5747          & 0.6309          & 0.6379          & 0.5145          & 0.7251          & 0.8483          \\
\multicolumn{1}{l|}{}                                      &  \textbf{FMC}\cite{ma2024fast}       & 0.5459          & 0.5775          & 0.6977          & 0.6509          & 0.5775          & 0.5995          & 0.5269          & 0.7274          & 0.7977          \\ \cline{2-11} 
\multicolumn{1}{l|}{}                                      & \textbf{SMCNN(OURS)}                                 & \textbf{0.7044} & \textbf{0.6832} & \textbf{0.8903} & \textbf{0.5097} & \textbf{0.6822} & \textbf{0.6878} & \textbf{0.4239} & \textbf{0.7491} & \textbf{0.8658} \\
\multicolumn{1}{l|}{}                                      & \textbf{SMCNmF(OURS)}                                & \textbf{0.7044} & \textbf{0.6832} & \textbf{0.8913} & \textbf{0.5081} & \textbf{0.6842} & \textbf{0.6819} & \textbf{0.4239} & \textbf{0.7492} & \textbf{0.8686} \\ \hline
\multicolumn{1}{l|}{\multirow{10}{*}{\textbf{GoogleNews}}} & \textbf{SMC\_F}                                      & 0.8612          & 0.6553          & 0.7930          & 0.7754          & 0.6598          & 0.6790          & 0.6395          & 0.6771          & 0.7736          \\
\multicolumn{1}{l|}{}                                      & \textbf{SMC\_NR}                                     & 0.8603          & 0.6554          & 0.7930          & 0.7754          & 0.6370          & 0.6899          & 0.6394          & 0.6771          & 0.7736          \\
\multicolumn{1}{l|}{}                                      & \textbf{SMC\_GD}                                     & 0.6024          & 0.7051          & 0.8019          & 0.5219          & 0.7087          & 0.6778          & 0.5617          & 0.7216          & 0.8374          \\ \cline{2-11}  
\multicolumn{1}{l|}{}                                      & \textbf{MC\_ON}                                      & 0.8425          & 0.6423          & 0.6301          & 0.7595          & 0.6458          & 0.6532          & 0.5768          & 0.6436          & 0.7354          \\
\multicolumn{1}{l|}{}                                      & \textbf{MC\_fON}                                     & 0.8433          & 0.6345          & 0.5907          & 0.7597          & 0.6548          & 0.6362          & 0.5769          & 0.6565          & 0.7354          \\ \cline{2-11} 
\multicolumn{1}{l|}{}                                      & \textbf{CMC}\cite{li2015estimating} & 0.8085          & 0.6175          & 0.7989          & 0.5776          & 0.6138          & 0.6985          & 0.7141          & 0.7241          & 0.8848          \\
\multicolumn{1}{l|}{}                                      & \textbf{DMC}\cite{li2020scalable}   & 0.9477          & 0.6340          & 0.7950          & 0.5231          & 0.6134          & 0.6950          & 0.5645          & 0.6600          & 0.7984          \\
\multicolumn{1}{l|}{}                                      &  \textbf{FMC}\cite{ma2024fast}       & 0.6109          & 0.7075          & 0.7998          & 0.5362          & 0.6743          & 0.6843          & 0.5322          & 0.6743          & 0.6949          \\ \cline{2-11} 
\multicolumn{1}{l|}{}                                      & \textbf{SMCNN(OURS)}                                 & \textbf{0.5553} & \textbf{0.7366} & \textbf{0.8170} & \textbf{0.4241} & \textbf{0.7129} & \textbf{0.6980} & \textbf{0.5115} & \textbf{0.7685} & \textbf{0.8922} \\
\multicolumn{1}{l|}{}                                      & \textbf{SMCNmF(OURS)}                                & \textbf{0.5684} & \textbf{0.7364} & \textbf{0.8130} & \textbf{0.4232} & \textbf{0.7120} & \textbf{0.6980} & \textbf{0.5116} & \textbf{0.7686} & \textbf{0.8958} \\ \hline
\bottomrule
\end{tabular} 
}
\end{table} 
%\footnote{*** you say NuNorm denotes Nuclear Norm but you still has Nuclear Norm in your table!}
Table \ref{tab:ab:Q2S1} shows the performance of RMSE, Recall@top 20\%, and nDCG@top20\% of variations on various datasets. Among these different settings, our SMCNN/SMCNmF outperforms the other variations on various datasets, which further reveals the necessitate of forcing the similarity matrix to be low-rank by nuclear norm and finding the optimal results via well-defined efficient and general solvers. Combined with the performance of PSD-MC methods (CMC, DMC, FMC), we summarize that using either PSD or LoR on SMC problems cannot achieve the best performance over using both PSD and LoR properties. (Detailed descriptions of the entire algorithm are shown in Appendix \ref{App:ab}).% (The remaining results are shown in Appendix \ref{App:ab}) Besides, we also provide an alternative solver to solve the SMC problem. (Detailed description of the entire algorithm are shown in Appendix \ref{App:ab}). 

\subsection{Lower-Rank Matrix Property Evaluation}   
%\footnote{*** smc is missing.} \footnote{*** delete all SMC QN, SMCNN QN, SMCNmF  QN as they add their raise more concerns. \textcolor{red}{va3-5 are LoR? Is it okay to delete it?}} 
Table \ref{tab:Rank:hatS} shows the rank of the estimated matrix $\hat{S}$ generated by different methods. Specifically, the rank is calculated by forcing all the eigenvalues larger than a threshold value $1e-15$, which is the default setting in MATLAB. Our SMCNN/SMCNmF ensures the \textit{Lower-Rank Matrix Property} always holds, which verifies the theoretical observations in Theorem \ref{theorem:guarantee}. Combining the theoretical analysis and empirical results, we conclude that SMCNN/SMCNmF achieves a satisfying estimation performance with fast speed, which further verifies the potential usage of the low-rank regularizer can provide an effective lower-rank estimation with high efficiency.  
%update hatS Rank 
\begin{table}[ht]
\centering
\huge 
\caption{Rank($\hat{S}$), with $m=1,000$ search candidates and $n=200$ query items, where $r=100$ with various missing ratio $\rho$, fixed $\lambda=0.001$, $\gamma =0.001$, and $T=10,000$. The top 2 performances (a smaller $r$ is better) are highlighted in bold.}
\vspace{-0.4cm}
\label{tab:Rank:hatS} 
\resizebox{\columnwidth}{!}{   
\begin{tabular}{c|c|ccccccccc}
\toprule
\hline
\textbf{Datasets}                    & $\rho$       & \textbf{CMC} & \textbf{DMC} & \textbf{SMC} & \textbf{FGSR} & \textbf{BPMF} & \textbf{NNFN} & \textbf{facNNFN} & \textbf{SMCNN} & \textbf{SMCNmF} \\ \hline
\multirow{3}{*}{\textbf{ImageNet}}   & \textbf{0.7} & 1200         & 1200         & 1064         & 683           & 1052          & 1200          & 1200             & \textbf{81}    & \textbf{81}     \\
                                     & \textbf{0.8} & 1200         & 1200         & 1104         & 691           & 1052          & 1200          & 1200             & \textbf{80}    & \textbf{73}     \\
                                     & \textbf{0.9} & 1200         & 1200         & 1156         & 951           & 1060          & 1200          & 1200             & \textbf{77}    & \textbf{75}     \\ \hline
\multirow{3}{*}{\textbf{MNIST}}      & \textbf{0.7} & 1200         & 1200         & 1172         & 684           & 1068          & 1200          & 1200             & \textbf{75}    & \textbf{77}     \\
                                     & \textbf{0.8} & 1200         & 1200         & 1168         & 699           & 1060          & 1200          & 1200             & \textbf{71}    & \textbf{80}     \\
                                     & \textbf{0.9} & 1200         & 1200         & 1180         & 901           & 1068          & 1200          & 1200             & \textbf{70}    & \textbf{76}     \\ \hline
\multirow{3}{*}{\textbf{PROTEIN}}    & \textbf{0.7} & 1200         & 1200         & 1184         & 684           & 1044          & 1200          & 1200             & \textbf{78}    & \textbf{76}     \\
                                     & \textbf{0.8} & 1200         & 1200         & 1196         & 692           & 1048          & 1200          & 1200             & \textbf{77}    & \textbf{82}     \\
                                     & \textbf{0.9} & 1200         & 1200         & 1200         & 933           & 1044          & 1200          & 1200             & \textbf{82}    & \textbf{76}     \\ \hline
\multirow{3}{*}{\textbf{CIFAR10}}    & \textbf{0.7} & 1200         & 1200         & 1200         & 684           & 1044          & 1200          & 1200             & \textbf{76}    & \textbf{79}     \\
                                     & \textbf{0.8} & 1200         & 1200         & 1200         & 692           & 1048          & 1200          & 1200             & \textbf{84}    & \textbf{74}     \\
                                     & \textbf{0.9} & 1200         & 1200         & 1200         & 952           & 1068          & 1200          & 1200             & \textbf{78}    & \textbf{75}     \\ \hline
\multirow{3}{*}{\textbf{GoogleNews}} & \textbf{0.7} & 1200         & 1200         & 1144         & 864           & 1080          & 1200          & 1200             & \textbf{85}    & \textbf{77}     \\
                                     & \textbf{0.8} & 1200         & 1200         & 1157         & 904           & 1076          & 1200          & 1200             & \textbf{83}    & \textbf{83}     \\
                                     & \textbf{0.9} & 1200         & 1200         & 1156         & 1040          & 1068          & 1200          & 1200             & \textbf{79}    & \textbf{78}     \\ \hline
                                     \bottomrule
\end{tabular} 
}
\end{table}

%\footnote{***rewrite what we mean is that too small has noise for float32.}
 % The reason is that sufficiently small eigenvalues appear to be insignificant or even a noise \cite{james2013introduction,gautschi2011numerical} and may yield computational precision, which can be aligned with the IEEE Standard for Floating-Point Arithmetic \footnote{\url{https://stackoverflow.com/questions/588004/is-floating-point-math-broken}}. 
%------------------------------------------------------------------------------------------
\section{Conclusion}
\label{sec:conclusion} 
In this paper, we propose a novel Similarity Matrix Completion (SMC) framework and introduce two novel, scalable, and effective algorithms, SMCNN and SMCNmF, to solve the SMC problem. The algorithms exploit both the positive semidefinite (PSD) and low-rank properties of the similarity matrix to achieve fast optimization while maintaining stable estimation effectiveness. Theoretical analysis and empirical evaluations support the effectiveness, efficiency, and stability of proposed SMCNN and SMCNmF, shedding new light on various matrix learning problems in the real world. 

\begin{acks}
Acknowledgment.
\end{acks} 
 
\bibliographystyle{ACM-Reference-Format}
\bibliography{sample-base.bib}  

\clearpage
\appendix
\onecolumn

\section{Proofs} 
\label{App:proof}
% \begin{theorem} 
% \label{theorem:ref_RMSE}
% $||S^*-\hat{S}||_F^2 \leq ||S^*-S^0||_F^2$. The equality holds if and only if $S^0 = \hat{S}$. 
% \end{theorem}

\subsection{Proof of Lemma \ref{lemma:lemma1}}

% \begin{lemma}
% \label{lemma:ref_lemma1} 
% $\|S\|_{*} = \|V V^\top \|_*=  \min_{S=V V^\top } \|V\|_F \|V\|_F = \min_{S=V V^\top } \frac{1}{2}(\|V\|_F^2+\|V\|_F^2)$,   
% \end{lemma}  

\begin{proof}
\label{proof:lemma1} 

Following the Cauchy-Schawartz inequality \cite{horn1990cauchy,bhatia1995cauchy}, for any real square matrices $C$ and $D$, 
\begin{equation*}
\label{eq:CSIne}
\begin{aligned}
    \text{tr}(CD) & = \sum_ {ij} C_{ij}D_{ij}\leq (\sum_ {ij}C_{ij}^2)^{\frac{1}{2}}(\sum_{ij}D_{ij}^2)^{\frac{1}{2}}\\
    & = (\text{tr}(C^\top C))^{\frac{1}{2}}(\text{tr}(D^\top D))^{\frac{1}{2}} = \|C\|_F\|D\|_F.\\
\end{aligned}
\end{equation*}

The nuclear norm $\|S\|_*$ is the sum of the singular value of $S$ \cite{srebro2004maximum}. Perform the SVD on the real symmetric matrix $S$ as $S = A\Sigma A^\top$, where $\Sigma$ is the diagonal matrix of the singular values of $S$, and $A$ is the associated orthogonal matrix\cite{stewart1993early}. 
By using the Cholesky decomposition \cite{higham1990analysis} on $S = VV^\top$, we have: 

\begin{equation*}
\label{eq:CD}
\begin{aligned}
\|S\|_* & = \text{tr}(\Sigma) = \text{tr}(A^\top V V^\top A)=\text{tr}(AA^\top V V^\top ) \\
&\leq  \|V\|_F\|V\|_F \leq \frac{1}{2}(\|V\|_F^2 + \|V\|_F^2) 
\end{aligned}
\end{equation*}
The first inequality is based on the Cauchy-Schwarz inequality and the fact that $A$ is an orthogonal matrix. The equality holds when $V= A\sqrt{\Sigma}$.
\end{proof}

Recall $F_{1}(V) \equiv ||VV^{\top } -S^0\| _{F}^{2} +\lambda ||V||_{F}^{2}$. Define $\displaystyle Vec( \cdot )$ as a vectorize operator.

\subsection{Proof of Theorem \ref{theorem:guarantee}}
\begin{proof}
    % According to Lemma \ref{lemma:lemma1}, $||S||_{*} = \min  ||V||_{F}^{2}$. Then, the optimal solution for  $ F_{1}( V)$ is the optimal solution for (\ref{eq:problem_def_PSDLoR}).

    % According to theorem \ref{theorem:convergence}, Algorithm 1 can converge to a stationary point. 
    Recall the general convex Semi-definite programming problem can be formulated as \citep{journee2008low}:

\begin{equation}
\label{eq:sdp}
\begin{array}{ll}
\min _{S} & f(S) \\
\text { s.t. } & \operatorname{Tr}\left(A_i S\right)=b_i, A_i \in \mathbb{S}^n, b_i \in \mathbb{R}, i=1, \ldots, m \\
& S \succeq 0
\end{array}
\end{equation}

Taking $\forall i$, $A_i=0, b_i=0$, and $f(S) = ||S-S^0||^2_F$, we have:

\begin{equation} 
    \min_{S} \| S- S^0\|_F^2 \quad  s.t.~~ S \succeq 0,
\end{equation}

which is exactly Problem (\ref{eq:obj}) without the rank constraint.

According to Theorem 6 in \citep{journee2008low}, the stationary point in (\ref{eq:sdp}) with $\text{rank}(V)<r$ is the optimal solution for (\ref{eq:sdp}). Note that $\text{rank}(S)  = \text{rank}(VV^\top) \leq r$, which satisfying the rank constraint in (\ref{eq:obj}). Thus, it is also the optimal solution of (\ref{eq:obj}).

% Note that according to Sylvester’s rank inequality,
% 2rank$(V^*)-r \leq $ rank $(S^*)$, since  rank $(S^*) < r$, implying $\text{rank}(V^*)<r$.

%     Since that point is the optimal solution of $\nabla _{V} F_{1}(V) $, it is also the optimal solution for (\ref{eq:problem_def_PSDLoR}).
\end{proof}

\subsection{Proof of Theorem \ref{theorem:RMSE}}

\begin{proof}
\begin{align*}
 & ||S^{*} -\hat{S} ||_{F}^{2}\\
\leq  & ||S^{*} -\hat{S} ||_{F}^{2} -2\langle S^{*} -\hat{S}, S^0 -\hat{S} \rangle \\
\leq  & ||S^{*} -\hat{S} ||_{F}^{2} +||S^0 -\hat{S} ||_{F}^{2} -2\langle S^{*} -\hat{S}, S^0 -\hat{S} \rangle \\
= & ||(S^{*} -\hat{S} )-(S^0 -\hat{S} )||_{F}^{2}\\
= & ||S^{*} -S^0 ||_{F}^{2}
\end{align*}

Note that $\hat{S} = \mathcal{P}_C(S^0)$, where $C\subseteq \mathbb{R}^{n \times n}$ is the nonempty closed convex set of low-rank PSD algebraic variety, and $\mathcal{P}$ is the  Euclidean projection %ontoprojection operator that takes $S^0$ as input and returns the $\hat{S}$ such that $\hat{S}$ belongs to the $C$. 
onto $C$.
Then, $\langle S^* - \hat{S}, S^0-\hat{S} \rangle \leq 0$, by the various characterization of convex projection refer to Theorem 4.3-1\cite{ciarlet2013linear}. 
\end{proof} 

\textbf{Remark}: Theorem \ref{theorem:RMSE} provides a theoretical guarantee to ensure the estimated similarity matrix $S$ always approximates the unknown ground-truth similarity matrix $S^*$. Different from the previous work adopted Kolmogorov’s criterion \cite{staiger1998tight} to ensure the estimated similarity matrix is positive semi-definite and close to the unknown ground truth, we utilize different concepts, variational characterization of convex projection to ensure the estimated similarity matrix lay on both PSD and low-rank properties. 

\subsection{Proof of Theorem \ref{theorem:convergence}}
\begin{lemma}
\label{lemma:vec}
    $Vec\left( \nabla _{V} F_{1}\left( V^{t-1}\right)\right) = \nabla _{Vec( V)} F_{1}\left( Vec\left( V^{t-1}\right)\right)$
\end{lemma}

\begin{proof}
   Given an arbitrary function $F(V)$ with a matrix $V \in \mathbb{R}^{n\times r}$ as input. The gradient of $F(V)$ w.r.t. $V$ is the Jacobin matrix $J= \nabla_VF(V)\in \mathbb{R}^{n\times r}$, where each element is the partial derivative of $F(.)$ w.r.t. the corresponding element of $V$, that is, $J_{ij}=\frac{\partial F(.)_i}{\partial V_{j}}$. %\textcolor{blue}{?$J_{ij}=\frac{\partial F(.)_i}{\partial V_{j}}$.}
   
   The l.h.s. $Vec\left( \nabla _{V} F_{1}\left( V^{t-1}\right)\right) = Vec(J_{F_1(V)}) \in \mathbb{R} ^{nr \times 1}$,
   
   The r.h.s $\nabla _{Vec( V)} F_{1}\left( Vec\left( V^{t-1}\right)\right) = J_{F_1(Vec(V^{t-1}))} \in \mathbb{R}^{nr \times 1} $.

   Here, each element in the Jacobin matrix is calculated by the partial derivative of $F_1 (V)$ w.r.t. the corresponding element of $V$. As a result, the l.h.s. is the same as r.h.s. 
\end{proof}

\begin{lemma}
\label{lemma:smooth_part}
    $||( V^{t} V{^{t}}^{\top } V^{t} -V^{t-1} V{^{t-1}}^{\top } V^{t-1}) ||_{F} \leq 3G^{2}\left\Vert V^{t} -V^{t-1}\right\Vert _{F}$
\end{lemma}

\begin{proof}

\begin{align*}
 ||\left( V^{t} V{^{t}}^{\top } V^{t} -V^{t-1} V{^{t-1}}^{\top } V^{t-1}\right) ||_{F}\\
\leq ||\left( V^{t} V{^{t}}^{\top } V^{t} -V^{t} V{^{t}}^{\top } V^{t-1}\right) ||_{F} +||\left( V^{t} V{^{t}}^{\top } V^{t-1} -V^{t} V{^{t-1}}^{\top } V^{t-1}\right) ||_{F}\\
+||\left( V^{t} V{^{t-1}}^{\top } V^{t-1} -V^{t-1} V{^{t-1}}^{\top } V^{t-1}\right) ||_{F}\\
\leq \left(\left\Vert V^{t} V^{t}\right\Vert _{F} +\left\Vert V^{t} V^{t-1}\right\Vert _{F} +\left\Vert V^{t} V^{t}\right\Vert _{F}\right)\left\Vert V^{t} -V^{t-1}\right\Vert _{F}\\
\leq 3G^{2}\left\Vert V^{t} -V^{t-1}\right\Vert _{F}   
\end{align*}

The first inequality is based on the triangle inequality.
The second inequality is by the norm property, i.e., $||AB||\leq ||A||||B||$.

\end{proof}

\begin{lemma}
\label{lemma:smooth}
    $\left\Vert Vec\left( \nabla _{V} F_{1}\left( V^{t}\right)\right) -Vec\left( \nabla _{V} F_{1}\left( V^{t-1}\right)\right)\right\Vert _{2} \leq \left( 12G^{2} +2\lambda \right)\left\Vert V^{t} -V^{t-1}\right\Vert _{F}$ 
\end{lemma}

\begin{proof}
    Recall $\nabla _{V} F_{1}( V) \ =4\left( VV^{\top } -S^0\right) V+2\lambda V$, by using Lemma \ref{lemma:smooth_part}:

\begin{align*}
    \left\Vert Vec\left( \nabla _{V} F_{1}\left( V^{t}\right)\right) -Vec\left( \nabla _{V} F_{1}\left( V^{t-1}\right)\right)\right\Vert _{2}\\
=\left\Vert \left( \nabla _{V} F_{1}\left( V^{t}\right)\right) -\left( \nabla _{V} F_{1}\left( V^{t-1}\right)\right)\right\Vert _{F}\\
=||4\left( V^{t} V{^{t}}^{\top } V^{t} -V^{t-1} V{^{t-1}}^{\top } V^{t-1}\right) +2( \lambda I -S^0)\left( V^{t} -V^{t-1}\right) ||_{F}\\
\leq ||4\left( V^{t} V{^{t}}^{\top } V^{t} -V^{t-1} V{^{t-1}}^{\top } V^{t-1}\right) ||_{F}+||( 2\lambda I -2 S^0)\left( V^{t} -V^{t-1}\right) ||_{F}\\
\leq 12G^2\left\Vert V^{t} -V^{t-1}\right\Vert_{F} + 2\lambda \left\Vert V^{t} -V^{t-1}\right\Vert_{F} + 2 ||S^0|| \left\Vert V^{t} -V^{t-1}\right\Vert_{F}\\
\leq \left( 12G^{2} +2\lambda \right)\left\Vert V^{t} -V^{t-1}\right\Vert _{F}
\end{align*}

The second inequality is also based on the triangle inequality, Lemma \ref{lemma:smooth_part}, and norm inequality $||AB||\leq||A|||B||$. The last inequality is based on the assumption $\|S^0\| = \|VV^\top\| \leq G^2$.
\end{proof}

Now we begin our final proof:

\begin{proof}
According to Lemma \(\ref{lemma:vec}\), we have the vectorization operation for iteration \(t\) as follows:
\[
\text{Vec}\left( V^{t}\right) = \text{Vec}\left( V^{t-1}\right) - \gamma \text{Vec}\left( \nabla _{V} F_{1}\left( V^{t-1}\right)\right),
\]
where \(\gamma\) is the step size. Then, by invoking the descent lemma\footnote{\url{https://angms.science/doc/CVX/CVX_SufficientDecreaseLemma_GD.pdf}} and Lemma \(\ref{lemma:smooth}\), we obtain:
\[
F_{1}\left( V^{t}\right) - F_{1}\left( V^{t-1}\right) \leq -\frac{1}{12G^{2} + 2\lambda} \left\| \text{Vec}\left( \nabla _{V} F_{1}\left( V^{t-1}\right)\right) \right\|^{2},
\]
which leads to the following inequality after rearranging:
\[
\left( 12G^{2}+ 2\lambda \right) \left( F_{1}\left( V^{1}\right) - F_{1}\left( V^{t}\right) \right) \geq \sum \left\| \text{Vec}\left( \nabla _{V} F_{1}\left( V^{t-1}\right) \right) \right\|^{2}.
\]
Setting the step size \(\gamma\) = \(\frac{1}{6G^{2} + \lambda}\), we can conclude that:
\[
\min_{t} \left\| \text{Vec}\left( \nabla _{V} F_{1}\left( V^{t}\right) \right) \right\|^{2} \leq O\left(\frac{1}{T}\right),
\]
\end{proof} 
%\textcolor{blue}{? because in decent lemma, step-size $\frac{1}{L} = \gamma$, $\frac{1}{2L} = \frac{1}{12G^2+2\lambda}$, as a result stepsize $\gamma = \frac{1}{6G^2+\lambda}$?} \hl{yes!!!!}
 
%------------------------------------------------------------------------------------------------------------------------------------
\section{Detailed Description}
\label{App:complexity}  
\subsection{Detailed of Similarity Matrix Calculation}  
\label{app:detailedSS}
Similarity matrix \cite{manz1995analysis} provides a comprehensive representation of pairwise similarities between data samples. Each element corresponds to a similarity score, calculated by a specific similarity metric. Among the various similarity metrics, Cosine Similarity \cite{hamers1989similarity} stands out as a widely adopted choice :  \[ s_{ij}  = \frac{x_i^{\top} \cdot x_j}{\|x_i\| \cdot \|x_j\|},  \] where the column vectors $x_i, x_j\in \mathbb{R}^{d}$ denotes two data samples in $d$-dimensional space and $\|.\|$ denotes $\ell^2$-norm of a vector. The entry $s_{ij}$ in the $i$-th row $j$-th column of the similarity matrix $S\in \mathbb{R}^{n\times n}$ denotes the pairwise similarity scores between two data samples $x_i$ and $x_j$, where $n$ denotes the total number of data samples. The cosine similarity score varies in $[-1,1]$, where $1$ indicates two vectors are very similar with almost the same directions, and $-1$ indicates the two vectors are dissimilar with almost opposite directions, and $0$ indicates no clear similarity between two vectors with their directions are nearly unrelated to each other. In practice, to ensure the cosine similarity is always well-defined when either $x_i$ or $x_j$ is a zero vector, a small value $\epsilon$ is added to the zero vector \cite{schutze2008introduction,chowdhury2010introduction}, a.k.a, epsilon smoothing. Besides, the other series of applications \cite{buciu2004application,ailem2017non,blok2003probability} adjust the similarity score in $(0,1]$ to interpret it as the probabilities of similarity of pairwise data samples, which will not hurt the performance. %\footnote{*** could you conduct an exp to reveal  $[-1,1]$ has similar performance as $(0,1]$}

\textbf{Inaccurate Similarity Score}
Meanwhile, when the observation is incomplete, the initial similarity matrix $S^0$ is calculated by: \[  s^0_{ij}  = \frac{x'^{\top}_i \cdot x'_j}{\|x'_i\| \cdot \|x'_j\|},\] where $x'$ denotes the incomplete data samples, the initial inaccurate similarity score $s^0_{ij}$ is calculated by the overlapped features observed in both $x_i$ and $x_j$ with common indices. 

\subsection{Detailed Problem Formulation}
\label{app:dpf}
\subsubsection{Workflow}

\textbf{Step 1: Similarity Matrix Initialization}. Given the incomplete observed query database $Q'$ and complete observed search database $P$, we calculate the inaccurate initial similarity matrix $S^0$ (Refer to $ s^0_{ij}  = \frac{x'^{\top}_i \cdot x'_j}{\|x'_i\| \cdot \|x'_j\|}$, $x$ denotes $p$ or $q'$).

\textbf{Step 2: Similarity Matrix Completion (SMC)}. We formulate the SMC problem (Refer to problem \eqref{eq:problem_def_smc}, which aims to find an optimal estimated similarity matrix $\hat{S}$ to approximate the unknown ground truth $S^*$. To solve SMC problem, we propose SMCNN and SMCNmF to provide an effective and efficient result. (Refer to Section \ref{sec:method}.)

\textbf{Step 3: Estimated Similarity Matrix for Various Applications}. We adopt the typical similarity search task including textual datasets and visual datasets, on the estimated similarity matrix $\hat{S}$ to verify the performance of the proposed methods. (Refer to Section \ref{sec:experiments})

\subsubsection{Mathematical Formulation}
We further give a detailed description of the motivating scenario (Refer to Section \ref{subsec:pb}). We divided the entire similarity matrix $\hat{S}$ into four sub-matrices, denoting as $\hat{S} = \left[ \begin{array}{cc}
    S_{\text{pp}} & \hat{S}_{\text{pq}} \\
    \hat{S}_{\text{pq}}^\top & \hat{S}_{\text{qq}} 
\end{array}  \right] \in \mathbb{R}^{n \times n}$, where  $S_{\text{pp}} \in \mathbb{R}^{n_p\times n_p}$, $\hat{S}_{\text{pq}} \in \mathbb{R}^{n_p\times n_q}$, $\hat{S}_{\text{qq}} \in \mathbb{R}^{n_q\times n_q}$, and $n=n_p+n_q$. The incomplete observations are only in the query database, which means the similarity matrix $S_{\text{pp}}$ denoting the pairwise search samples is accurate and we only need to calibrate the inaccurate similarity matrices $\hat{S}_{\text{pq}}$ and $\hat{S}_{\text{qq}}$. Motivated by this observation, we aim to fully utilize the internal information of accurate similarity matrix $S_{\text{pp}}$ to guide the calibration process. 

Meanwhile, it is common that the size of the search database $P$ is much larger than the size of query database $Q$ \cite{graefe1993query}, denoting as $n_p>n_q$. In the evaluation part (refer to Section \ref{sec:experiments}), we simply adopt this setting and evaluate the performance.

\subsection{Detailed Complexity Analysis}
\label{app:complexity}
\noindent
\textbf{BPMF}\cite{mnih2007probabilistic}: BPMF needs $O((n^2r^2+nr^3)T)$ using Markov Chain Monte Carlo (MCMC) as the solver, where $O(n^2r^2)$ denotes MCMC sampling the latent factor from the posterior distribution iteratively and $O(nr^3)$ denotes updating hyperparameters in the Bayesian model, such as the variance of the latent factors. $T$ denotes the total number of iterations. 

\noindent
\textbf{FGSR}\cite{fan2019factor}: FGSR needs $O((n^2r)T)$ using the proximal alternating linearized algorithm coupled with iteratively re-weighted minimization. $T$ denotes the total number of iterations. 

\noindent
\textbf{CMC}\cite{li2015estimating}: CMC needs $O((n^3/y^2)T)$, where $O(n^3)$ denotes the SVD operation on $y$ sub-matrices with total $T$ iterations. 

\noindent
\textbf{DMC}\cite{li2020scalable}: DMC needs $O(n^3T)$, where $O(n^3)$ denotes the SVD operations on the entire dataset with total $T$ iterations.

\noindent
\textbf{FMC}\cite{ma2024fast}: FMC needs $O(n^3)$, where $O(n^3)$ denotes the SVD operations once on the sub-matrix calculated by the search dataset. Meanwhile, FMC needs $O(nT)$ to calculate the similarity vector, and $O((\log n) T)$ to calculate the Lagrange parameter.

%The computational complexity of matrix multiplication operation $VV^\top$ is $O(n^2r)$, matrix subtraction $VV^\top-S'$ is $O(n^2)$, scalar multiplication is $O(n^2)$. Matrix multiplication operation $V^\top V$ is $O(nr^2)$, Frobenius norm operation is $O(n^2)$, scalar division is $O(n^2)$, scalar multiplication $O(n^2)$, matrix subtraction is $O(nr)$, the computational complexity of projection operation is $P_{\Omega}$ is $O(n^2)$. Overall, the computational complexity is dominated by the matrix multiplication and Frobenius norm of approximately $O(n^2 r+ n r^2)$. The storage complexity is $O(nr+\|\Omega\|_0)$ to keep the estimated and $\hat{V}$ and the projection function. 

\noindent
\textbf{NNFN}\cite{wang2021scalable}: NNFN needs $O((n^2r_k)T)$, where $r_k$ denotes the estimated rank at $k$-th iteration of SVD operation with total $T$ iterations.

\noindent
\textbf{facNNFN} \cite{wang2021scalable}: facNNFN needs $O((n^2r+nr^2)T)$,where the gradient is calculated as \cite{wang2021scalable}:
\begin{equation*}
\label{eq:facNNFN_grad}
\begin{aligned}
&\nabla_W F(W,H) = [\nabla_W F(WH)]H + \lambda W - \lambda \frac{ W(H^\top H)}{\|WH^\top\|_F} \\
&\nabla_H F(W,H) = [\nabla_H F(WH)]^\top W + \lambda H - \lambda \frac{ H(W^\top W)}{\|WH^\top\|_F} \\
\end{aligned}
\end{equation*}

It needs $O(n^2r)$ to calculate the multiplication of the partially observed matrix $\nabla_W F(W,H) H$.% and $O(nr^2)$ to calculate $\frac{ W(H^\top H)}{\|WH^\top\|_F}$. 
The storage complexity is $O(2nr)$ to keep the estimated $\hat{W}$ and $\hat{H}$.

\noindent
\textbf{SMCNN}
The gradient is calculated as:
\begin{equation*}
\label{eq:problem_def_NN1_grad} 
     \nabla_V F_1(V) = 4(VV^\top - S^0) V + 2\lambda V  
\end{equation*}

Since SMCNN only updates the matrix $V$ in each iteration to estimate the similarity matrix $S$, it requires $O(n^2r)$ to calculate the $4(VV^\top - S^0) V$. The storage complexity is $O(nr)$ to keep the estimated $\hat{V}$.% and the indices of complete and incomplete observations. 
%The computational complexity of matrix multiplication operation $VV^\top$ is $O(n^2 r)$, matrix subtraction $VV^\top-S'$ required $O(n^2)$, scalar multiplication is $O(n^2)$, the computational complexity of projection operation is $P_{\Omega}$ is $O(n^2)$. Since we adopted the projection $P_{\Omega}$ on matrix $\hat{S}$ to preserve the complete observations. Overall, the computational complexity is dominated by the matrix multiplication operation, resulting in the computational complexity of approximately $O(\|\Omega\|_0)$. The storage complexity is $O(nr+\|\Omega\|_0)$ to keep the estimated $\hat{V}$ and the projection function. 

\noindent
\textbf{SMCNmF}
The gradient is calculated as: %\footnote{*** what if $vv^\top=0$}
\begin{equation*}
\label{eq:problem_def_NF2_grad} 
     \nabla_V F_2(V) = 4(VV^\top - S^0) V + 2\lambda V - 2 \lambda \frac{V(V^\top V)}{\|VV^\top\|_F} \\  
\end{equation*}

Similar to SMCNN, SMCNmF required $O(n^2r)$ to calculate the $4(VV^\top - S^0) V$.
%, and $O(nr^2)$ to calculate the $\frac{V(V^\top V)}{\|VV^\top\|_F}$. 
The storage complexity is $O(nr)$ to keep the estimated $\hat{V}$.
 
\subsubsection{Comparison $r_k$ and $r$ in NNFN.}
In Section \ref{sec:analysis}, we mention that usually $ r_k $ ($n \geq r_k \geq r$) is the rank of $\hat{S}$ estimated at each iteration. The nuclear norm regularizer encourages the estimated matrix to be low-rank by promoting sparsity in its singular values. As a result, the rank of the estimated matrix is effectively controlled by adjusting the value of the nuclear norm regularizer parameter.  

In Section \ref{sec:experiments}, we keep the default setting as the original paper to provide a fair comparison. For example, given a matrix $X$, let its SVD be $U Diag(\delta(X))V^\top$. Then, the SVT uses the function $\max(0, \delta(X)-\tau)$ to control the rank of the estimated matrix. In SVT, the default $\tau = 0.0001$ is set for an arbitrary matrix, which may not lie on the PSD property. Unlike the arbitrary matrix, the similarity matrix lies on the symmetric and PSD property, which means a small $\tau$ may not influence the rank of the estimated matrix. 

To provide a distinguished case on the symmetric and PSD matrix $\hat{S}$, we set a larger $\tau=0.001$ and $\tau=0.01$. Table \ref{tab:svt} shows the comparison results of rank and RMSE.

\begin{table}[!htb]
\centering 
\caption{Comparision of rank($\hat{S}$) and RMSE with $m=1,000$ search candidates and $n=200$ query items, where $r_k$ is correlated with various $\tau$ with various missing ratio $\rho$, fixed $\lambda=0.001$, $\gamma=0.001$, and $T=10,000$.}
\label{tab:svt} 
\resizebox{0.6\columnwidth}{!}{   
\begin{tabular}{l|l|lll|lll}
\toprule
\hline
\multicolumn{1}{l|}{\multirow{2}{*}{\textbf{Datasets}}} & \multirow{2}{*}{\textbf{\begin{tabular}[c]{@{}l@{}}Missing\\ Ratio\end{tabular}}} & \multicolumn{3}{c|}{\textbf{rank($\hat{S}$)}}    & \multicolumn{3}{c}{\textbf{RMSE}}       \\ \cline{3-8} 
\multicolumn{1}{l|}{}        &                    & \textbf{0.0001} & \textbf{0.001} & \textbf{0.01} & \textbf{0.0001} & \textbf{0.001} & \textbf{0.01} \\ \hline
\multirow{3}{*}{\textbf{ImageNet}}    & 0.7                & 1200   & 1191  & 954  & 0.9055          & \underline {0.7124}         & \underline{0.6186}        \\
           & 0.8                & 1200   & 1122  & 719  & 0.7488          & 0.7234         & 0.7024        \\
           & 0.9                & 1200   & 1120  & 698  & 0.8050          & 0.7576         & 0.7011        \\ \hline
\multirow{3}{*}{\textbf{MNIST}}       & 0.7                & 1200   & 1190  & 942  & 1.6331          & 0.7048         & 0.7473        \\
           & 0.8                & 1200   & 1129  & 703  & 1.0880          & 0.7986         & 0.7286        \\
           & 0.9                & 1200   & 1012  & 644  & 0.5204          & 0.7497         & 0.6497        \\ \hline
\multirow{3}{*}{\textbf{PROTEIN}}     & 0.7                & 1200   & 1183  & 882  & 1.0053          & 0.7186         & 0.6186        \\
           & 0.8                & 1200   & 1012  & 782  & 0.7962          & 0.6234         & 0.6724        \\
           & 0.9                & 1200   & 1002  & 608  & 0.4688          & 0.7076         & 0.8011        \\ \hline
\multirow{3}{*}{\textbf{CIFAR10}}     & 0.7                & 1200   & 1148  & 822  & 1.4627          & 0.9480         & 0.8473        \\
           & 0.8                & 1200   & 1080  & 780  & 2.0510          & 1.1132         & 1.1028        \\
           & 0.9                & 1200   & 1020  & 683  & 0.6322          & 0.5720         & 0.5995      \\  \hline
\multirow{3}{*}{\textbf{GoogleNews}}           & 0.7                & 1200   & 1190  & 982  & 0.7795          & 0.6240         & 0.5972        \\
           & 0.8                & 1200   & 1120  & 747  & 0.7190          & 0.7211         & 0.5986        \\
           & 0.9                & 1200   & 1020  & 640  & 0.8352          & 0.7473         & 0.6497        \\ \hline
\bottomrule
\end{tabular} 
}
\end{table}

We find that when the $\tau$ is large, the rank($\hat{S}$) is small. Since the matrix $S$ is positive semi-definite, most of the eigenvalues may be non-negative, which means a larger $\tau$ is needed to reduce the rank of $\hat{S}$. It is obvious that all the $r_k > r$ always holds for all datasets. Though the RMSE on ImageNet is better than SMCNN/SMCNmF, the computational complexity $O(n^2r_k)$ is still larger than SMCNN/SMcNmF $O(n^2r)$.

\subsection{Detailed Description of Ablation Study}  
\label{App:ab}
%Table \ref{tab:RMSE:abmethodsall} shows all the variants, including the property, regularizer, and solver. 
%\input{Table/Ablation/DetailedAbla}

\subsubsection{SMC\_F, SMC\_NR, SMC\_GD} 
For ease of understanding the comparisons between using the nuclear norm and the Frobinus norm, we design three variants of Eq.\eqref{eq:problem_def_PSDLoR}. 

\textbf{SMC\_F}: Instead of using the nuclear norm, we adopt the Frobenius norm as the regularizer. 
\begin{equation*}
\label{eq:SSV1}
\begin{aligned}
\textbf{SMC\_F}:& \min_{V} \|V V^\top -S^0\|_F^2 + \lambda \|VV^\top\|_F^2 \\
    %& s.t.~~0<\{ V V^\top\}_{ij}\leq 1,\\   
   % & s.t.~~ V V^\top \succeq 0
\end{aligned}
\end{equation*}
where the gradient is:
\begin{equation*}
\label{eq:SSV1_grad}
\begin{aligned}
\nabla_V F_{SMC\_F}(V) = 4(VV^\top - S^0) V + 2\lambda V  V^\top  V
\end{aligned}
\end{equation*}

To keep consistent with the SMCNN/SMCNmF, we use the gradient descent algorithm to solve this problem. The entire process is shown in Algorithm \ref{alg:SMCF}. 
\begin{algorithm}[!htb]
\caption{SMC with Frobenius norm (\textbf{SMC\_F})} 
\label{alg:SMCF} 
  \begin{algorithmic}[1] 
    \Require 
        $V^0 \in \mathbb{R}^{n\times r}$: randomly initialized $V^0$ with rank $r$;  
        $\gamma$: stepsize; 
        $\lambda$: weighted parameter. 
        $T$: iterations 
    \Ensure 
        $\hat{V}\in \mathbb{R}^{n \times r}$: Estimated $V$ to approximate $\hat{S}=\hat{V} \hat{V}^\top$; 
        $\hat{S}\in \mathbb{R}^{n \times n}$: Estimated similarity matrix $\hat{S}$. 
    \State Define $\min_V F_{SMC\_F}(V) \equiv \min_{V} \|V V^\top -S^0\|_F^2 + \lambda \|VV^\top\|_F $;
     \For {t = 1,..,T}
     \State update $V^t$ by updating $\nabla_V F_{SMC\_F}(V)$;
     \EndFor
    \State Calculate $\hat{S} = \hat{V} \hat{V}^\top $;
    \State \Return $ \hat{S}$.
    %\State Restrict $0< \{V V^\top\}_{ij} \leq 1$ via $\hat{S} = \min(\max(\hat{S}, \epsilon), 1)$;
   % \State \Return $ \hat{S} = P_{\Omega}(\hat{S}-S^0)$; 
  \end{algorithmic} 
\end{algorithm}

\textbf{SMC\_NR}: We ignore the regularizer term directly.
\begin{equation*}
\label{eq:SSV2}
\begin{aligned}
   \textbf{SMC\_NR}:  & \min_{V} \|V V^\top -S^0\|_F^2 \\
 %  & s.t.~~0<\{ V V^\top\}_{ij}\leq 1,\\  
  %  & s.t.~~ V V^\top \succeq 0,
\end{aligned}
\end{equation*}
where the gradient is:
\begin{equation*}
\label{eq:SSV2_grad}
\begin{aligned}
\nabla_V F_{SMC\_NR}(V) = 4(VV^\top - S^0) V
\end{aligned}
\end{equation*}

Similar to $\textbf{SMC\_F}$, we use the gradient descent algorithm and the entire process shown in Algorithm \ref{alg:SMCNR}
\begin{algorithm}[!htb]
\caption{SMC with No Regularizer (\textbf{SMC\_NR})} 
\label{alg:SMCNR} 
  \begin{algorithmic}[1] 
    \Require 
        $V^0 \in \mathbb{R}^{n\times r}$: randomly initialized $V^0$ with rank $r$;  
        $\gamma$: stepsize; 
        $\lambda$: weighted parameter. 
        $T$: iterations 
    \Ensure 
        $\hat{V}\in \mathbb{R}^{n \times r}$: Estimated $V$ to approximate $\hat{S}=\hat{V} \hat{V}^\top$; 
        $\hat{S}\in \mathbb{R}^{n \times n}$: Estimated similarity matrix $\hat{S}$. 
    \State Define $\min_V F_{SMC\_NR}(V) \equiv \min_{V} \|V V^\top -S^0\|_F^2$; % + \lambda \|VV^\top\|_* $;
     \For {t = 1,..,T}
     \State update $V^t$ by updating $\nabla_V F_{SMC\_NR}(V)$;
     \EndFor
    \State Calculate $\hat{S} = \hat{V} \hat{V}^\top $;
    \State \Return $ \hat{S}$.
    %\State Restrict $0< \{V V^\top\}_{ij} \leq 1$ via $\hat{S} = \min(\max(\hat{S}, \epsilon), 1)$;
    %\State \Return $ \hat{S} = P_{\Omega}(\hat{S}-S^0)$; 
  \end{algorithmic} 
\end{algorithm} 

\textbf{SMC\_GD}: We adopt the gradient descent algorithm to solve the Singular Value Thresholding (SVT) operation \cite{cai2010singular} in Problem \ref{eq:problem_def_PSDLoR}. 

\begin{equation*}
\label{eq:SSV6}
\begin{aligned}
   \textbf{SMC\_GD}:  & \min_{V} \|VV^\top -S^0\|_F^2 + \lambda\|VV^\top\|_* \\ 
\end{aligned}
\end{equation*}
where the gradient is :
\begin{equation}
\label{eq:SSV6_grad}
    \nabla_V F_{SMC\_GD}(V) = 4(VV^\top - S^0)V  + \lambda \mathcal{P} (\|VV^\top\|_*),
\end{equation}
where $\mathcal{P}$ denotes the proximal nuclear norm function via the SVT operation, which means the summation of all the eigenvalues (detailed in Lines 4-6). 

\begin{algorithm}[!htb]
\caption{SMC with GD (\textbf{SMC\_GD})} 
\label{alg:SMCGD} 
  \begin{algorithmic}[1] 
    \Require 
        $V\in \mathbb{R}^{n\times r}$: randomly initialized $V^0$;   
        $\tau$: thresholding parameter;
        $\lambda$: weighted parameter;        
        $T$: iterations 
    \Ensure  
        $\hat{V}\in \mathbb{R}^{n \times r}$: Estimated $\hat{V}$ to approximate the similarity matrix $\hat{S}=\hat{V}\hat{V}^\top$. 
    \State Define $\min_V F_{SMC\_GD}(S) \equiv \min_{S} \|VV^\top -S^0\|_F^2  + \lambda \|VV^\top\|_* $;
     \For {t = 1,..,T}
    \State Update $V^t$ by updating $\nabla_V F_{SMC\_GD} (V)$; 
    \State Calculate SVD on $S = U \Sigma W$;
    \State Update $\Sigma^t$ = $\max(\Sigma^t-\tau, 0)$;
     \State Reconstruct $S^t = U^t  \hat{\Sigma}^t  W^{t\top}$;
     \State Cholesky Factorization on $S^t=V^tV^{t\top}$.
     \EndFor 
    \State \Return $ \hat{S}$. 
  \end{algorithmic} 
\end{algorithm}

\subsubsection{MC\_ON, MC\_fON}
We apply the vanilla nuclear norm in the objective function and use different solvers to estimate the similarity matrix.

\noindent
\textbf{MC\_ON}: We directly used the original nuclear norm as the regularizer. 
\begin{equation*}
\label{eq:MCON}
\begin{aligned}
\textbf{MC\_ON}:& \min_{S} \|S -S^0\|_F^2 + \lambda \|S\|_* \\
%    & s.t.~~0<S_{ij}\leq 1,\\  
& s.t.~~ S \succeq 0
\end{aligned}
\end{equation*}

\begin{algorithm}[!htb]
\caption{MC with Original Nuclear norm (\textbf{MC\_ON})} 
\label{alg:MCOM} 
  \begin{algorithmic}[1] 
    \Require 
        $S^0 \in \mathbb{R}^{n\times n}$: input incomplete similarity matrix
        $\tau$: thresholding value to make eigenvalues are all positive;  
        $T$: iterations 
    \Ensure  
        $\hat{S}\in \mathbb{R}^{n \times n}$: Estimated similarity matrix $\hat{S}$. 
    \State Define $\min_V F_{MC\_ON}(S) \equiv \min_{S} \|S -S^0\|_F^2 + \lambda \|S\|_*$;
     \For {t = 1,..,T}
     \State Apply SVD on $S = U\Sigma V^\top$
     \State  Apply thresholding function $\Sigma_+ =\max (\Sigma, \tau)$; 
      \State  Update $S = U \max (\Sigma, \tau) V^\top$;   
     \EndFor
     \State \Return $ \hat{S}$.
%   \State \Return $\hat{S}$;
 %   \State Restrict $0< \hat{S}_{ij} \leq 1$ via $\hat{S} = \min(\max(\hat{S}, \epsilon), 1)$;
 %   \State \Return $ \hat{S} = P_{\Omega}(\hat{S}-S^0)$; 
  \end{algorithmic} 
\end{algorithm}

\noindent
\textbf{MC\_fON}: We adopt the Matrix Factorization rather than Cholesky Factorization.  
\begin{equation*}
\label{eq:MCfON}
\begin{aligned}
\textbf{MC\_fON}:& \min_{W,H} \|W H^\top -S^0\|_F^2 \\
    %& s.t.~~0<\{ W H^\top\}_{ij}\leq 1,\\  
    & s.t.~~\{ W H^\top\} \succeq 0,\\
\end{aligned}
\end{equation*}
where the gradient is:
\begin{equation*}
\label{eq:MCfON_grad}
\begin{aligned}
\nabla_W F_{MC\_fON}(W,H) = 2(WH^\top - S^0) H,  \quad  \nabla_H F_{MC\_fON}(W,H) = 2(WH^\top - S^0) W 
\end{aligned}
\end{equation*}

To keep consistent with the main algorithm, we use the gradient descent algorithm to solve this problem. The entire process is shown in Algorithm \ref{alg:fMCOM}. 
\begin{algorithm}[!htb]
\caption{MC with Matrix Factorization (\textbf{MC\_fON})} 
\label{alg:fMCOM} 
  \begin{algorithmic}[1] 
    \Require 
        $W^0,H^0 \in \mathbb{R}^{n\times r}$: randomly initialized $W^0$, $H^0$ with rank $r$;  
        $\gamma$: stepsize; 
        $\lambda$: weighted parameter. 
        $T$: iterations 
    \Ensure 
        $\hat{W},\hat{H}\in \mathbb{R}^{n \times r}$: Estimated $W$,$H$ to approximate $\hat{S}=\hat{W} \hat{H}^\top$; 
        $\hat{S}\in \mathbb{R}^{n \times n}$: Estimated similarity matrix $\hat{S}$. 
    \State Define $\min_{W,H} F_{MC\_fON}(W,H) \equiv \min_{W,H} \|W H^\top -S^0\|_F^2$;
     \For {t = 1,..,T}
     \State update $W^t$ by updating $\nabla_W F_{MC\_fON}(W,H)$;
     \State update $H^t$ by updating $\nabla_H F_{MC\_fON}(W,H)$;
     \EndFor
    \State Calculate $\hat{S} = \hat{W} \hat{H}^\top $;
    \State \Return $ \hat{S}$.
%    \State Restrict $0< \{W H^\top\}_{ij} \leq 1$ via $\hat{S} = \min(\max(\hat{S}, \epsilon), 1)$;
%    \State \Return $ \hat{S} = P_{\Omega}(\hat{S}-S^0)$; 
  \end{algorithmic} 
\end{algorithm}

%Table \ref{tab:ab:Q2S2} - \ref{tab:ab:Q2S5} show the ablation study results, including the RMSE, Recall @top20\%, and nDCG with various missing ratio $\rho$, and fixed rank $r$=100, $\lambda$=0.001, $\gamma$=0.001, and maximum iteration $T=10,000$. 
 
%\input{Table/Ablation/Q2S2}
%\input{Table/Ablation/Q2S3}
%\input{Table/Ablation/Q2S4}
%\input{Table/Ablation/Q2S5}

\clearpage
%------------------------------------------------------------------------------------------------------------------------------------
%Experimental Results 
\section{Extended Evaluation}
\label{sec:exp}  
 
\subsection{Extensive Comparison about SMCNN vs. SMCNmF}
\label{App:NNvsNmF}
As mentioned in Section \ref{sec:smcnmf}, the regularizer that satisfies the adaptivity and shrinkage properties can perform better. Specifically, SMCNmF shows that a novel regularizer, i.e., $R(V)=\|VV^\top\|_*-\|VV^\top\|_F$, applies adaptive shrinkage for singular values provably. SMCNmF also lies on a matrix factored form that allows fast optimization by a general learning-based algorithm with a sound theoretical guarantee.

Table \ref{tab:NNvsNmF} shows that SMCNmF performs better than SMCNN with specific parameter settings considering the SMC problem. Meanwhile, the estimated similarity matrix can be better used for the downstream task on various datasets. 

 \begin{table}[!htb]
\centering
\scriptsize
\caption{Comparisons of SMCNN vs. SMCNmF considering RMSE, Recall/nDCG@top20\% with various missing ratio $\rho$, various search and query samples, and fixed rank $r = 100$, $\lambda = 0.001$, $\gamma = 0.001$, and $T = 10,000$.} 
\label{tab:NNvsNmF}  
\resizebox{0.8\columnwidth}{!}{   
\begin{tabular}{l|l|l|lll|lll}
\toprule
\hline
\multicolumn{3}{l}{\textbf{Datasets/Methods}}                              & \multicolumn{3}{c}{\textbf{$\rho=0.7$}}    & \multicolumn{3}{c}{\textbf{$\rho=0.8$}}    \\ \hline
\multicolumn{1}{l|}{\multirow{6}{*}{\textbf{\begin{tabular}[c]{@{}l@{}}$n=2,000$\\ $m=200$\end{tabular}}}} & \textbf{}          & \textbf{}        & \textbf{RMSE}   & \textbf{Recall} & \textbf{nDCG}   & \textbf{RMSE}   & \textbf{Recall} & \textbf{nDCG}   \\ \cline{4-9}
\multicolumn{1}{l|}{}                      & \multirow{2}{*}{\textbf{MNIST}}    & \textbf{SMCNN}   & 0.7299  & 0.6938  & \textbf{0.7156} & 0.6703  & \textbf{0.7384} & 0.8954  \\
\multicolumn{1}{l|}{}                      &            & \textbf{SMCNmF}  & \textbf{0.7257} & \textbf{0.7063} & 0.7086  & \textbf{0.6703} & 0.7383  & \textbf{0.8965} \\ \cline{2-9}
\multicolumn{1}{l|}{}                      &            & \textbf{}        & \multicolumn{3}{c}{\textbf{$\rho=0.7$}}    & \multicolumn{3}{c}{\textbf{$\rho=0.8$}}    \\ \cline{4-9}
\multicolumn{1}{l|}{}                      & \multirow{2}{*}{\textbf{PROTEIN}}  & \textbf{SMCNN}   & 0.7280  & 0.6652  & \textbf{0.7978} & 0.6589  & 0.7167  & 0.6983  \\
\multicolumn{1}{l|}{}                      &            & \textbf{SMCNmF}  & \textbf{0.7256} & \textbf{0.6588} & 0.7467  & \textbf{0.6465} & \textbf{0.7375} & \textbf{0.6987} \\ \hline \bottomrule  
\end{tabular} 
}
\end{table}

% RMSE, Recall, nDCG 
\subsection{Effectiveness Evaluation}
\label{App:RMSE}
Table \ref{tab:rmse:s2q2}-\ref{tab:rmse:s3q2} show the remaining results about effectiveness. 
 \begin{table}[!htb]
\centering
\scriptsize
%\huge
\caption{Comparisons of RMSE,  Recall/nDCG@top20\% with $m=2,000$ search candidates and $n=200$ query items, where missing ratio $\rho =\{0.7,0.8,0.9\}$, rank $r = 100$, $\lambda = 0.001$, $\gamma = 0.001$, and $T = 10,000$. The top 2 performances are highlighted in bold.} 
\label{tab:rmse:s2q2}  
\resizebox{0.8\columnwidth}{!}{  
\begin{tabular}{ccc|ccc|ccc|ccc}
\toprule
\hline
\multicolumn{3}{c|}{\multirow{2}{*}{\textbf{Dataset/Method}}}                                                                      & \multicolumn{3}{c|}{\textbf{0.7}}                   & \multicolumn{3}{c|}{\textbf{0.8}}                   & \multicolumn{3}{c}{\textbf{0.9}}                    \\ \cline{4-12} 
\multicolumn{3}{c|}{}                                                                                                              & \textbf{RMSE}   & \textbf{Recall} & \textbf{nDCG}   & \textbf{RMSE}   & \textbf{Recall} & \textbf{nDCG}   & \textbf{RMSE}   & \textbf{Recall} & \textbf{nDCG}   \\ \hline
\multicolumn{1}{c|}{\multirow{12}{*}{\textbf{GoogleNews}}} & \multicolumn{1}{c|}{\multirow{4}{*}{\textbf{MVI}}} & \textbf{MEAN}\cite{kim2004reuse}   & 2.2060          & 0.4877          & 0.6994          & 3.8500          & 0.4667          & 0.6991          & 2.4722          & 0.5313          & 0.7963          \\
\multicolumn{1}{c|}{}                                      & \multicolumn{1}{c|}{}                              & \textbf{LR}\cite{seber2012linear}     & 2.4477          & 0.4467          & 0.6670          & 1.0470          & 0.5542          & 0.6993          & 1.5358          & 0.5604          & 0.7728          \\
\multicolumn{1}{c|}{}                                      & \multicolumn{1}{c|}{}                              & \textbf{GR}\cite{balzano2010online}     & 1.1630          & 0.5750          & 0.6560          & 1.9693          & 0.4250          & 0.6987          & 1.9013          & 0.6875          & 0.7602          \\
\multicolumn{1}{c|}{}                                      & \multicolumn{1}{c|}{}                              & \textbf{KFMC}\cite{fan2019online}   & 1.1889          & 0.4692          & 0.6865          & 1.6920          & 0.5021          & 0.6991          & 1.2918          & 0.6042          & 0.7614          \\
\hline
\bottomrule
\end{tabular} 
}
\end{table}
 \begin{table}[!htb]
\centering
\scriptsize
\caption{Comparisons of RMSE, Recall/nDCG@top20\% with $m=3,000$ search candidates and $n=200$ query items, where missing ratio $\rho =\{0.7,0.8,0.9\}$, rank $r = 100$, $\lambda = 0.001$, $\gamma = 0.001$, and $T = 10,000$. The top 2 performances are highlighted in bold.} 
\label{tab:rmse:s3q2}  
\resizebox{0.8\columnwidth}{!}{ 
\begin{tabular}{cc|c|ccc|ccc|ccc}
\toprule
\hline
\multicolumn{3}{c|}{\multirow{2}{*}{\textbf{Dataset/Method}}}                      & \multicolumn{3}{c|}{\textbf{0.7}}                    & \multicolumn{3}{c|}{\textbf{0.8}}                    & \multicolumn{3}{c}{\textbf{0.9}}                \\ \cline{4-12}
\multicolumn{3}{c|}{}                                                                                         & \textbf{RMSE}   & \textbf{Recall} & \textbf{nDCG}   & \textbf{RMSE}   & \textbf{Recall} & \textbf{nDCG}   & \multicolumn{1}{c}{\textbf{RMSE}} & \multicolumn{1}{c}{\textbf{Recall}} & \multicolumn{1}{c}{\textbf{nDCG}} \\ \hline
\multicolumn{1}{c|}{\multirow{12}{*}{\textbf{GoogleNews}}} & \multirow{4}{*}{\textbf{MVI}} & \textbf{MEAN}\cite{kim2004reuse}   & 1.2304          & 0.7172          & 0.7849          & 1.1189          & 0.6905          & 0.8178          & 1.0655          & 0.6582          & 0.6775          \\
\multicolumn{1}{c|}{}                                     &                               & \textbf{LR}\cite{seber2012linear}     & 1.3688          & 0.6573          & 0.7639          & 1.1774          & 0.6436          & 0.7404          & 1.1036          & 0.6291          & 0.6717          \\
\multicolumn{1}{c|}{}                                     &                               & \textbf{GR}\cite{balzano2010online}     & 1.4860          & 0.6677          & 0.7850          & 2.3145          & 0.6348          & 0.7455          & 3.1756          & 0.6116          & 0.6712          \\
\multicolumn{1}{c|}{}                                     &                               & \textbf{KFMC}\cite{fan2019online}   & 1.3280          & 0.6787          & 0.7043          & 1.1758          & 0.6499          & 0.7720          & 2.1283          & 0.6115          & 0.6735          \\
\hline
\bottomrule
\end{tabular} 
}
\end{table}

\clearpage
\subsection{Efficiency Evaluation}
\label{App:Time}
\subsubsection{Remaining Results of Efficiency}
Fig.\ref{fig:time:MNIST}-Fig.\ref{fig:time:Google} show the remaining results about Efficiency with various rank $r$ on various datasets with search candidates $n=1,000$ and query items $n=200$. 

\begin{figure*}[!htb]
    \centering    
    \hspace{-1.8cm}
      \subfigure[(a) MNIST, $r$=50]{
   \includegraphics[width =0.35\columnwidth, trim=15 0 30 10, clip]{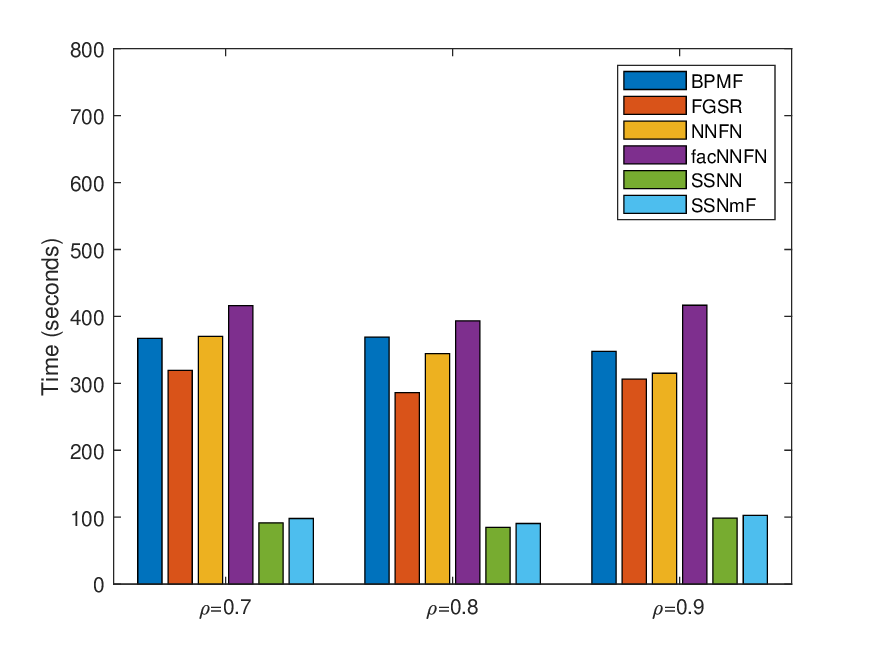}}
      \subfigure[(b) MNIST, $r$=100]{
   \includegraphics[width =0.35\columnwidth, trim=15 0 30 10, clip]{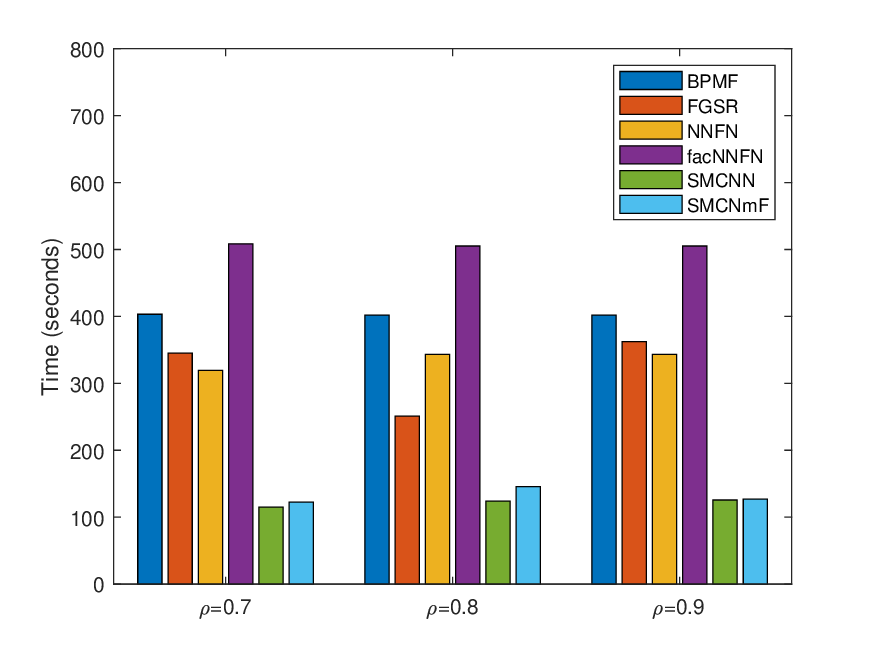}}
      \subfigure[(c) MNIST, $r$=200]{
   \includegraphics[width =0.35\columnwidth, trim=15 0 30 10, clip]{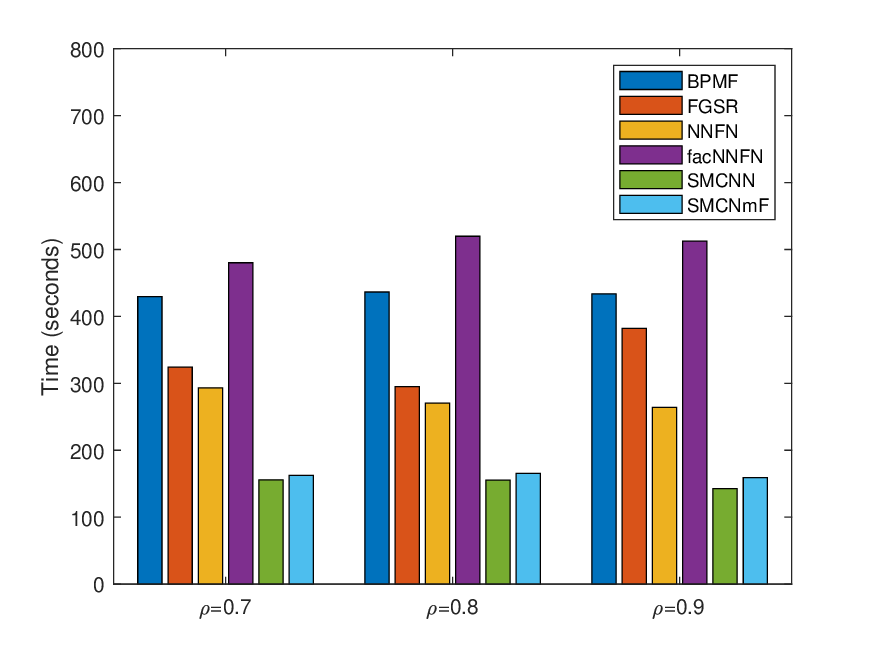}}
       \hspace{-2cm}
    \caption{Running Time versus $\rho$ on MNIST dataset, with $m=1,000$ search candidates and $n=200$ query items, with rank $r=\{50, 100, 200\}$, $\lambda = 0.001$, $\gamma = 0.001$, $T=10,000$. (Best viewed in $\times$ 2 sized color pdf file)}
    \label{fig:time:MNIST}
\end{figure*} 
\begin{figure*}[!htb]
    \centering    
    \hspace{-1.8cm}
      \subfigure[(a) PROTEIN, $r$=50]{
   \includegraphics[width =0.35\columnwidth, trim=15 0 30 10, clip]{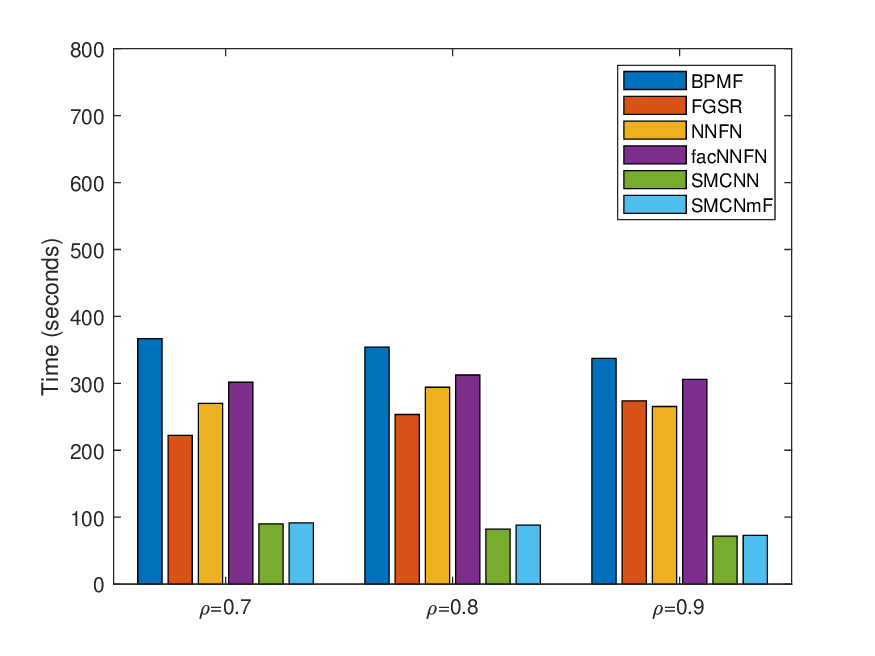}}
      \subfigure[(b) PROTEIN, $r$=100]{
   \includegraphics[width =0.35\columnwidth, trim=15 0 30 10, clip]{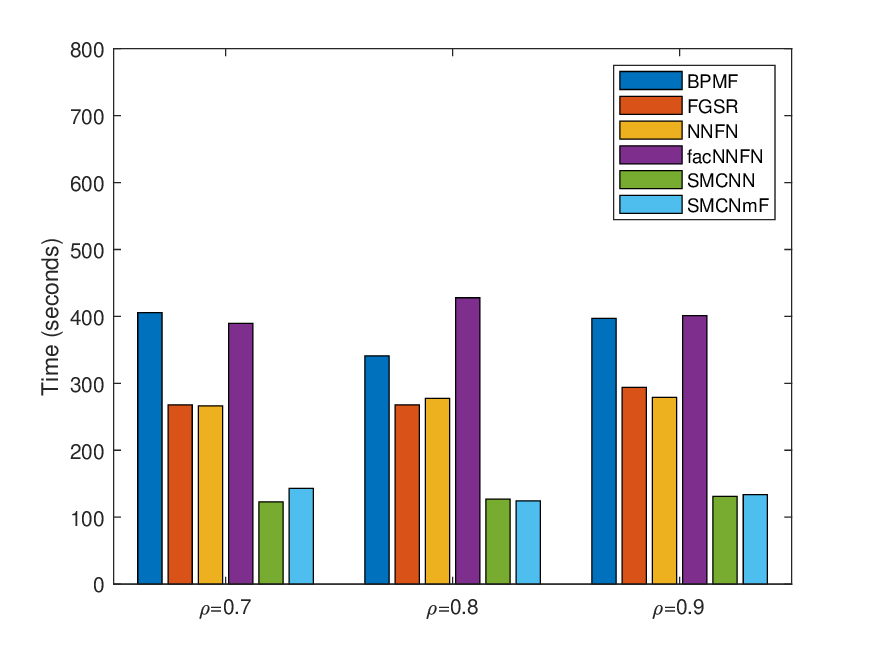}}
      \subfigure[(c) PROTEIN, $r$=200]{
   \includegraphics[width =0.35\columnwidth, trim=15 0 30 10, clip]{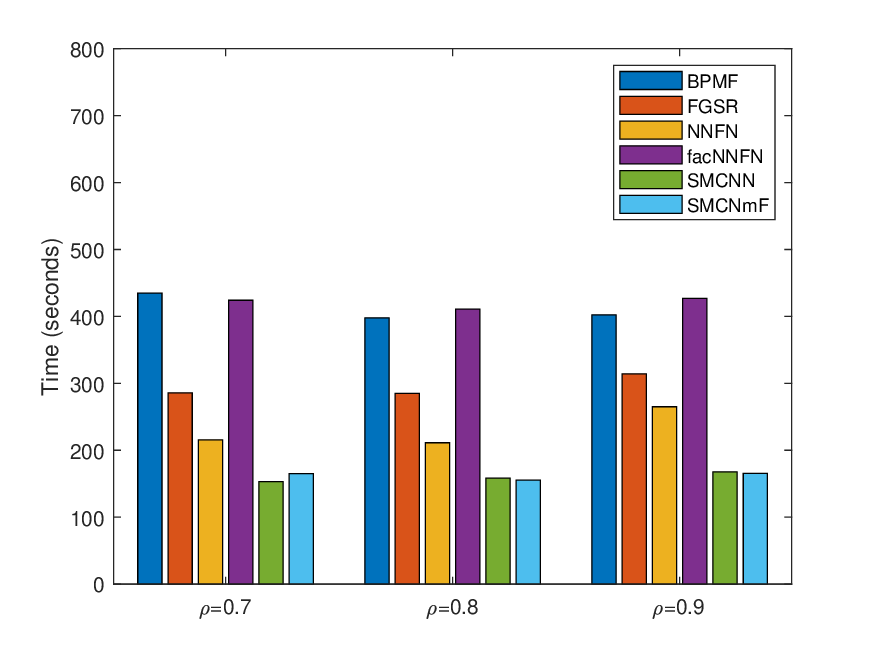}}
       \hspace{-2cm}
    \caption{Running Time versus $\rho$ on PROTEIN dataset, with $m=1,000$ search candidates and $n=200$ query items, with rank $r=\{50, 100, 200\}$, $\lambda = 0.001$, $\gamma = 0.001$, $T=10,000$. (Best viewed in $\times$ 2 sized color pdf file)}
    \label{fig:time:PROTEIN}
\end{figure*}

\begin{figure*}[!htb]
    \centering    
   \hspace{-1.8cm}
      \subfigure[(a) CIFAR, $r$=50]{
   \includegraphics[width =0.35\columnwidth, trim=15 0 30 10, clip]{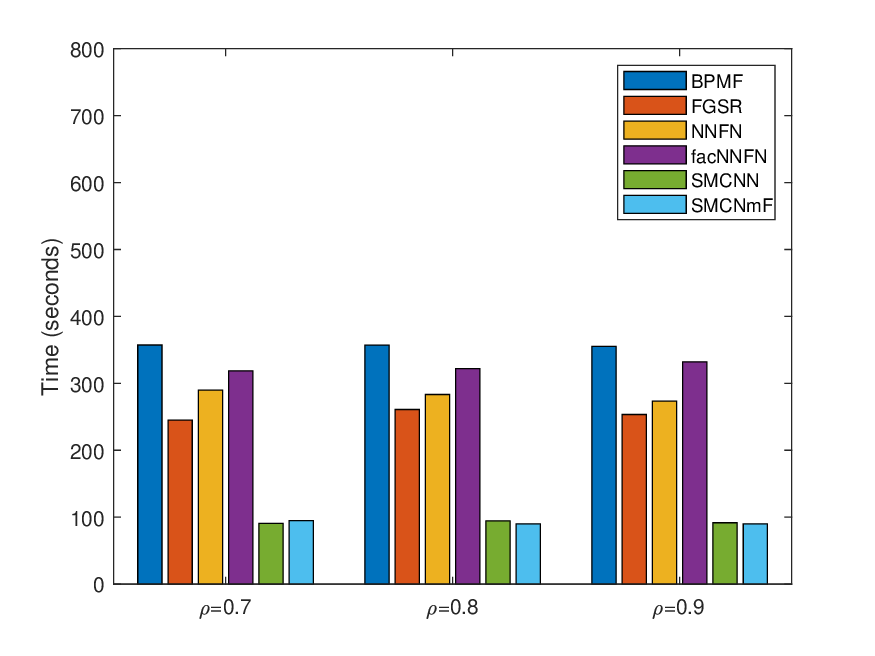}}
      \subfigure[(b) CIFAR, $r$=100]{
   \includegraphics[width =0.35\columnwidth, trim=15 0 30 10, clip]{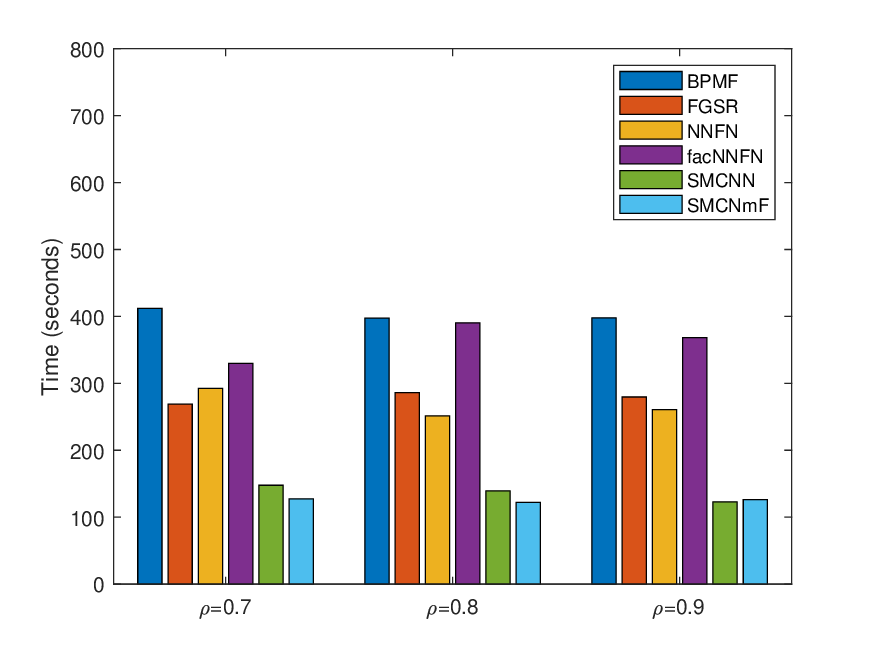}}
      \subfigure[(c) CIFAR, $r$=200]{
   \includegraphics[width =0.35\columnwidth, trim=15 0 30 10, clip]{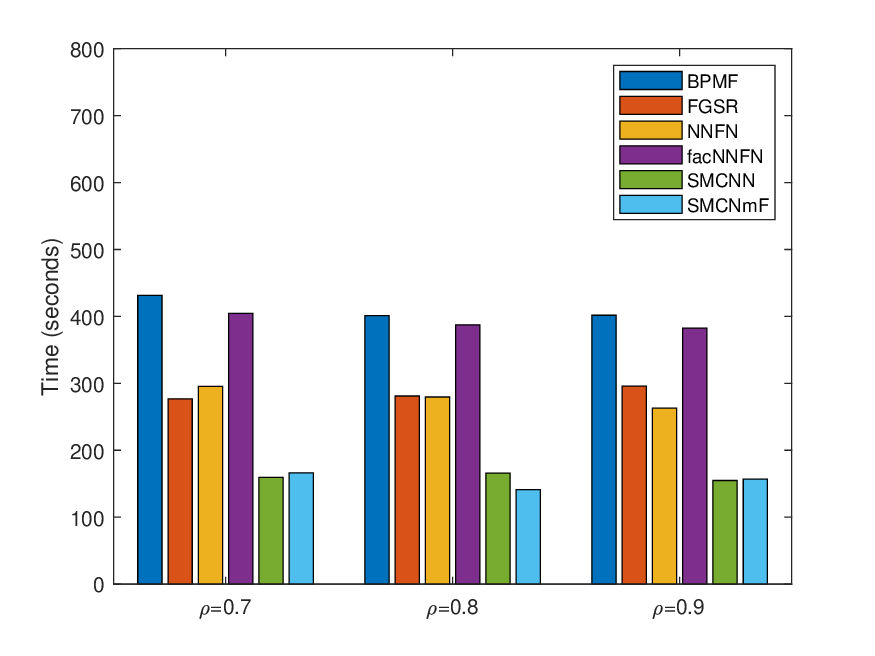}}
       \hspace{-2cm}
    \caption{Running Time versus $\rho$ on CIFAR dataset, with $m=1,000$ search candidates and $n=200$ query items, with rank $r=\{50, 100, 200\}$, $\lambda = 0.001$, $\gamma = 0.001$, $T=10,000$. (Best viewed in $\times$ 2 sized color pdf file)}
    \label{fig:time:CIFAR}
\end{figure*} 

\begin{figure*}[!htb]
    \centering    
    \hspace{-1.8cm}
      \subfigure[(a) GoogleNews, $r$=50]{
   \includegraphics[width =0.35\columnwidth, trim=15 0 30 10, clip]{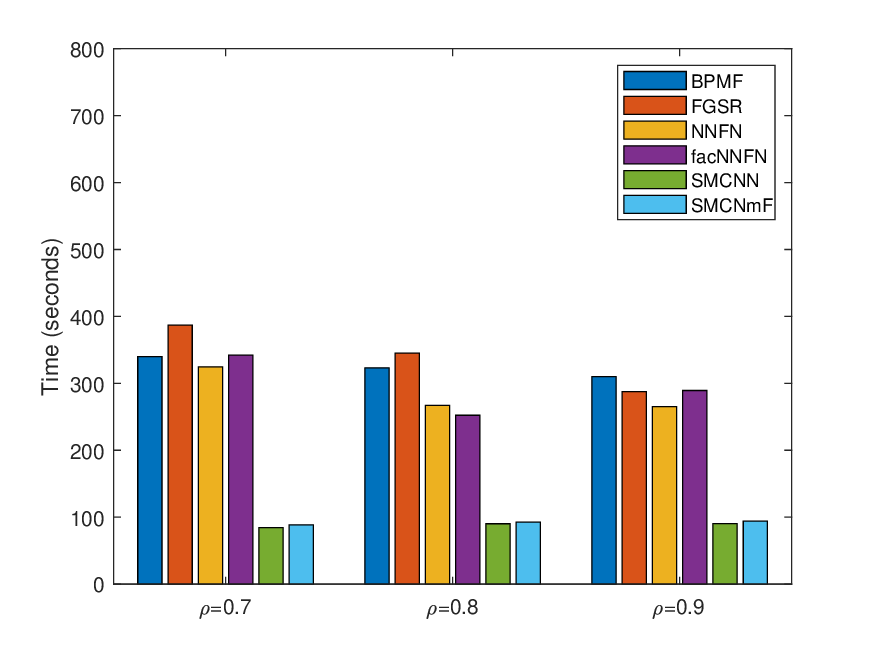}}
      \subfigure[(b) GoogleNews, $r$=100]{
   \includegraphics[width =0.35\columnwidth, trim=15 0 30 10, clip]{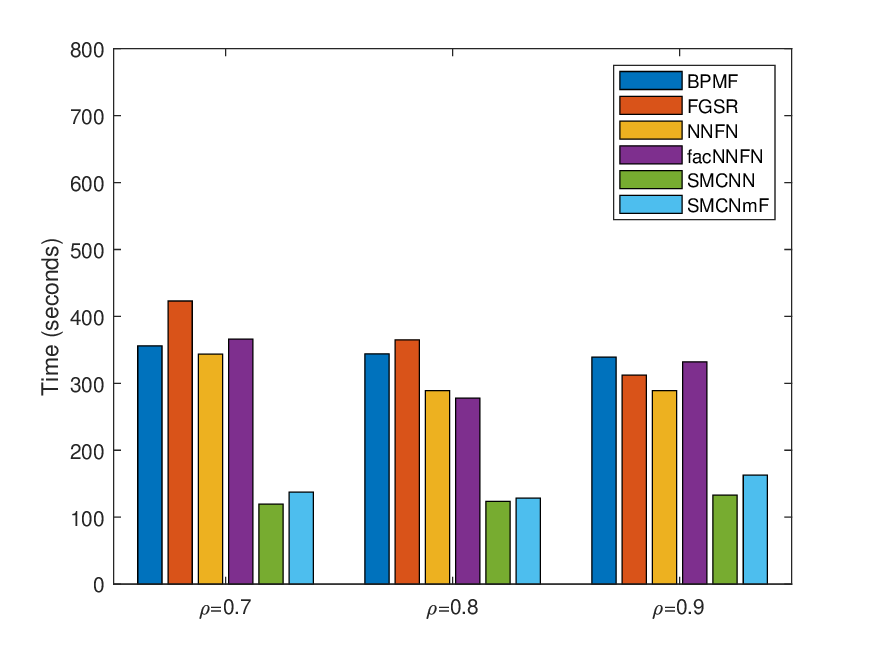}}
      \subfigure[(c) GoogleNews, $r$=200]{
   \includegraphics[width =0.35\columnwidth, trim=15 0 30 10, clip]{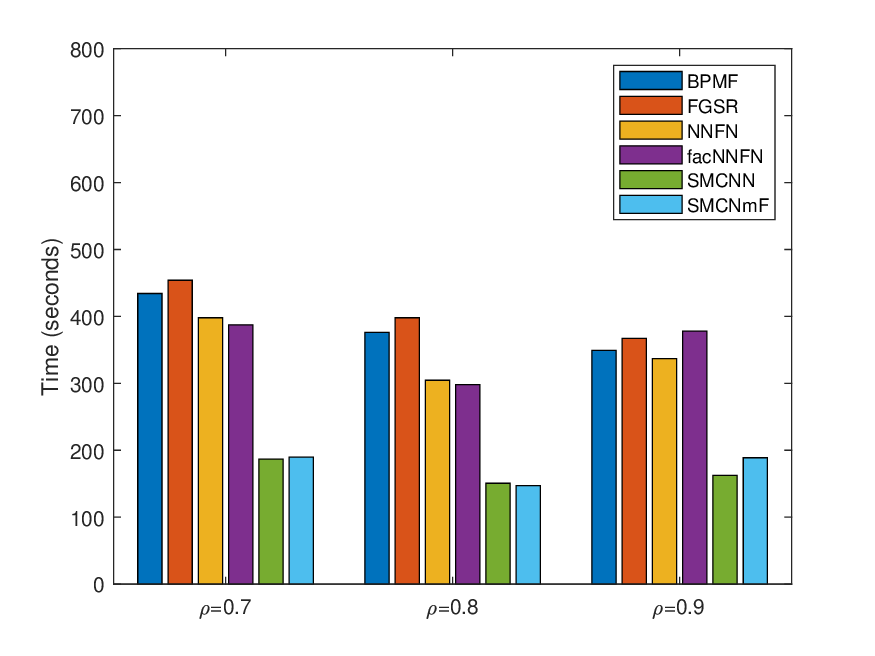}}
       \hspace{-2cm}
    \caption{Running Time versus $\rho$ on GoogleNews dataset, with $m=1,000$ search candidates and $n=200$ query items, with rank $r=\{50, 100, 200\}$, $\lambda = 0.001$, $\gamma = 0.001$, $T=10,000$. (Best viewed in $\times$ 2 sized color pdf file)}
    \label{fig:time:Google}
\end{figure*}

%\subsubsection{Remaining Results on other Comparison Methods}
%Table \ref{tab:timetol} shows the running time varying with missing ratios on various datasets with search candidates $n=1,000$ and query items $n=200$, with different stopping criteria, i.e, tolerance is $1e-4$ and maximum iterations $T=10,000$.  

%\input{Table/Time/TimeAll}
%To keep consistent with the default setting in the original paper, we set the $tol$ as the stopping criteria. To provide a better comparison, we also set the stopping criteria by setting the maximum number of iterations as $T=10,000$. Combining the results in Fig.\ref{fig:time:ImageNet} and Fig.\ref{fig:time:MNIST}-Fig.\ref{fig:time:Google}, we summarize that the computational cost highly lies on SVD operation, and MF technique can help to reduce the computational cost, which is consistent with the theoretical analysis in Section \ref{sec:analysis}.

\clearpage
%Ablation Study 
\subsection{Hyperparameters Analysis}
\label{App:Para} 
\subsubsection{Rank $r$ }
%missing ratio= 0.7,0.8,0.9
Fig.\ref{fig:rank:R07}-Fig.\ref{fig:rank:R09} show the RMSE/Recall varying with Rank($V$)=$r$ on various datasets with various missing ratio $\rho$, fixed $\lambda = 0.001$ and fixed $\gamma = 0.001$. Combining the RMSE and Recall, it is obvious that the RMSE of both SMCNN and SMCNmF decreases with the increasing number of rank $r$ on various datasets, and the corresponding Recall shows the opposite tendencies. 
 
\begin{figure*}[!htb]
    \centering  
      \subfigure{
   \includegraphics[width = 0.18\columnwidth, trim=15 0 30 10, clip]{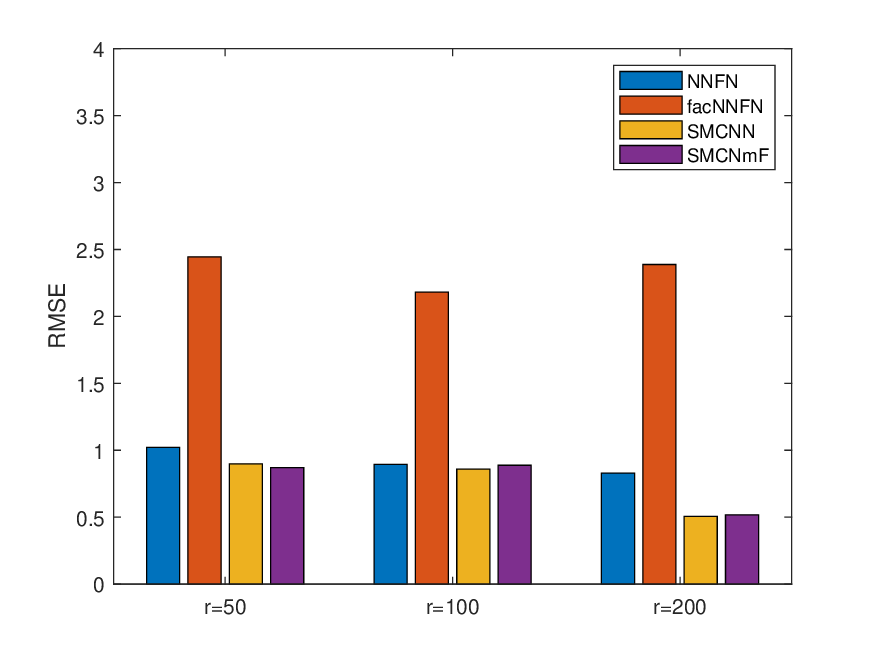}}
      \subfigure{
   \includegraphics[width = 0.18\columnwidth, trim=15 0 30 10, clip]{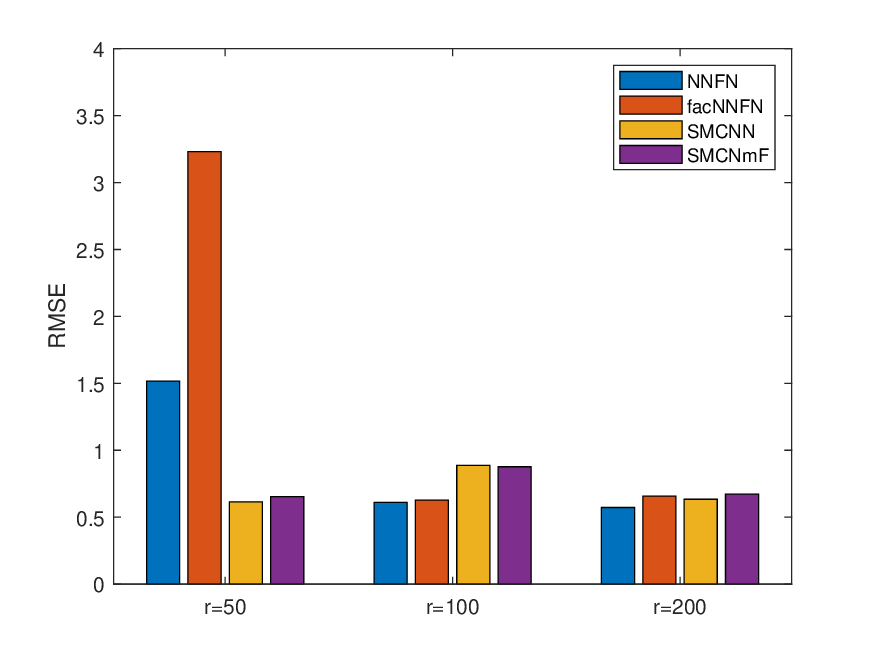}}
      \subfigure{
   \includegraphics[width = 0.18\columnwidth, trim=15 0 30 10, clip]{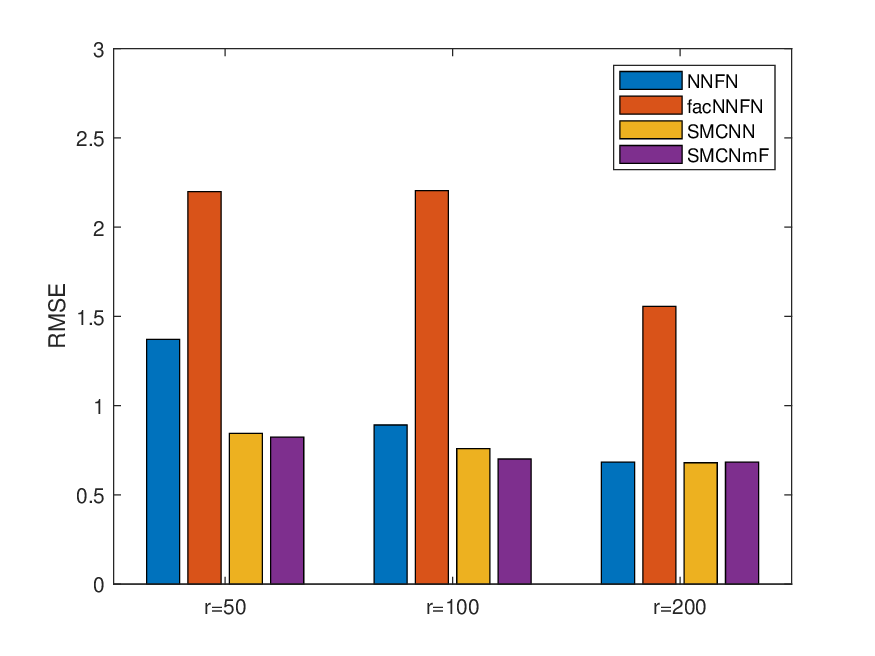}}
      \subfigure{ 
   \includegraphics[width = 0.18\columnwidth, trim=15 0 30 10, clip]{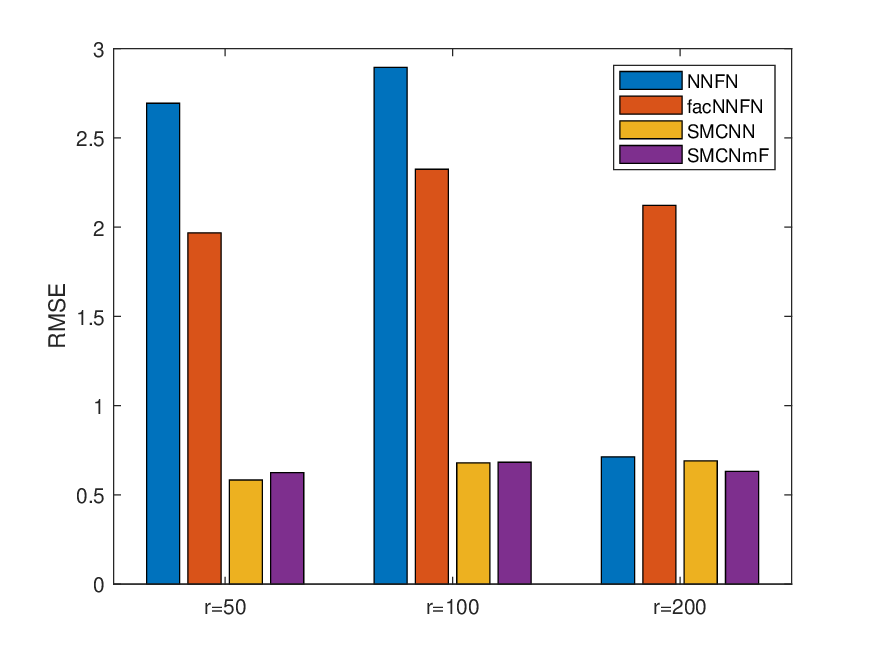}}
   \subfigure{ 
   \includegraphics[width = 0.18\columnwidth, trim=15 0 30 10, clip]{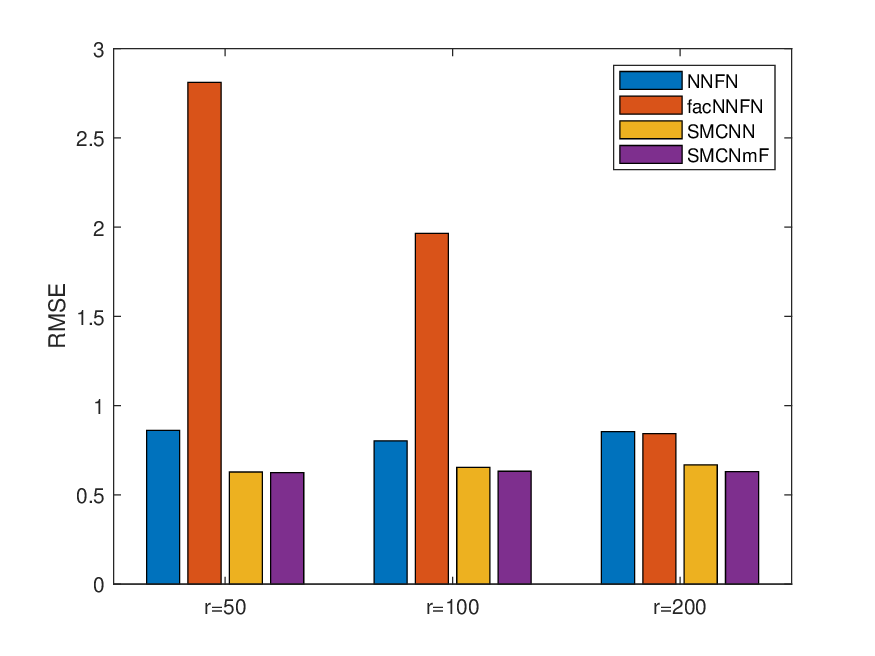}}
   \vspace{-0.5cm}
   
      \subfigure[(a) ImageNet]{
   \includegraphics[width = 0.18\columnwidth, trim=15 0 30 10, clip]{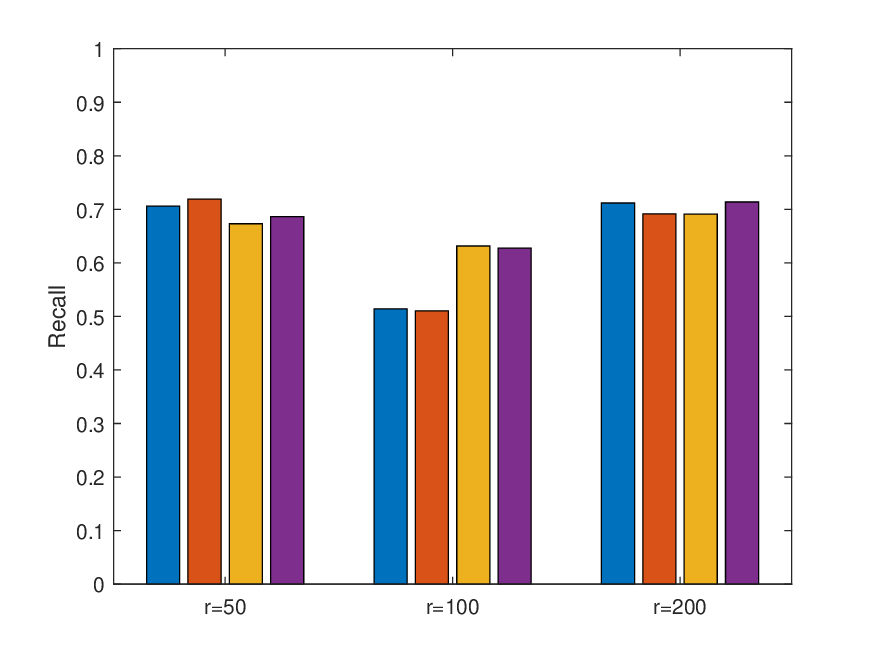}}
      \subfigure[(b) MNIST]{
   \includegraphics[width = 0.18\columnwidth, trim=15 0 30 10, clip]{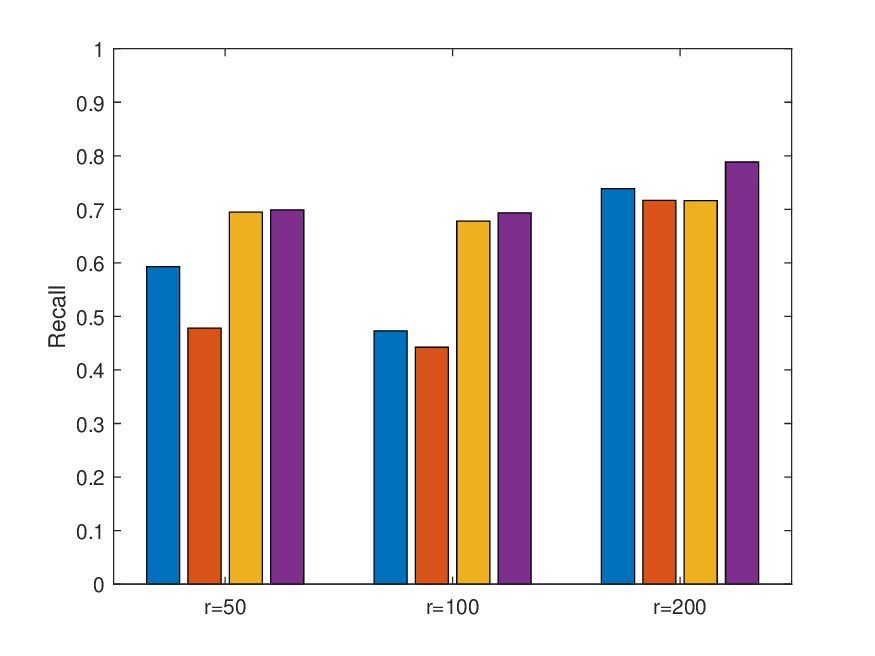}}
      \subfigure[(c) PROTEIN]{
   \includegraphics[width = 0.18\columnwidth, trim=15 0 30 10, clip]{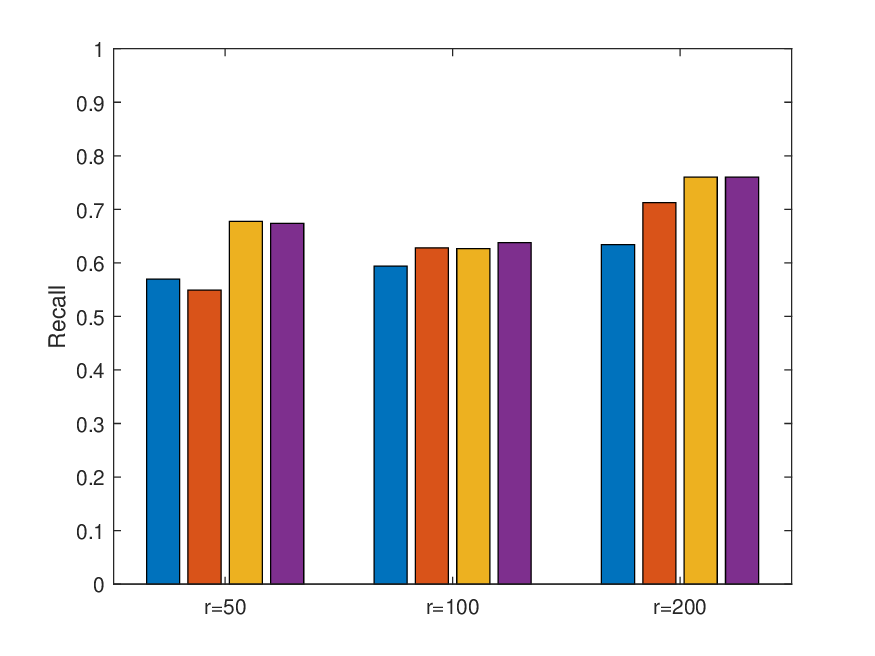}}
      \subfigure[(d) CIFAR]{ 
   \includegraphics[width = 0.18\columnwidth, trim=15 0 30 10, clip]{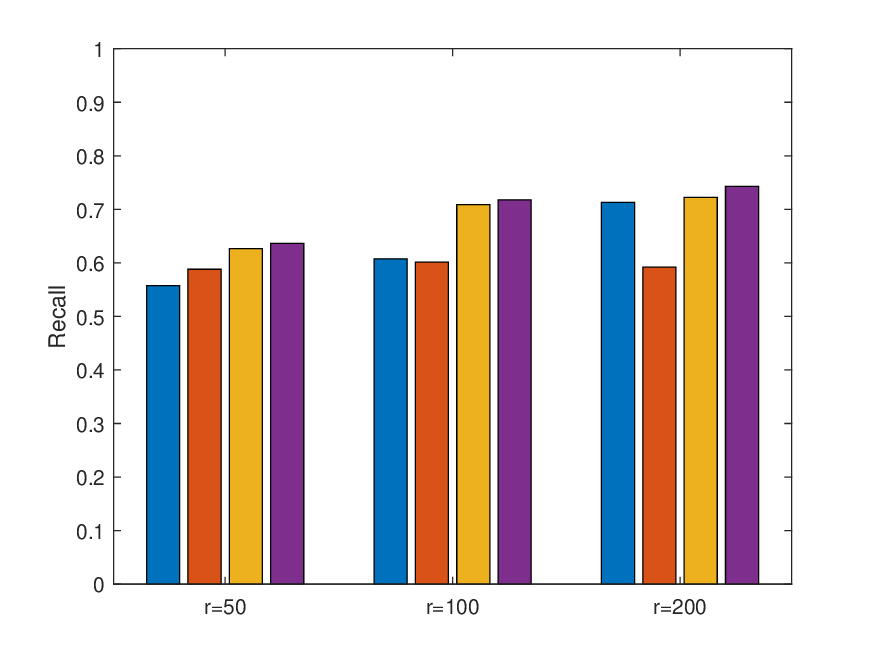}}
    \subfigure[(e) GoogleNews]{ 
   \includegraphics[width = 0.18\columnwidth, trim=15 0 30 10, clip]{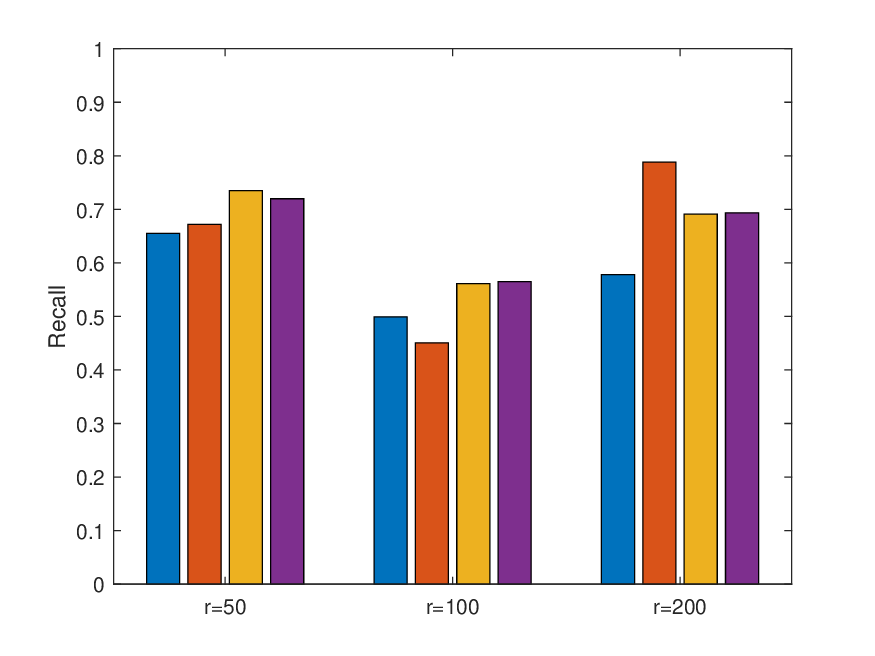}}
   \caption{RMSE/Recall@20\% versus rank on various dataset, with $m=1,000$ search candidates and $n=200$ query items, where missing ratio $\rho=0.7$, rank $r=50,100,200$, $\lambda = 0.001$, $\gamma = 0.001$, iterations $T=10,000$.}
    \label{fig:rank:R07}
\end{figure*}

\begin{figure*}[!htb]
    \centering  
      \subfigure{
   \includegraphics[width = 0.18\columnwidth, trim=15 0 30 10, clip]{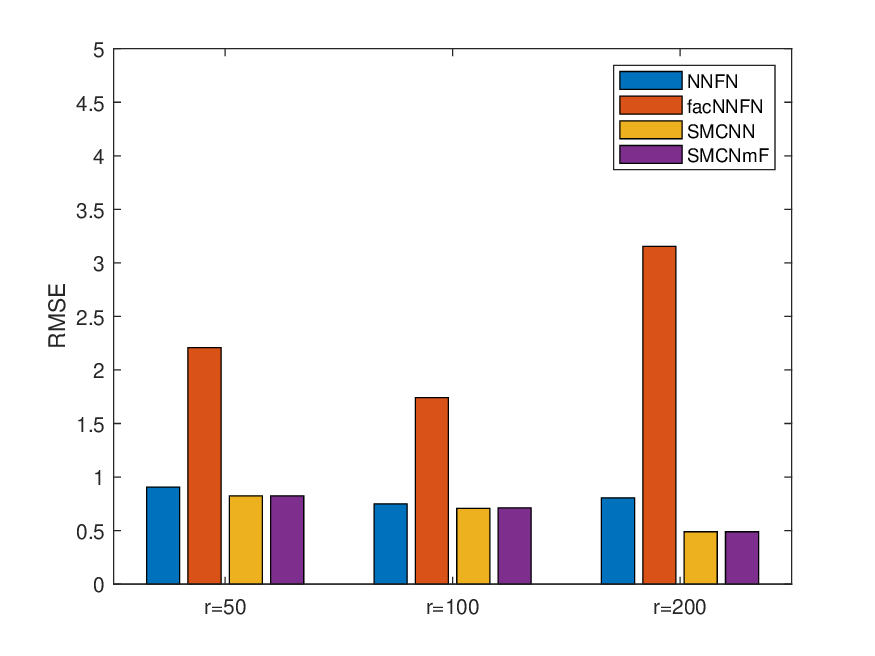}}
      \subfigure{
   \includegraphics[width = 0.18\columnwidth, trim=15 0 30 10, clip]{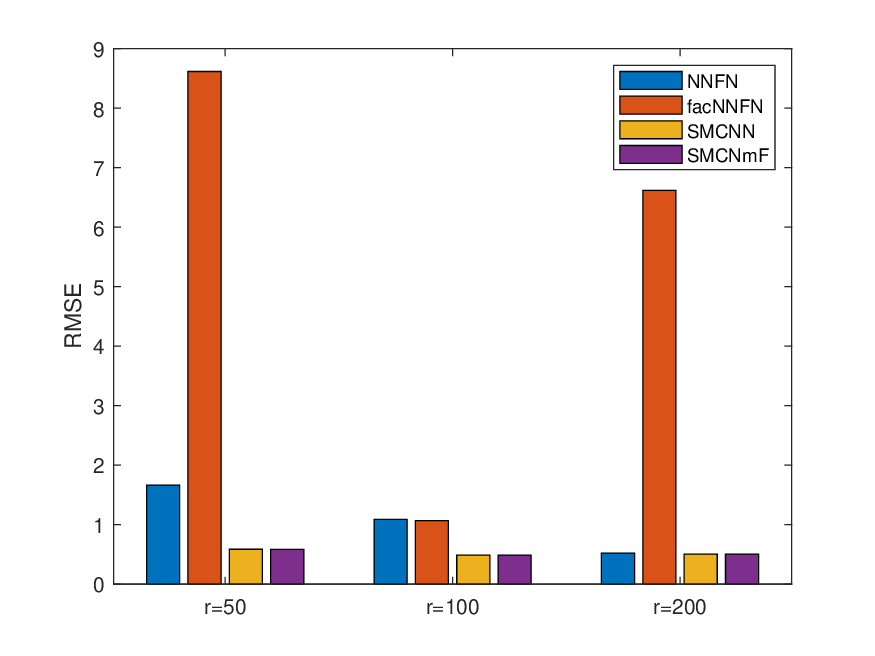}}
      \subfigure{
   \includegraphics[width = 0.18\columnwidth, trim=15 0 30 10, clip]{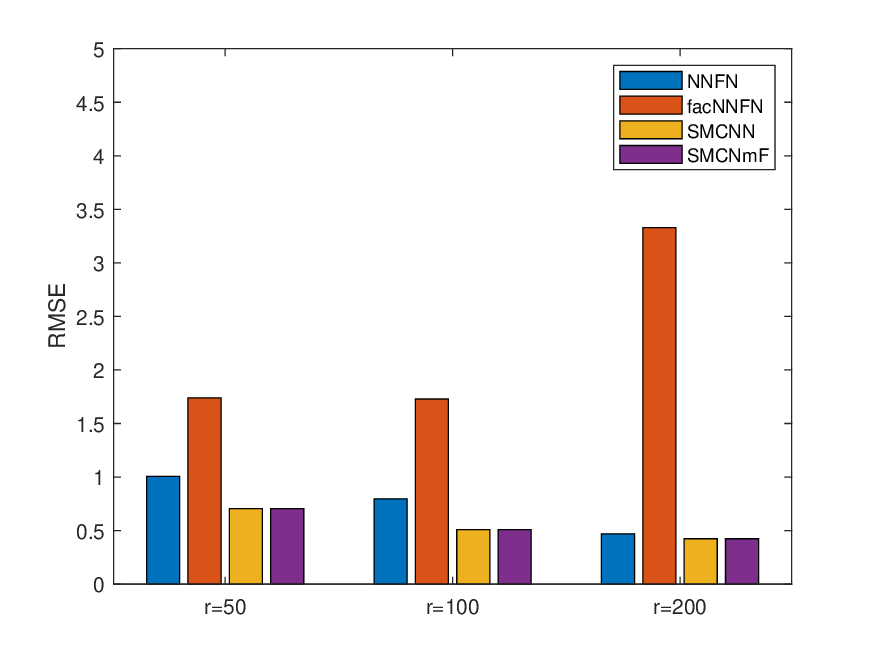}}
      \subfigure{ 
   \includegraphics[width = 0.18\columnwidth, trim=15 0 30 10, clip]{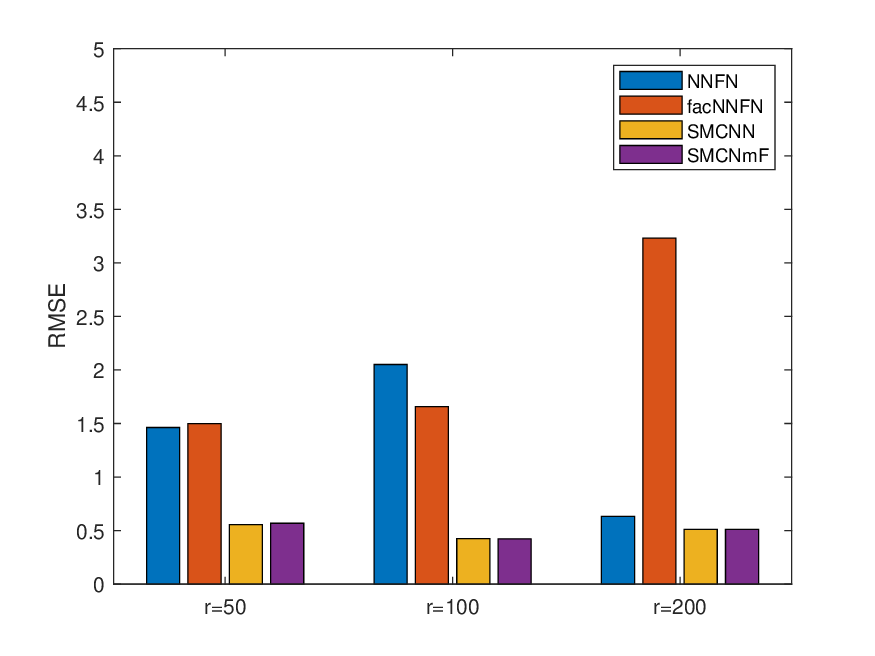}}
   \subfigure{ 
   \includegraphics[width = 0.18\columnwidth, trim=15 0 30 10, clip]{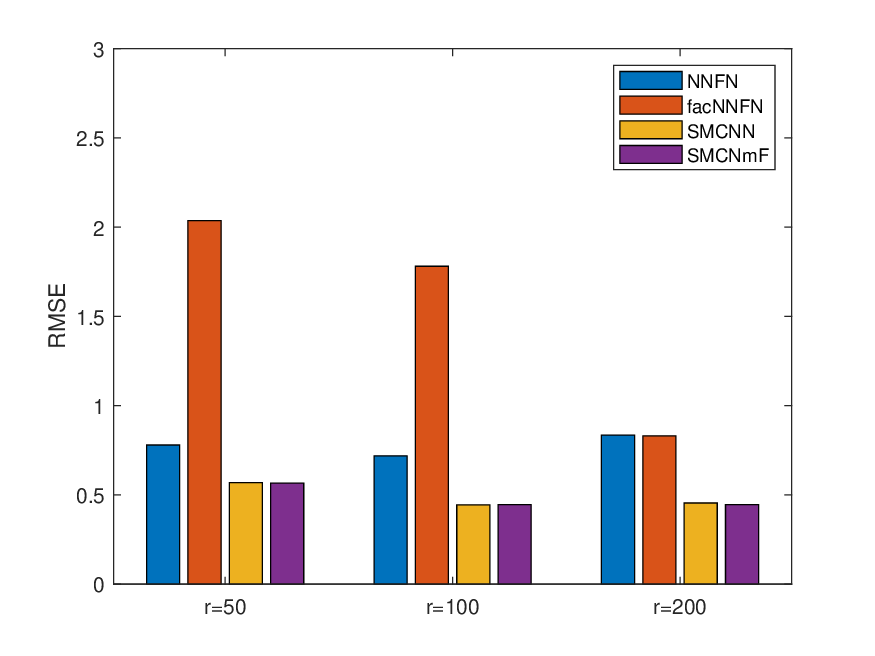}}
   \vspace{-0.5cm}
   
      \subfigure[(a) ImageNet]{
   \includegraphics[width = 0.18\columnwidth, trim=15 0 30 10, clip]{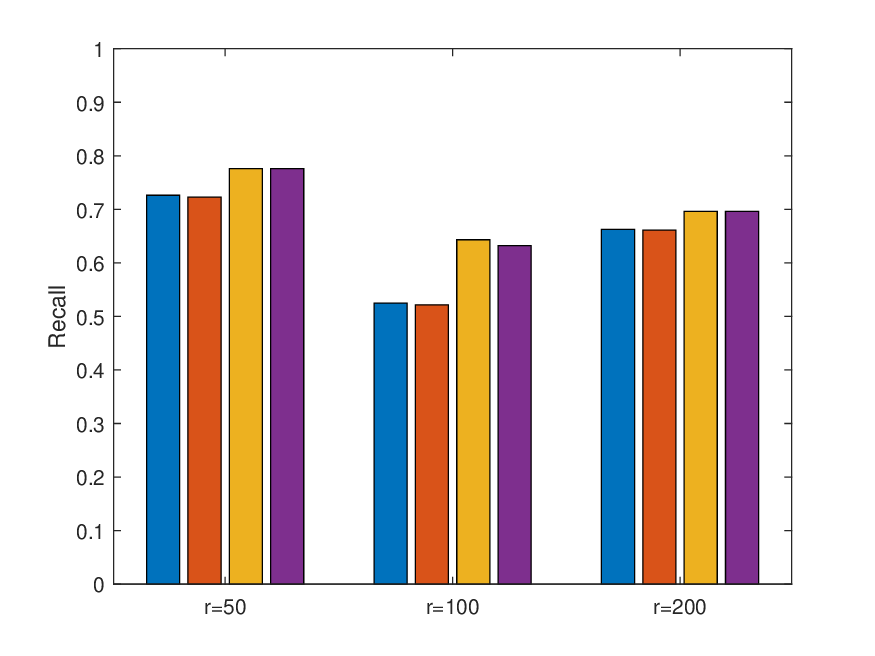}}
      \subfigure[(b) MNIST]{
   \includegraphics[width = 0.18\columnwidth, trim=15 0 30 10, clip]{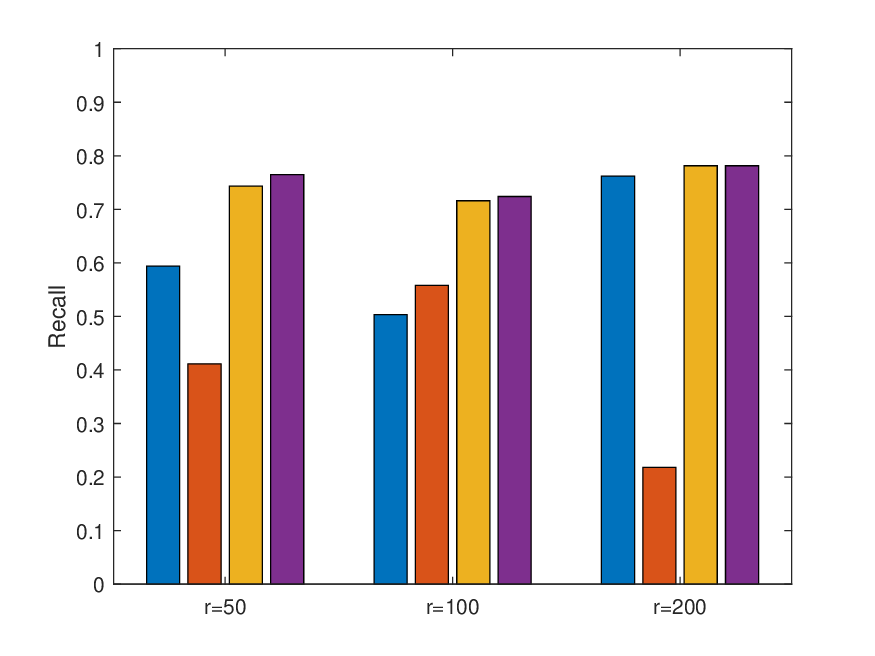}}
      \subfigure[(c) PROTEIN]{
   \includegraphics[width = 0.18\columnwidth, trim=15 0 30 10, clip]{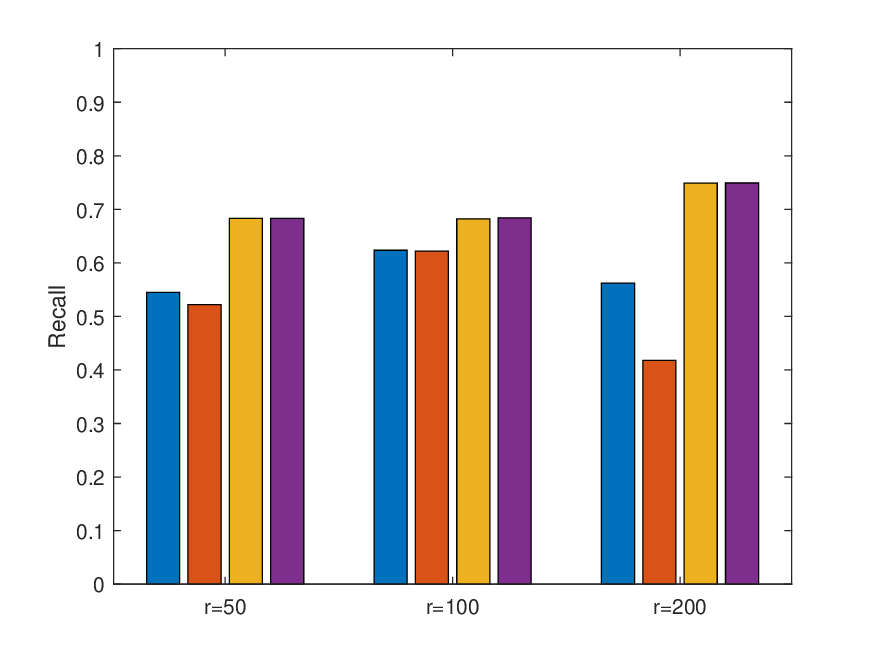}}
      \subfigure[(d) CIFAR]{ 
   \includegraphics[width = 0.18\columnwidth, trim=15 0 30 10, clip]{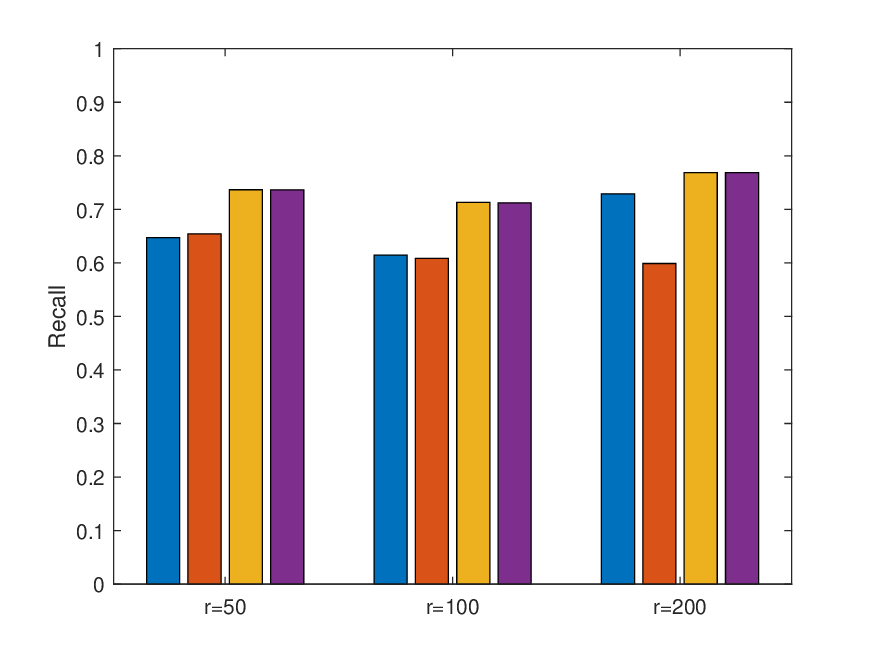}}
    \subfigure[(e) GoogleNews]{ 
   \includegraphics[width = 0.18\columnwidth, trim=15 0 30 10, clip]{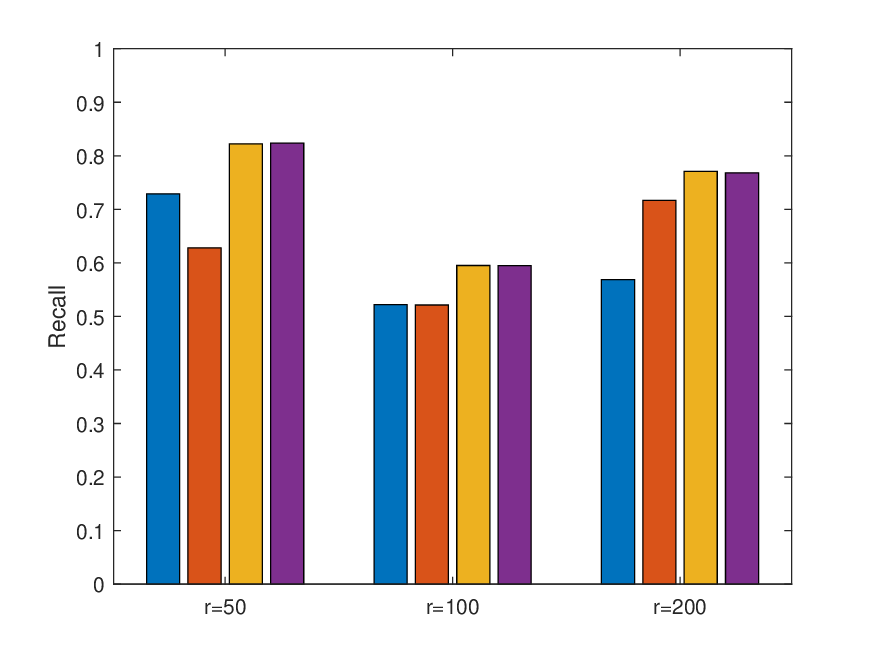}}
   \caption{RMSE/Recall@20\% versus rank on various dataset, with $m=1,000$ search candidates and $n=200$ query items, where missing ratio $\rho=0.8$, rank $r=50,100,200$, $\lambda = 0.001$, $\gamma = 0.001$, iterations $T=10,000$.}
    \label{fig:rank:R08}
\end{figure*}

\begin{figure*}[!htb]
    \centering  
      \subfigure{
   \includegraphics[width = 0.18\columnwidth, trim=15 0 30 10, clip]{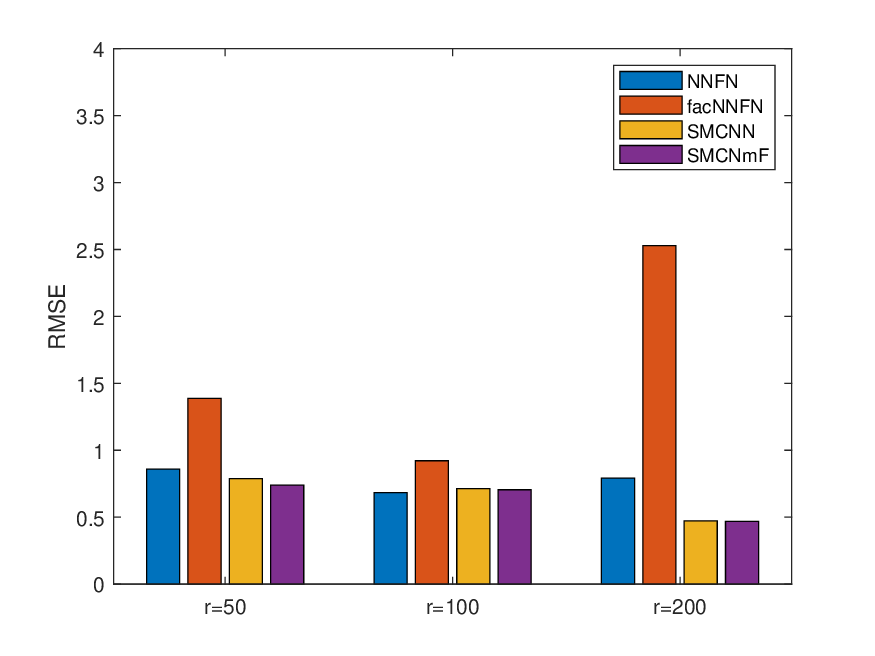}}
      \subfigure{
   \includegraphics[width = 0.18\columnwidth, trim=15 0 30 10, clip]{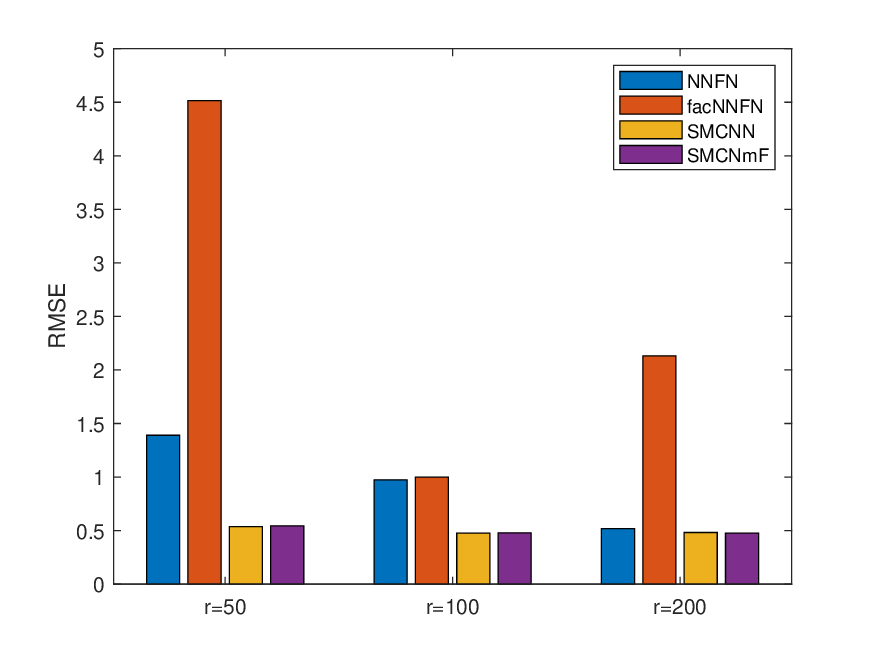}}
      \subfigure{
   \includegraphics[width = 0.18\columnwidth, trim=15 0 30 10, clip]{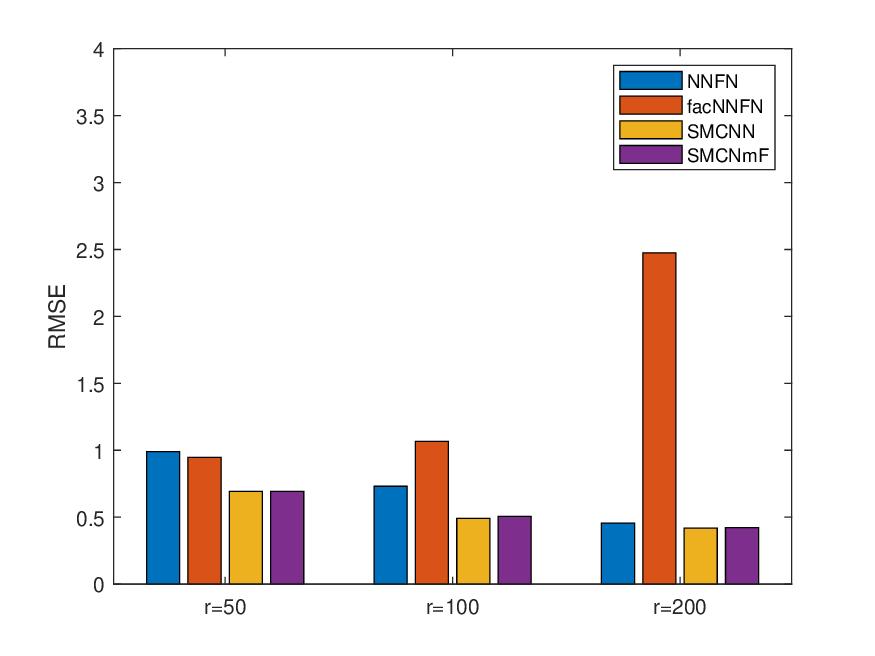}}
      \subfigure{ 
   \includegraphics[width = 0.18\columnwidth, trim=15 0 30 10, clip]{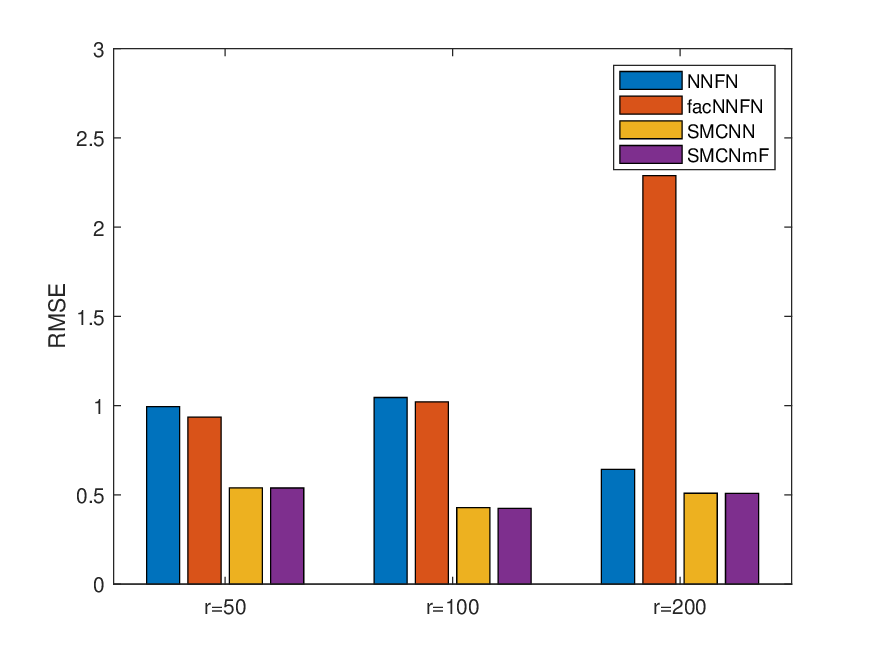}}
   \subfigure{ 
   \includegraphics[width = 0.18\columnwidth, trim=15 0 30 10, clip]{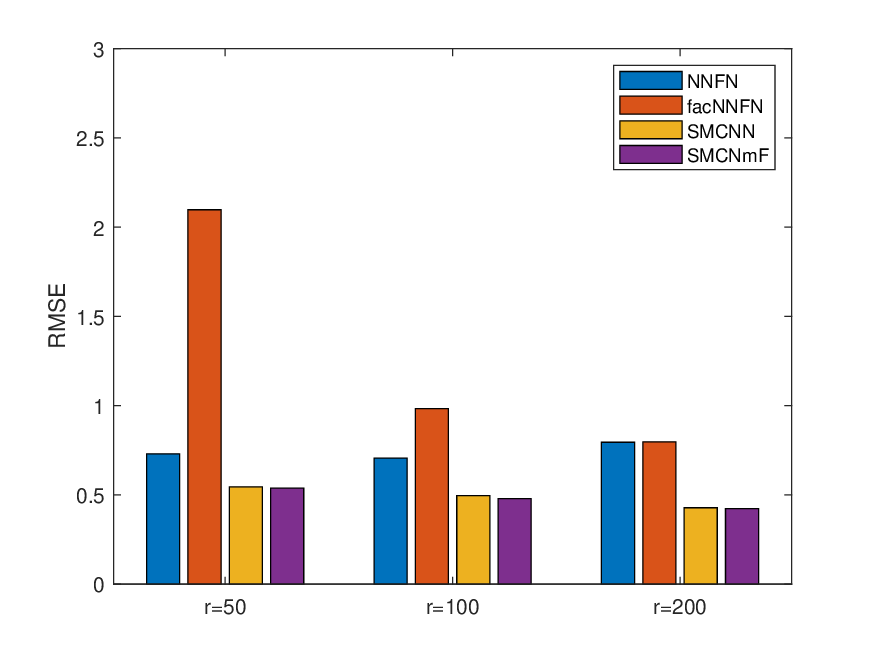}}
   \vspace{-0.5cm}
   
      \subfigure[(a) ImageNet]{
   \includegraphics[width = 0.18\columnwidth, trim=15 0 30 10, clip]{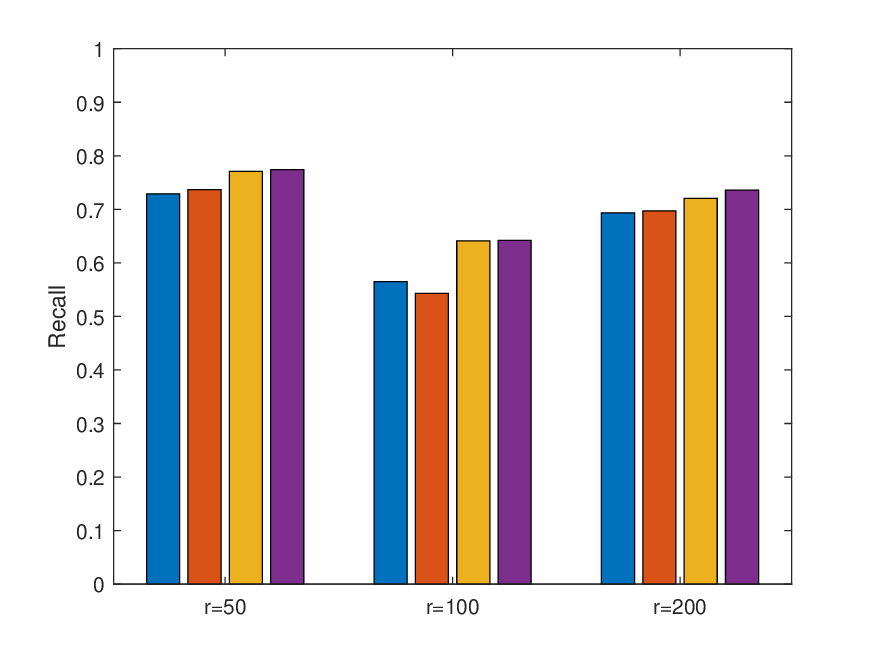}}
      \subfigure[(b) MNIST]{
   \includegraphics[width = 0.18\columnwidth, trim=15 0 30 10, clip]{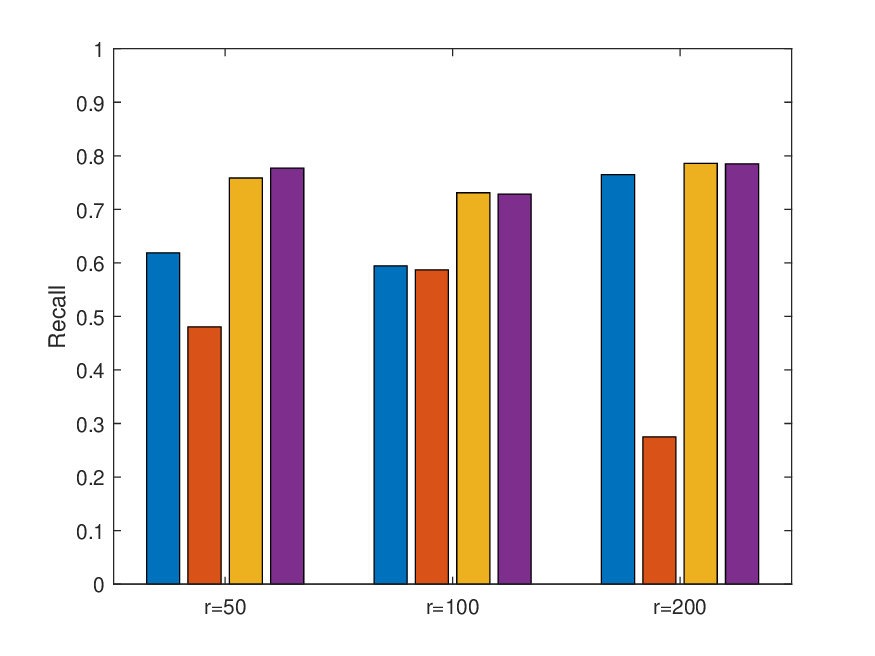}}
      \subfigure[(c) PROTEIN]{
   \includegraphics[width = 0.18\columnwidth, trim=15 0 30 10, clip]{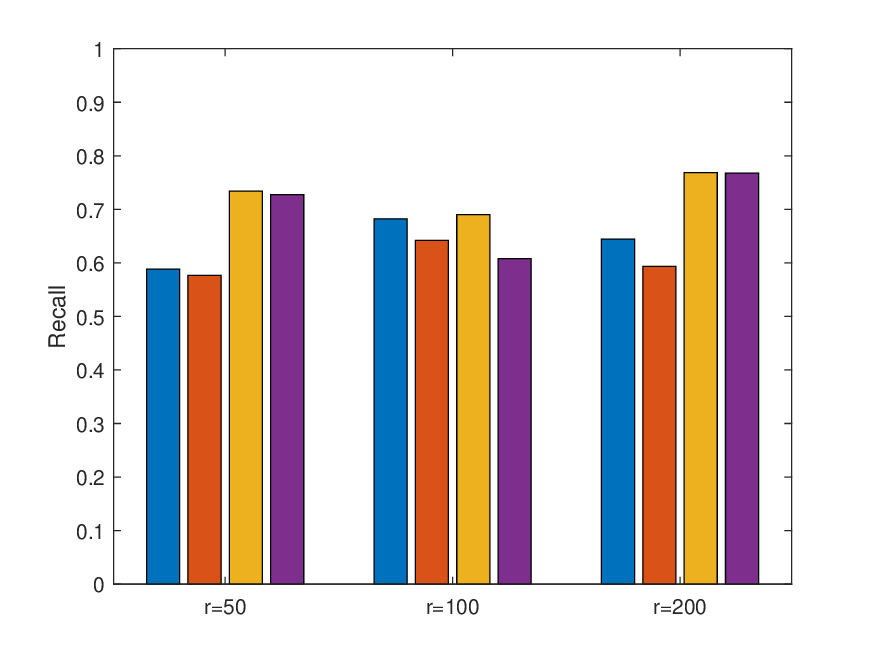}}
      \subfigure[(d) CIFAR]{ 
   \includegraphics[width = 0.18\columnwidth, trim=15 0 30 10, clip]{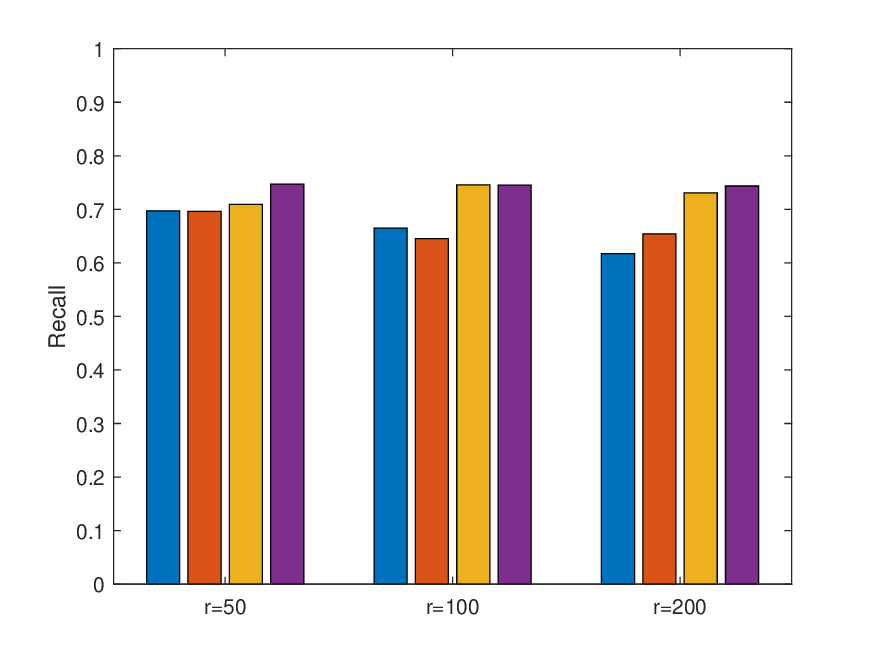}}
    \subfigure[(e) GoogleNews]{ 
   \includegraphics[width = 0.18\columnwidth, trim=15 0 30 10, clip]{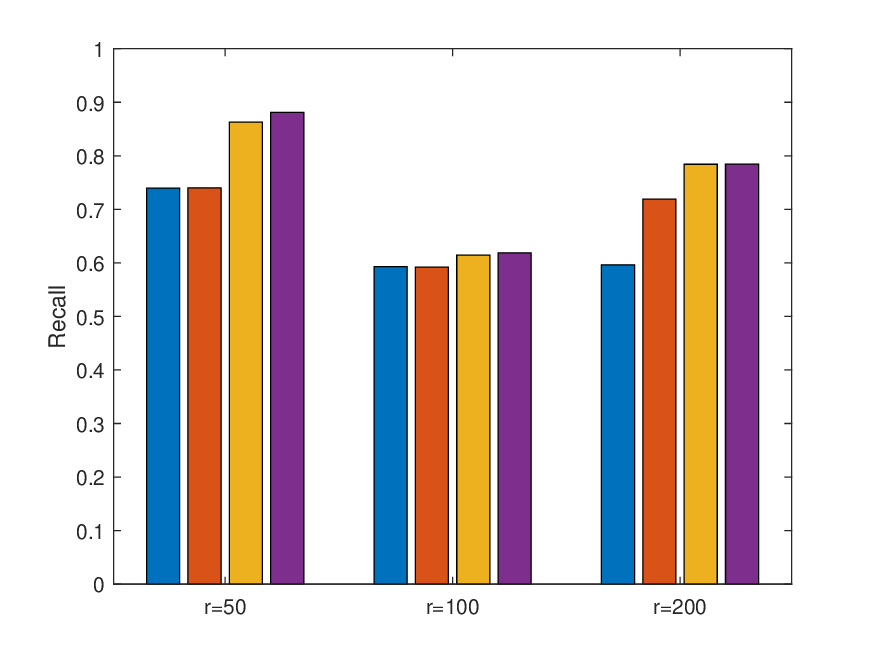}}
   \caption{RMSE/Recall@20\% versus rank on various dataset, with $m=1,000$ search candidates and $n=200$ query items, where missing ratio $\rho=0.9$, rank $r=50,100,200$, $\lambda = 0.001$, $\gamma = 0.001$, iterations $T=10,000$.}
    \label{fig:rank:R09}
\end{figure*}

%-------------------------------------------------------------------------
\clearpage
\subsubsection{Weighted Parameter $\lambda$}
%\lambda  
Fig.\ref{fig:lambda:R07}-Fig.\ref{fig:lambda:R09} show RMSE/Recall varying with weighted parameter $\lambda$ on various datasets with various missing ratio $\rho$, and fixed rank $r=100$, and $\gamma = 0.001$. It is obvious to find that the RMSE first decreases from $\lambda=0.0001$ to $\lambda=0.01$, and increases from $\lambda=0.1$ to $\lambda=10$ on ImageNet, MNIST, and PROTEIN datasets. However, the RMSE changes a little on CIFAR datasets. It is probably caused by the different data distribution across different datasets. Correspondingly, recall shows the same tendencies on various datasets.

\begin{figure*}[!htb]
    \centering  
      \subfigure{
   \includegraphics[width = 0.18\columnwidth, trim=15 0 30 10, clip]{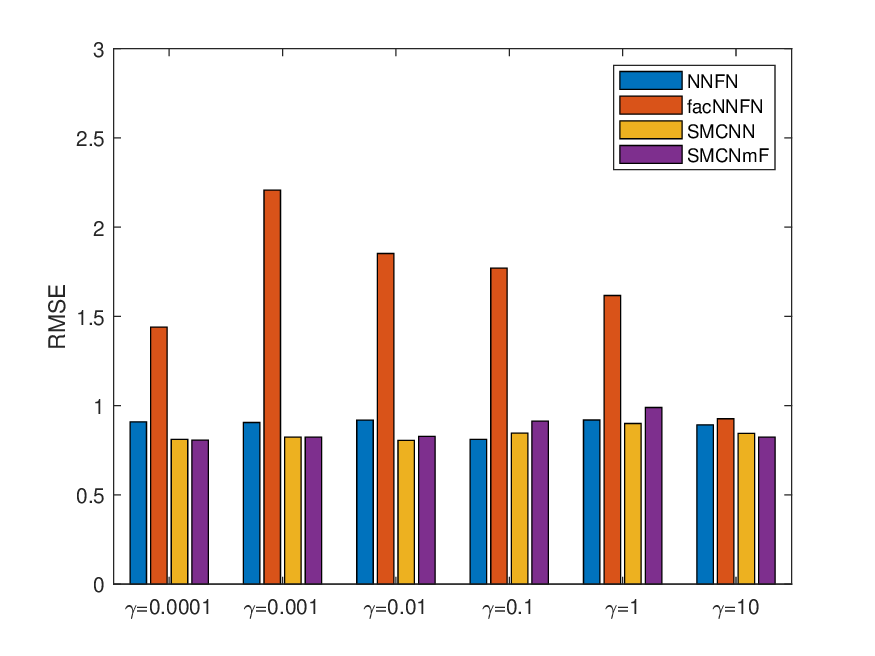}}
      \subfigure{
   \includegraphics[width = 0.18\columnwidth, trim=15 0 30 10, clip]{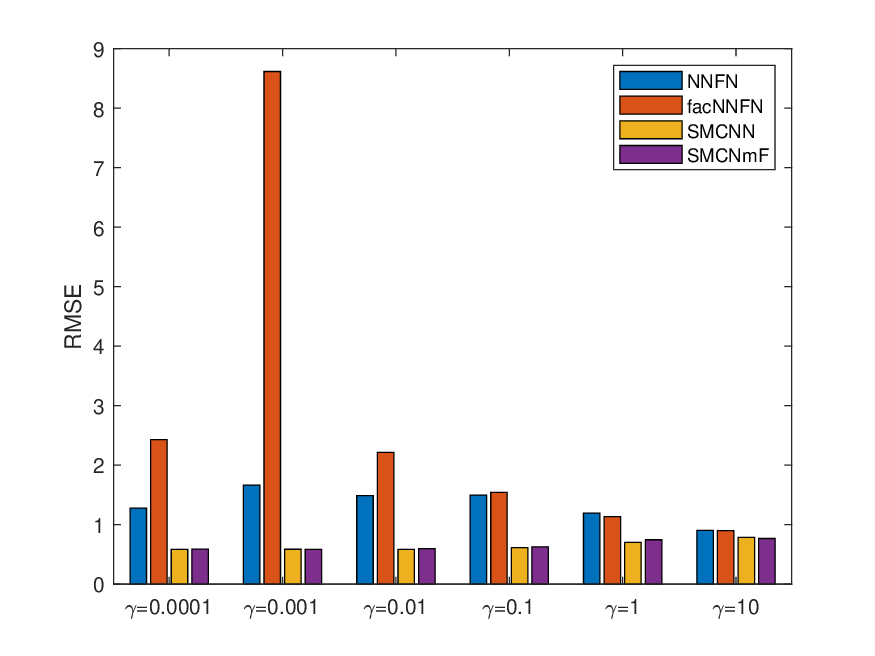}}
      \subfigure{
   \includegraphics[width = 0.18\columnwidth, trim=15 0 30 10, clip]{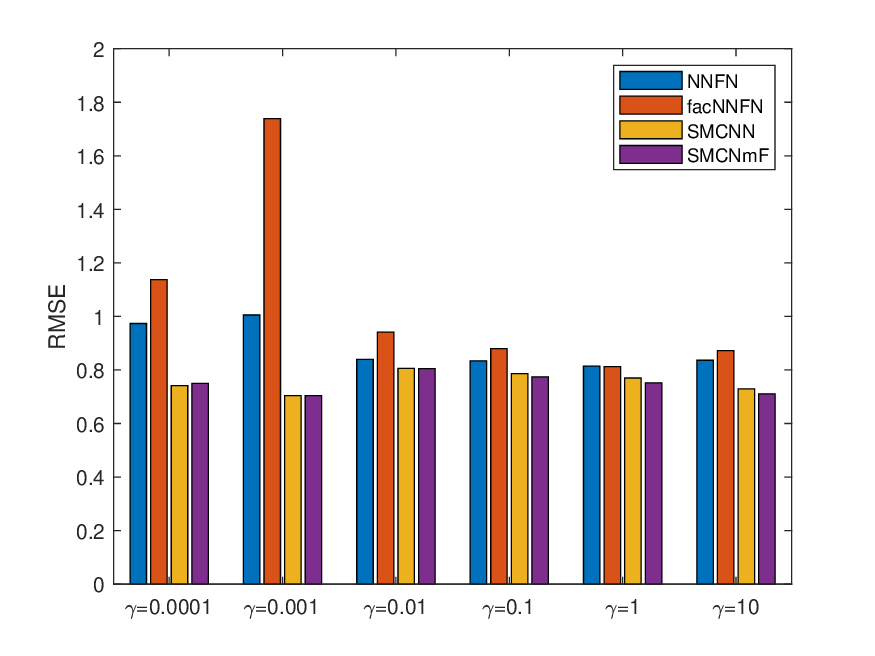}}
      \subfigure{ 
   \includegraphics[width = 0.18\columnwidth, trim=15 0 30 10, clip]{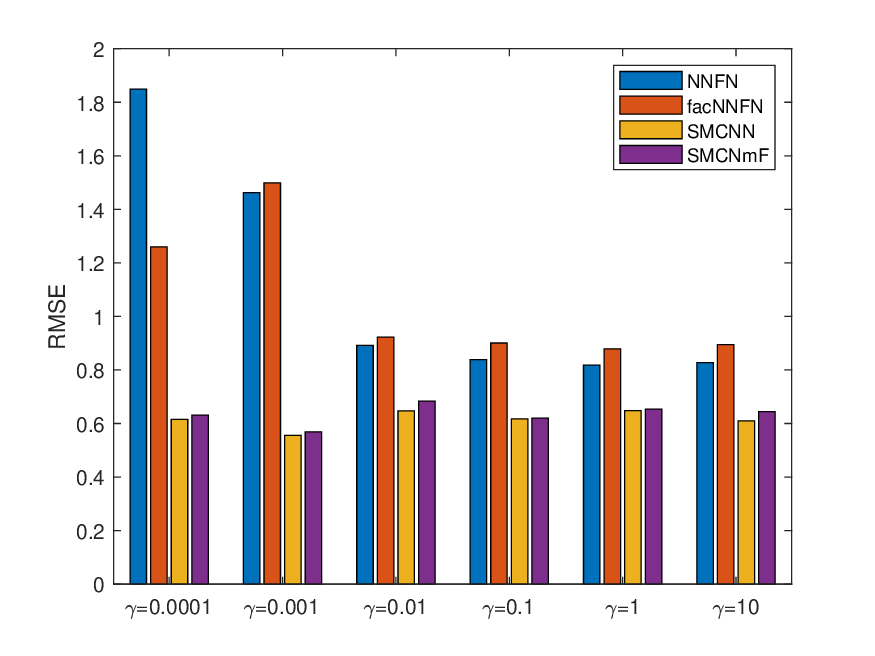}}
   \subfigure{ 
   \includegraphics[width = 0.18\columnwidth, trim=15 0 30 10, clip]{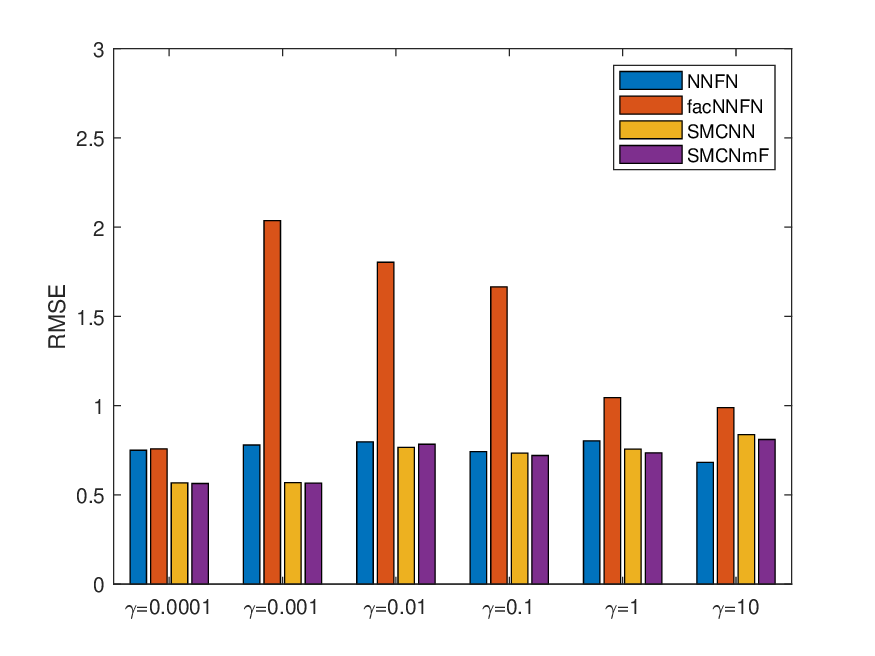}}
   \vspace{-0.5cm}
   
      \subfigure[(a) ImageNet]{
   \includegraphics[width = 0.18\columnwidth, trim=15 0 30 10, clip]{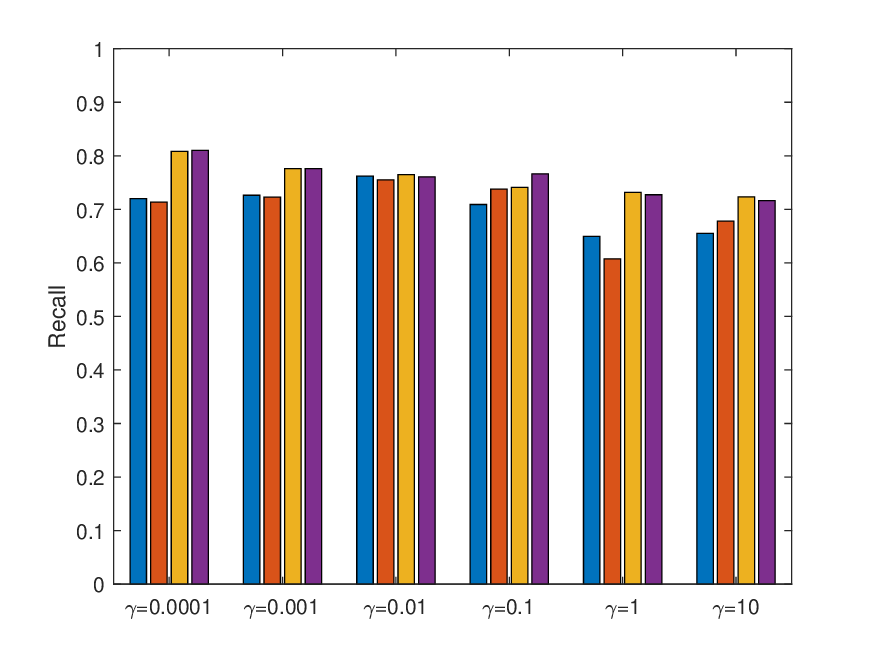}}
      \subfigure[(b) MNIST]{
   \includegraphics[width = 0.18\columnwidth, trim=15 0 30 10, clip]{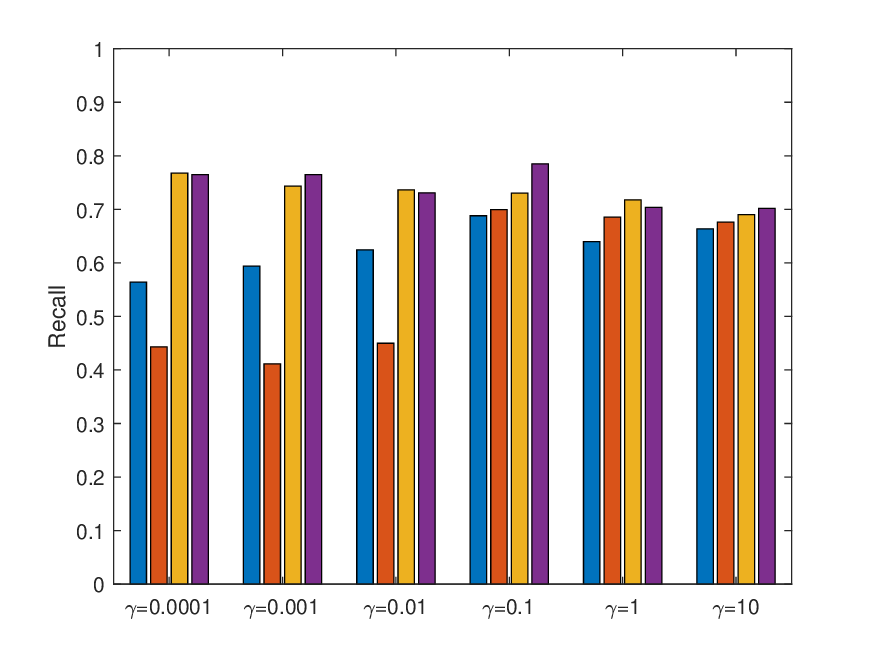}}
      \subfigure[(c) PROTEIN]{
   \includegraphics[width = 0.18\columnwidth, trim=15 0 30 10, clip]{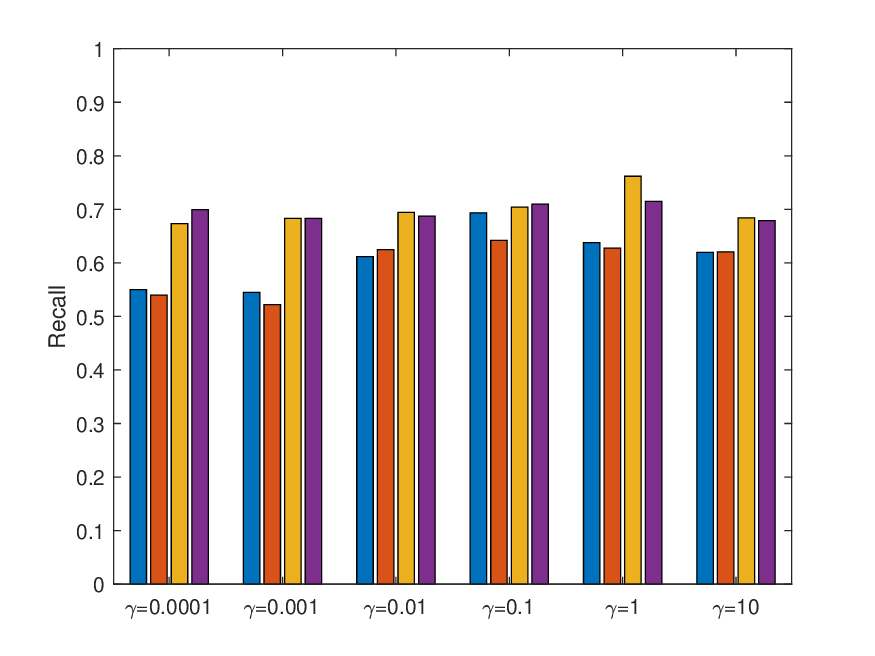}}
      \subfigure[(d) CIFAR]{ 
   \includegraphics[width = 0.18\columnwidth, trim=15 0 30 10, clip]{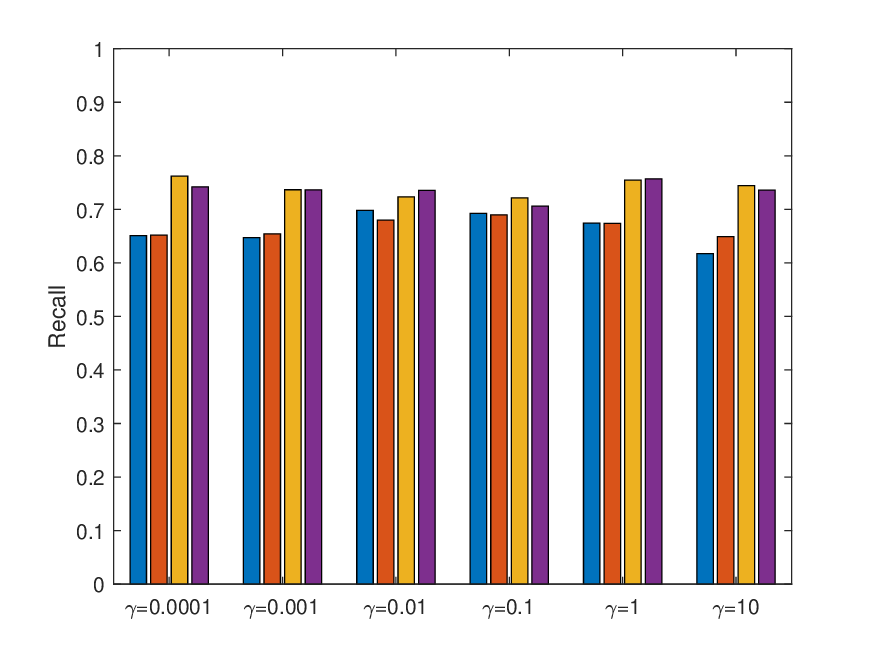}}
    \subfigure[(e) GoogleNews]{ 
   \includegraphics[width = 0.18\columnwidth, trim=15 0 30 10, clip]{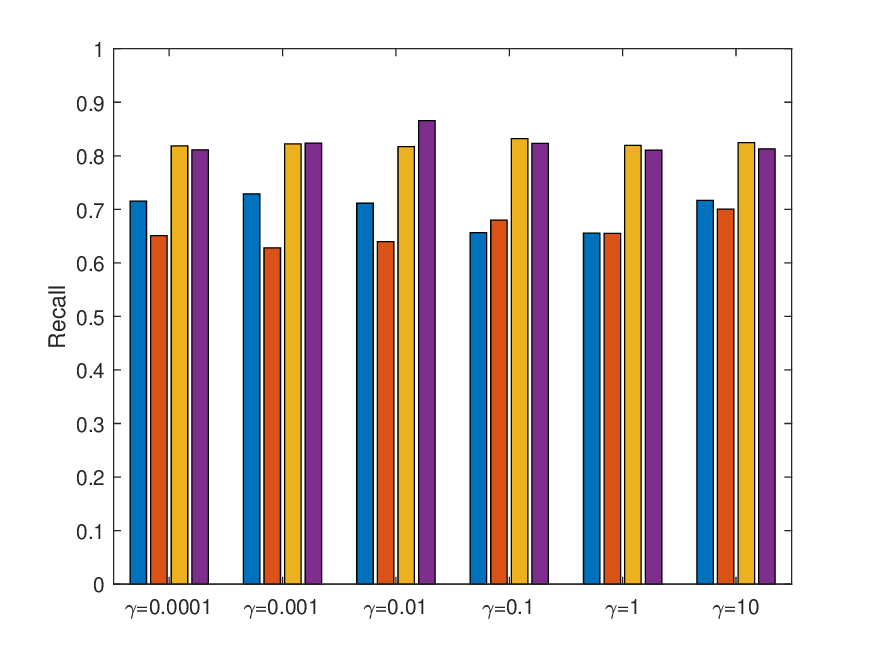}}
   \caption{RMSE/Recall@top20\% versus $\lambda$ on various dataset, with $m=1,000$ search candidates and $n=200$ query items, where missing ratio $\rho=0.7$, rank $r=100$, $\gamma = 0.001$, iterations $T=10,000$.}
    \label{fig:lambda:R07}
\end{figure*}

\begin{figure*}[!htb]
    \centering  
      \subfigure{
   \includegraphics[width = 0.18\columnwidth, trim=15 0 30 10, clip]{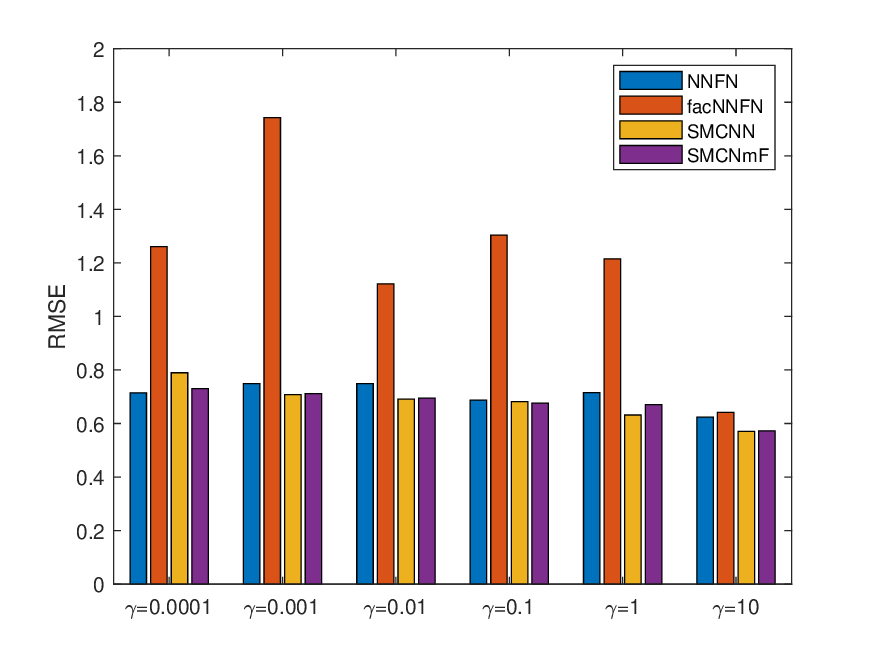}}
      \subfigure{
   \includegraphics[width = 0.18\columnwidth, trim=15 0 30 10, clip]{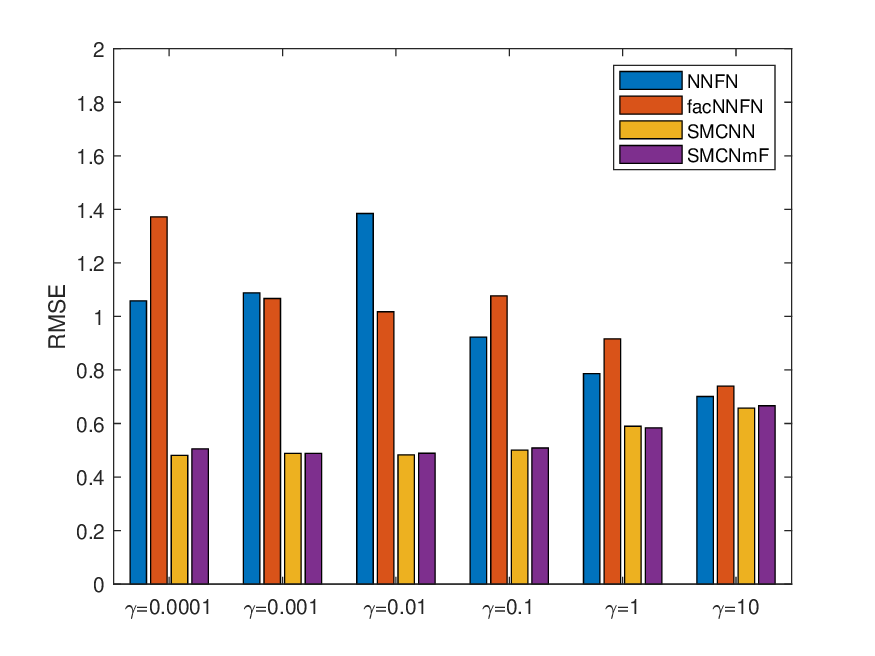}}
      \subfigure{
   \includegraphics[width = 0.18\columnwidth, trim=15 0 30 10, clip]{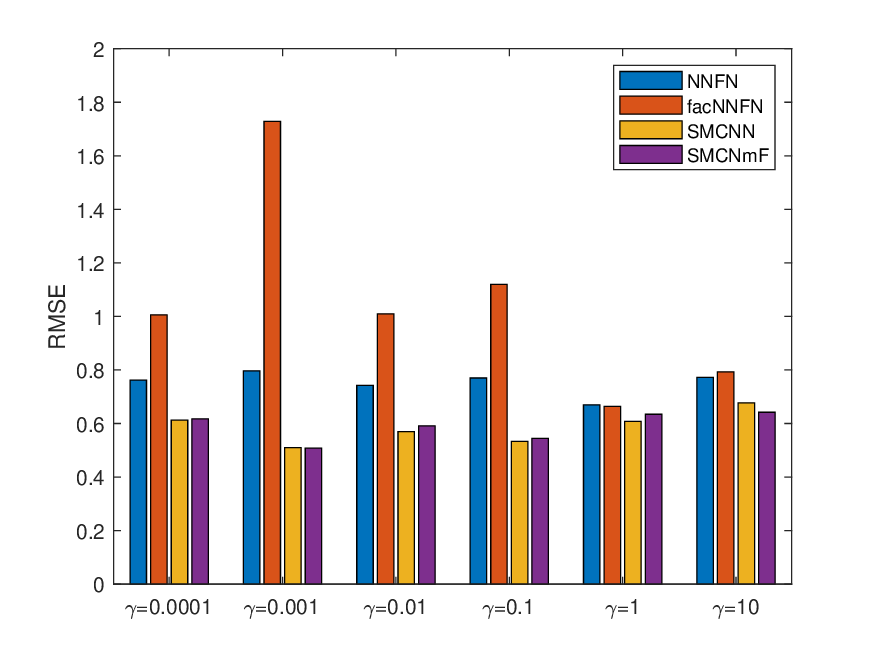}}
      \subfigure{ 
   \includegraphics[width = 0.18\columnwidth, trim=15 0 30 10, clip]{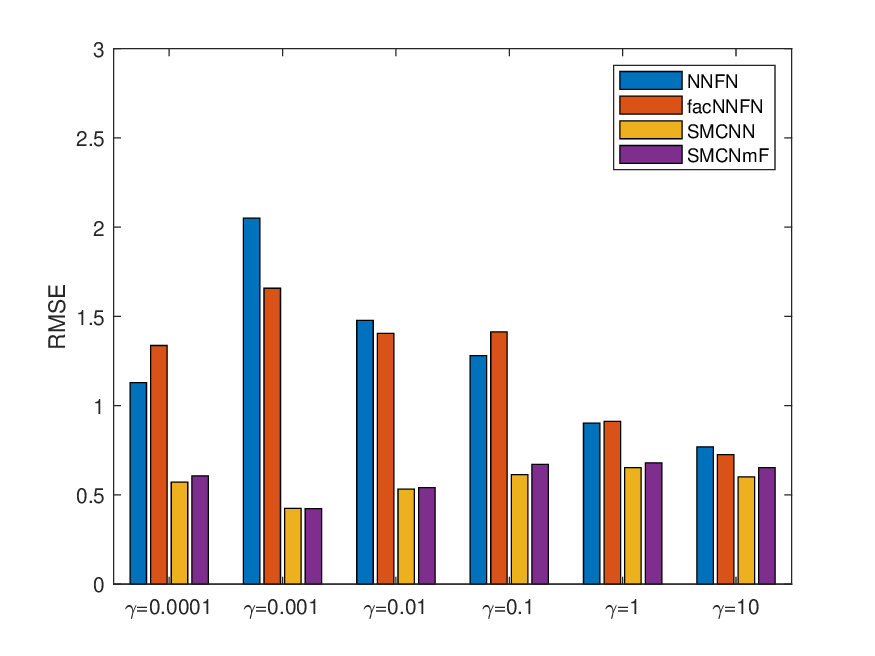}}
   \subfigure{ 
   \includegraphics[width = 0.18\columnwidth, trim=15 0 30 10, clip]{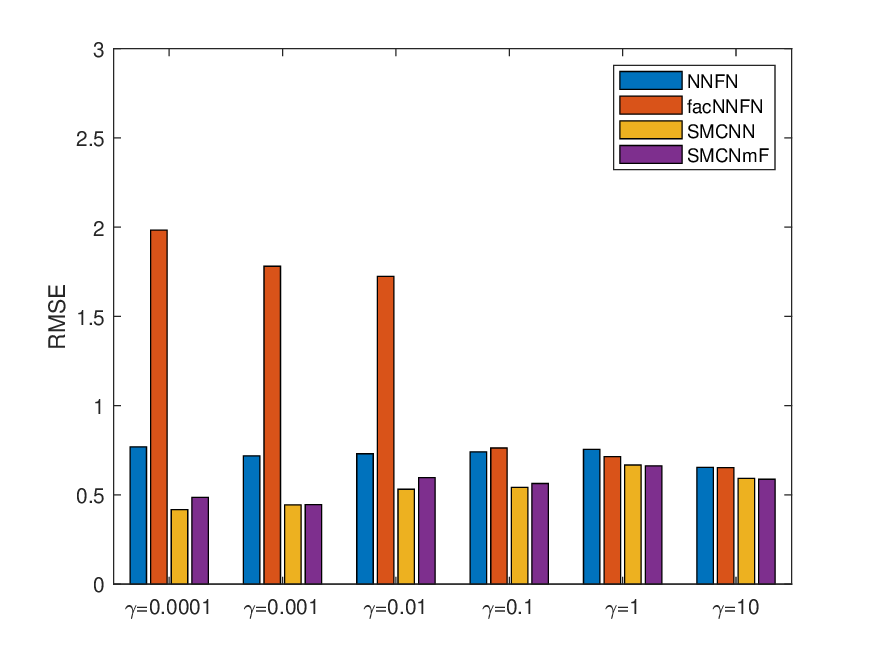}}
   \vspace{-0.5cm}  
      \subfigure[(a) ImageNet]{
   \includegraphics[width = 0.18\columnwidth, trim=15 0 30 10, clip]{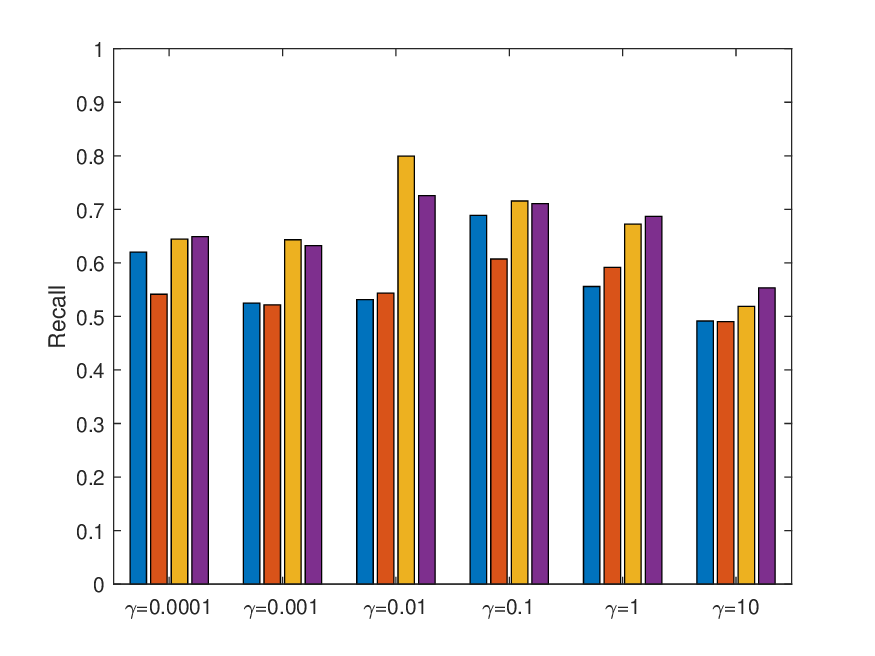}}
      \subfigure[(b) MNIST]{
   \includegraphics[width = 0.18\columnwidth, trim=15 0 30 10, clip]{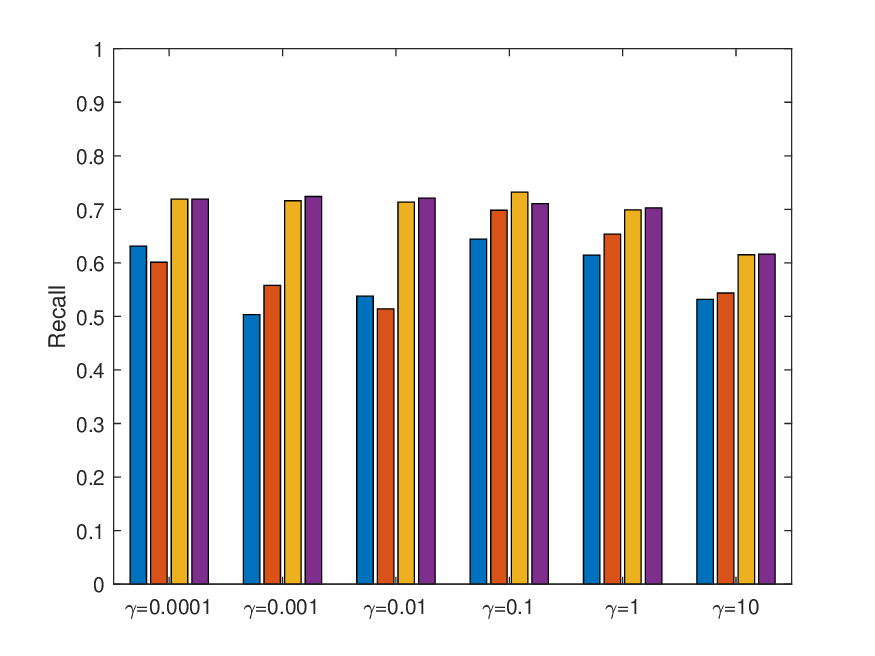}}
      \subfigure[(c) PROTEIN]{
   \includegraphics[width = 0.18\columnwidth, trim=15 0 30 10, clip]{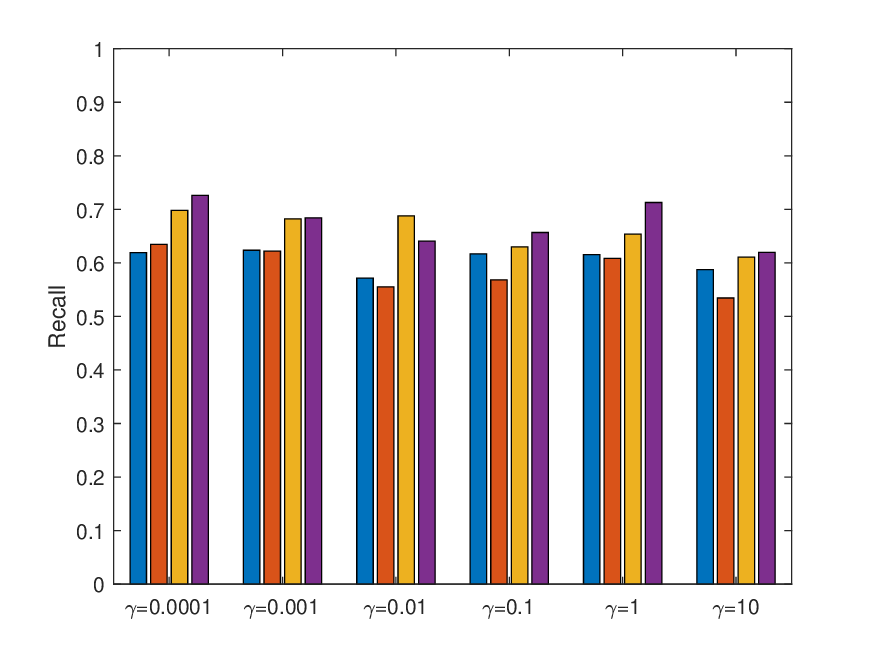}}
      \subfigure[(d) CIFAR]{ 
   \includegraphics[width = 0.18\columnwidth, trim=15 0 30 10, clip]{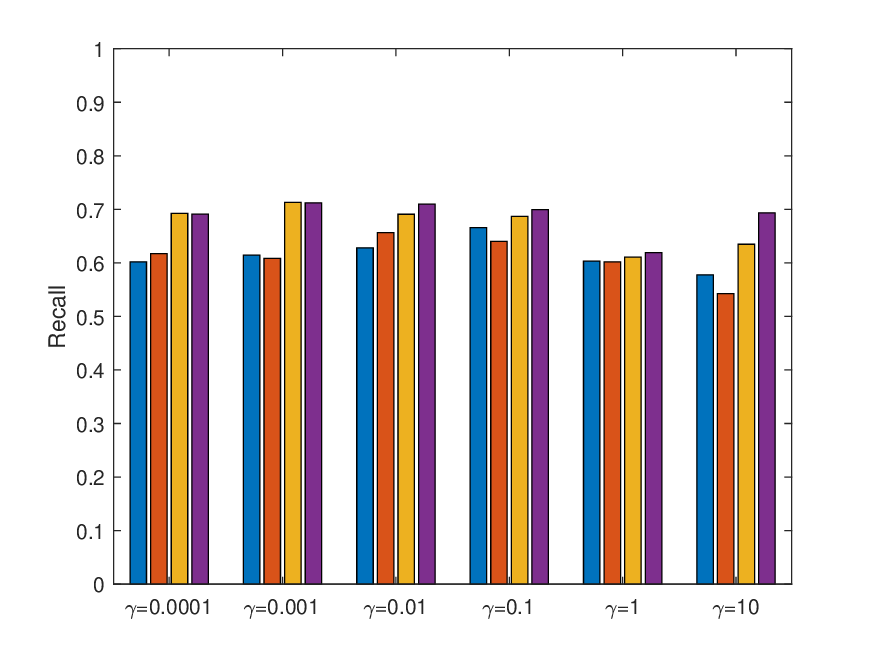}}
    \subfigure[(e) GoogleNews]{ 
   \includegraphics[width = 0.18\columnwidth, trim=15 0 30 10, clip]{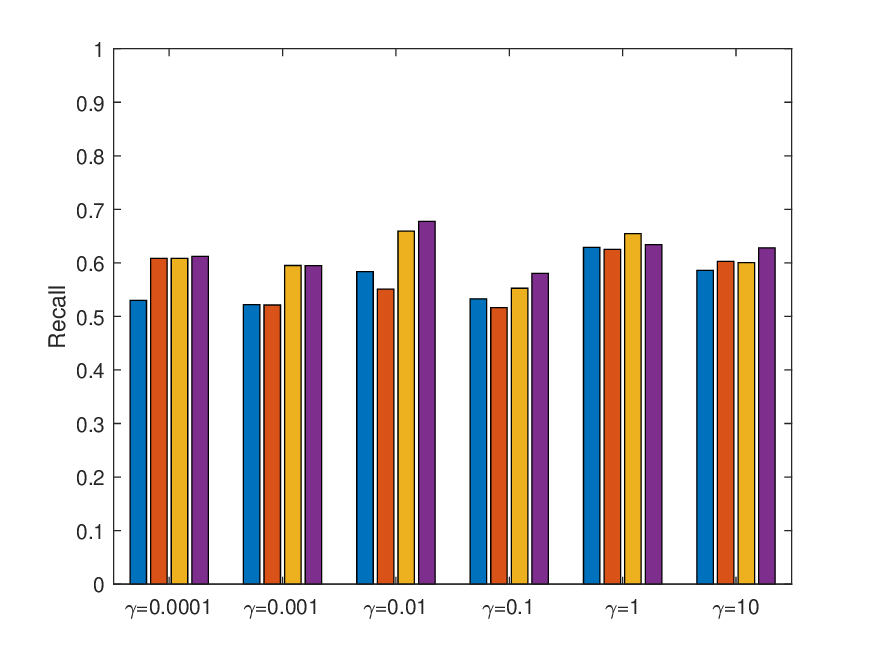}}
   \caption{RMSE/Recall@top20\% versus $\lambda$ on various dataset, with $m=1,000$ search candidates and $n=200$ query items, where missing ratio $\rho=0.8$, rank $r=100$, $\gamma = 0.001$, iterations $T=10,000$.}
    \label{fig:lambda:R08}
\end{figure*}

\begin{figure*}[!htb]
    \centering  
      \subfigure{
   \includegraphics[width = 0.18\columnwidth, trim=15 0 30 10, clip]{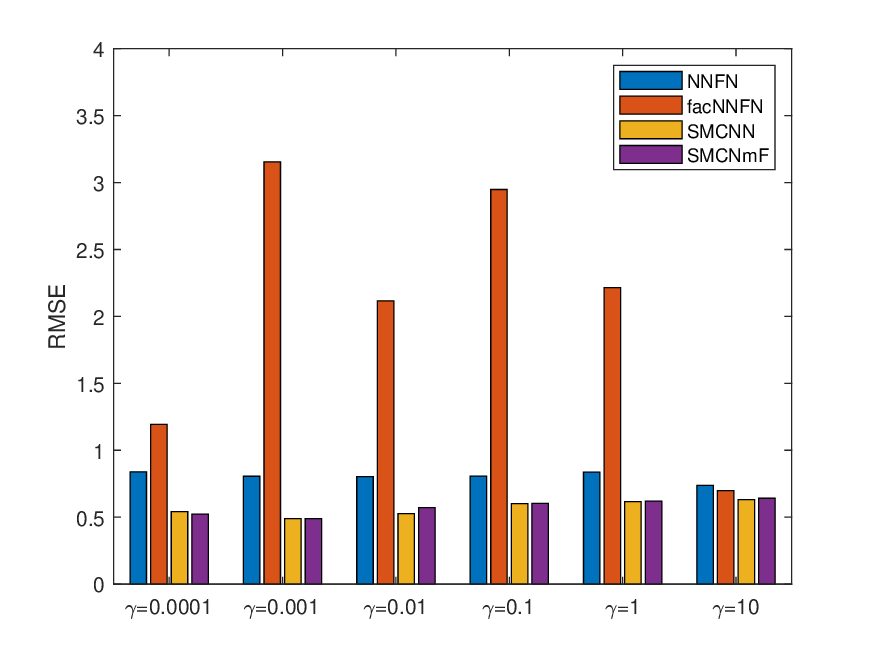}}
      \subfigure{
   \includegraphics[width = 0.18\columnwidth, trim=15 0 30 10, clip]{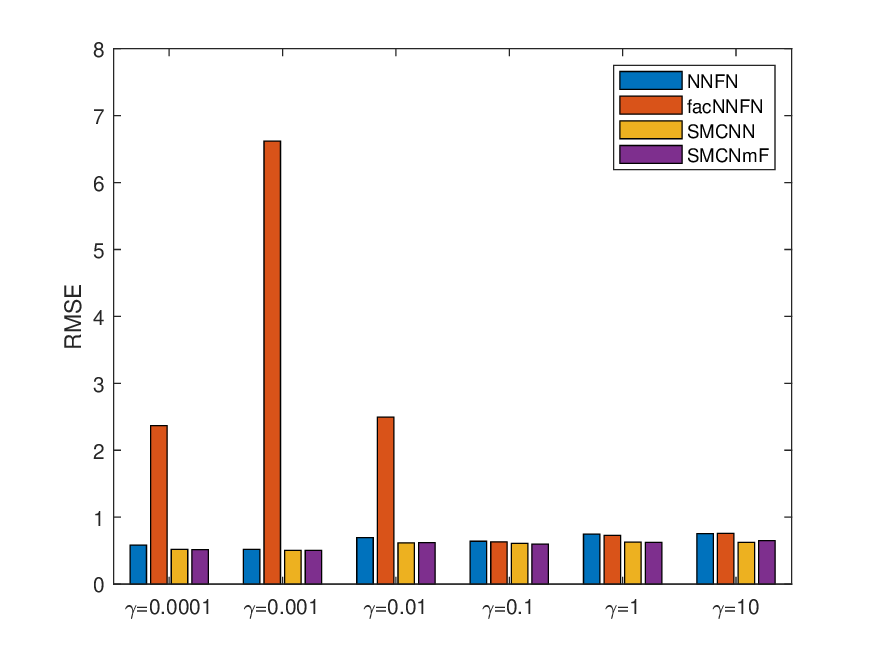}}
      \subfigure{
   \includegraphics[width = 0.18\columnwidth, trim=15 0 30 10, clip]{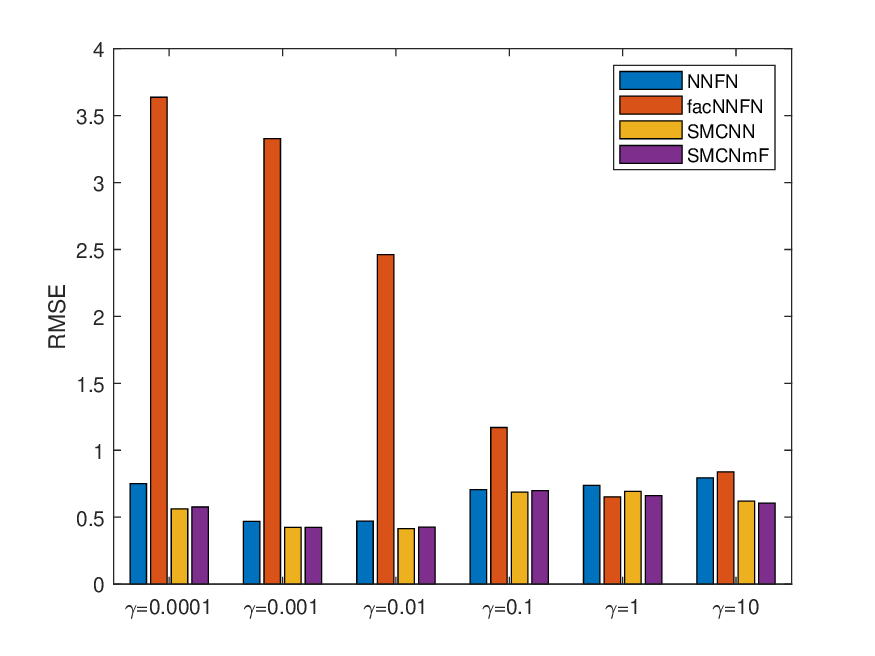}}
      \subfigure{ 
   \includegraphics[width = 0.18\columnwidth, trim=15 0 30 10, clip]{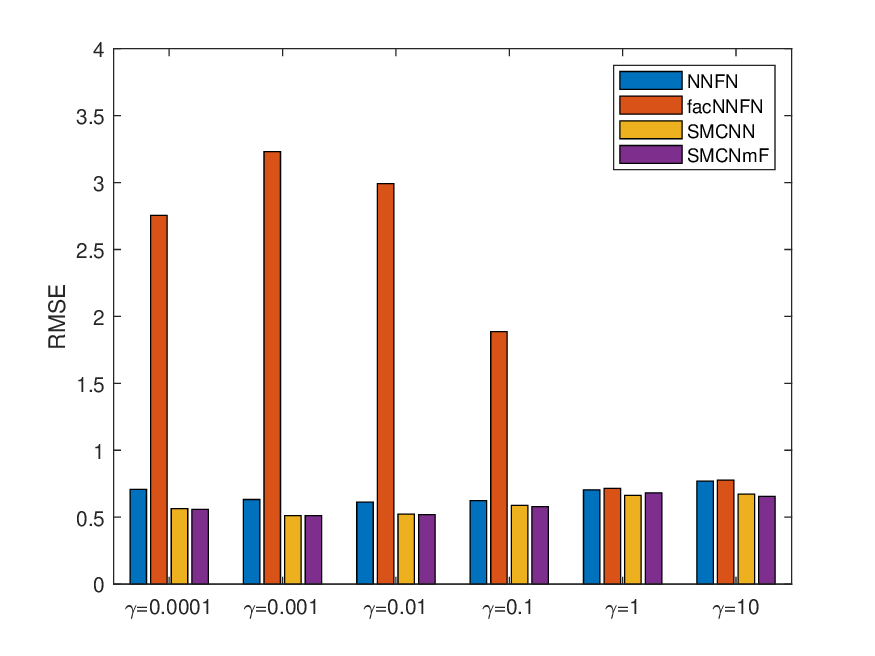}}
   \subfigure{ 
   \includegraphics[width = 0.18\columnwidth, trim=15 0 30 10, clip]{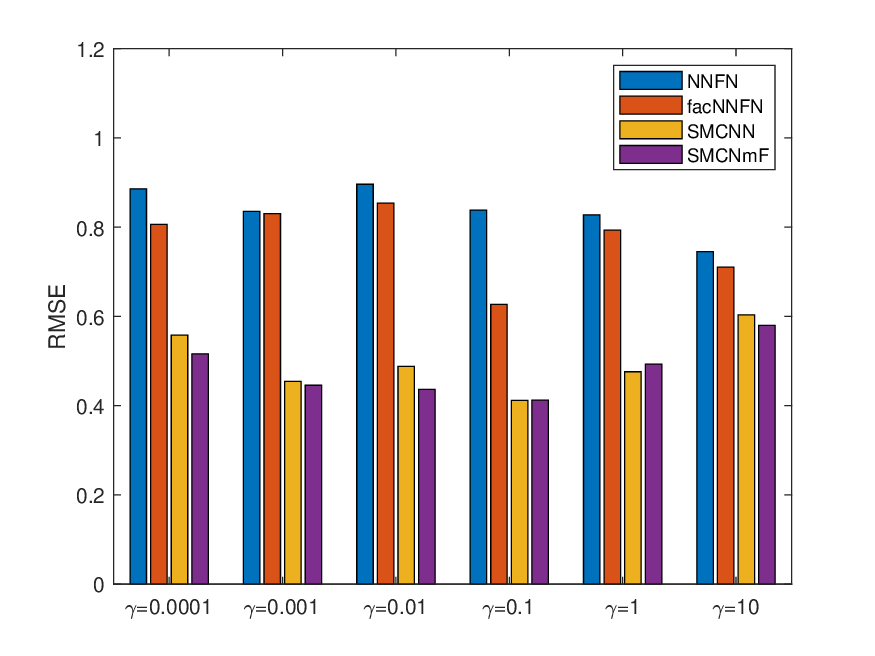}}
   \vspace{-0.5cm}
   
      \subfigure[(a) ImageNet]{
   \includegraphics[width = 0.18\columnwidth, trim=15 0 30 10, clip]{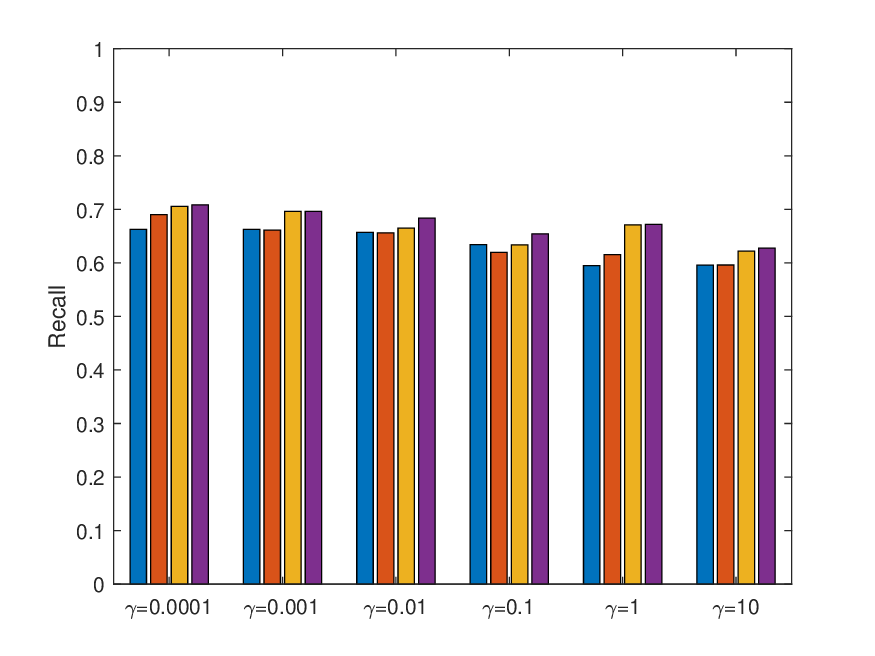}}
      \subfigure[(b) MNIST]{
   \includegraphics[width = 0.18\columnwidth, trim=15 0 30 10, clip]{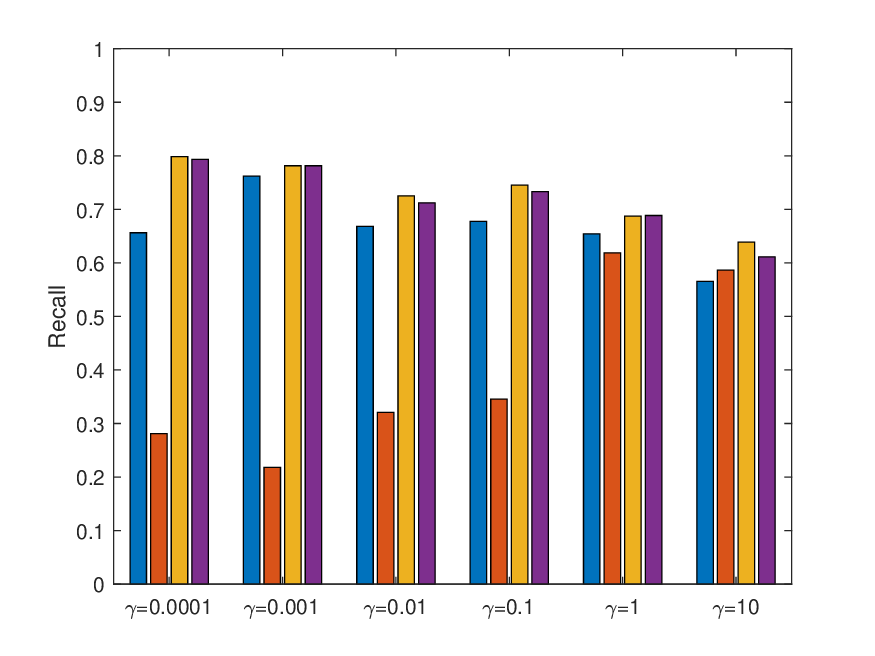}}
      \subfigure[(c) PROTEIN]{
   \includegraphics[width = 0.18\columnwidth, trim=15 0 30 10, clip]{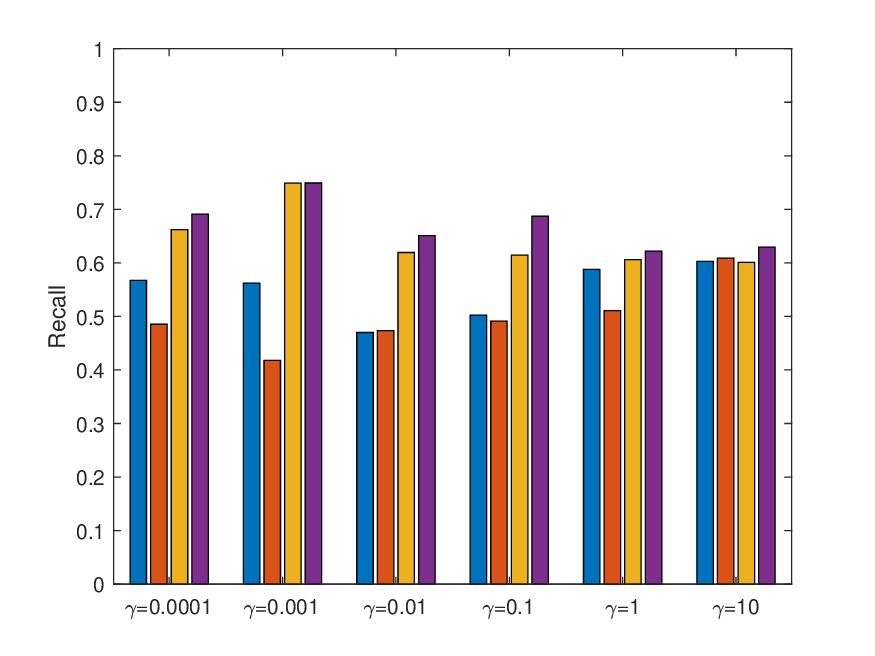}}
      \subfigure[(d) CIFAR]{ 
   \includegraphics[width = 0.18\columnwidth, trim=15 0 30 10, clip]{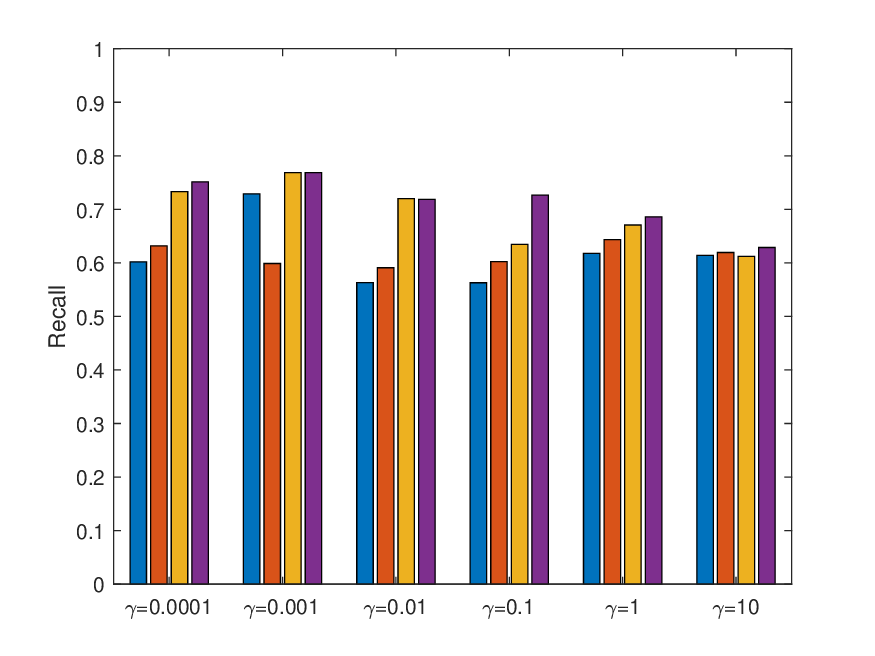}}
    \subfigure[(e) GoogleNews]{ 
   \includegraphics[width = 0.18\columnwidth, trim=15 0 30 10, clip]{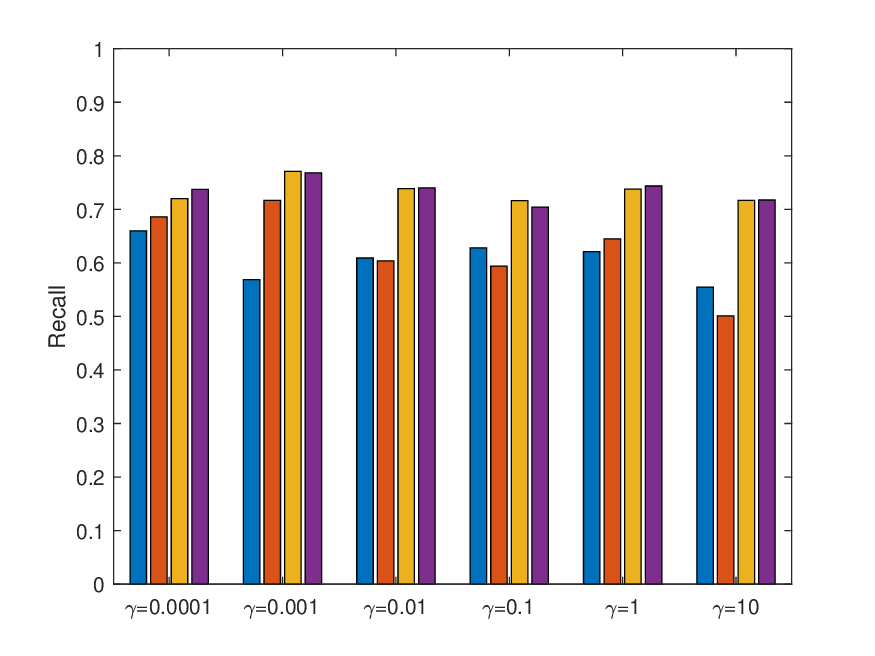}}
   \caption{RMSE/Recall@top20\% versus $\lambda$ on various dataset, with $m=1,000$ search candidates and $n=200$ query items, where missing ratio $\rho=0.9$, rank $r=100$, $\gamma = 0.001$, iterations $T=10,000$.}
    \label{fig:lambda:R09}
\end{figure*}

%-------------------------------------------------------------------------
\clearpage
\subsubsection{Stepsize $\gamma$}
%\gamma  
Fig.\ref{fig:gamma:R07}-Fig.\ref{fig:gamma:R09} show the RMSE/Recall varying with stepsize $\gamma$ on various datasets with various missing ratio $\rho$, fixed rank $r=100$, and fixed $\lambda = 0.001$.  Similar to the analysis about $\lambda$, the RMSE of SMCNN/SMCNmF shows different performances on different datasets. Specifically, we find that both SMCNN/SMCNmF perform stably on MNIST and PROTEIN datasets with a fixed $\lambda$ and various $\gamma$. However, the performance is not such stable on the ImageNet and CIFAR datasets. The corresponding recall shows the same tendencies. 

\begin{figure*}[!htb]
    \centering  
      \subfigure{
   \includegraphics[width = 0.18\columnwidth, trim=15 0 30 10, clip]{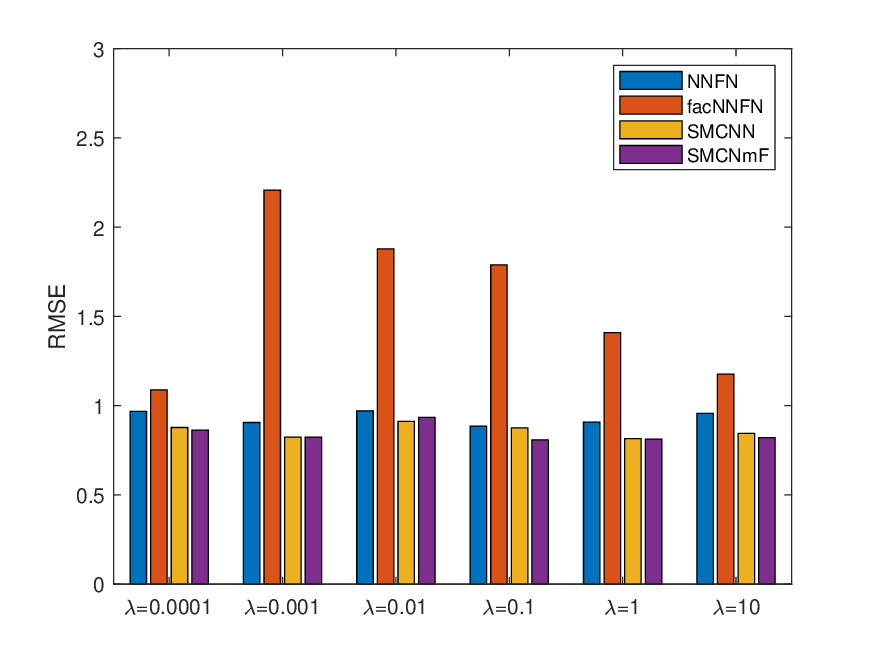}}
      \subfigure{
   \includegraphics[width = 0.18\columnwidth, trim=15 0 30 10, clip]{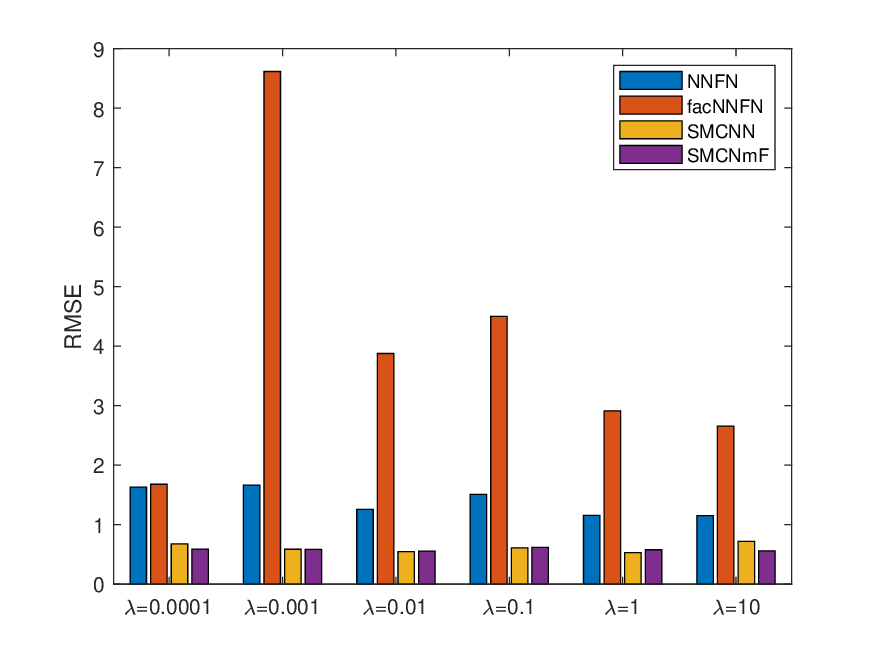}}
      \subfigure{
   \includegraphics[width = 0.18\columnwidth, trim=15 0 30 10, clip]{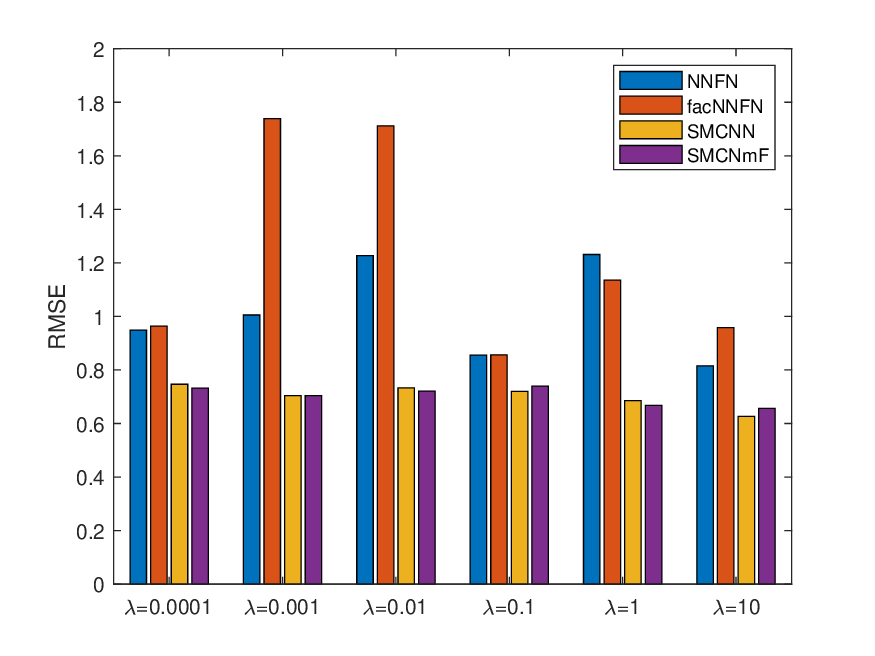}}
      \subfigure{ 
   \includegraphics[width = 0.18\columnwidth, trim=15 0 30 10, clip]{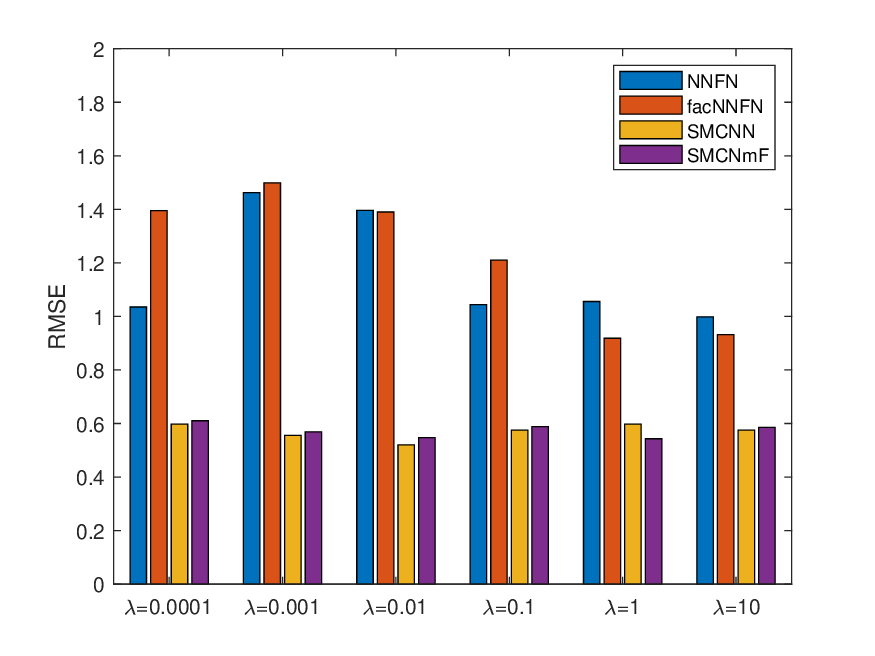}}
   \subfigure{ 
   \includegraphics[width = 0.18\columnwidth, trim=15 0 30 10, clip]{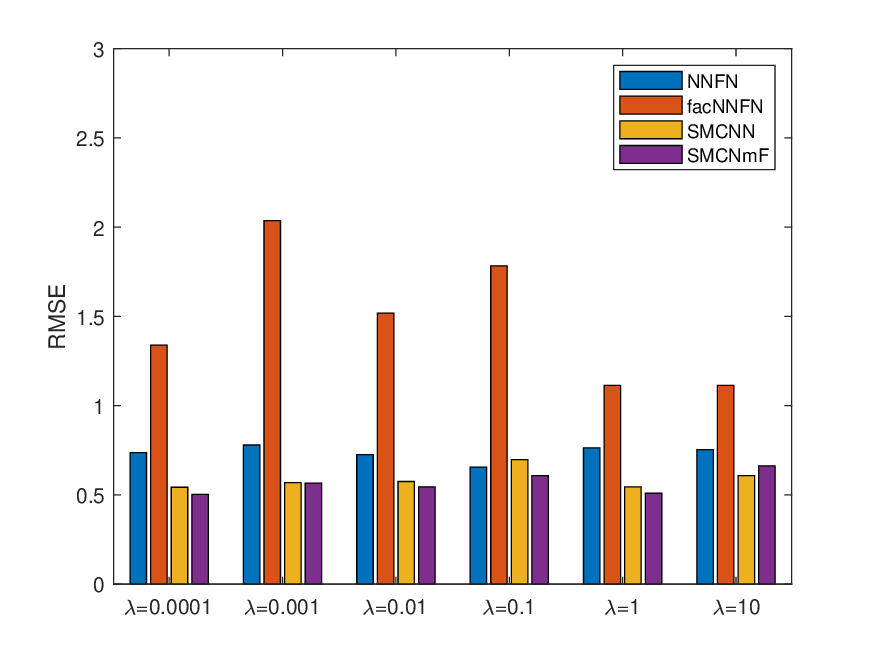}}
   \vspace{-0.5cm}
   
      \subfigure[(a) ImageNet]{
   \includegraphics[width = 0.18\columnwidth, trim=15 0 30 10, clip]{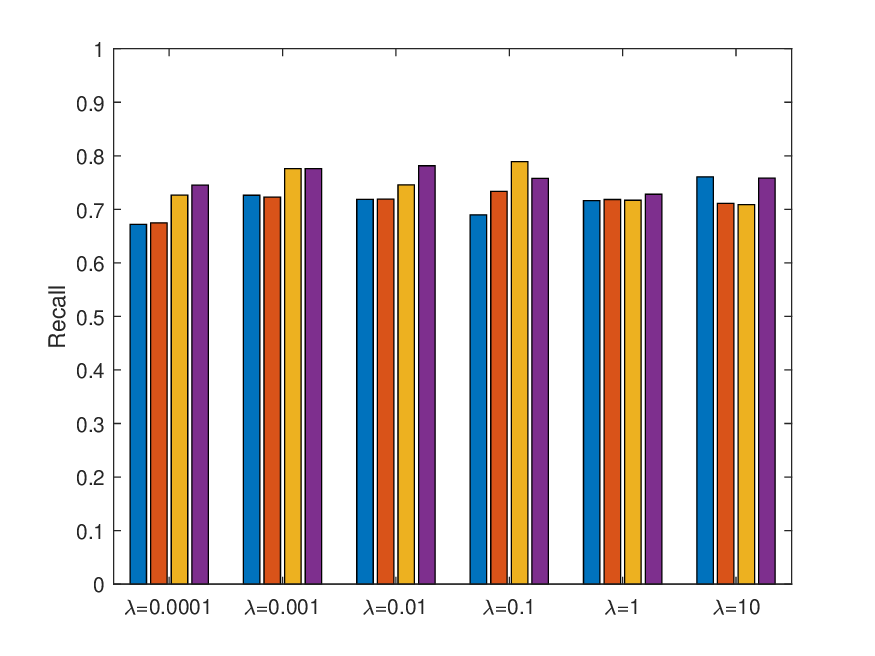}}
      \subfigure[(b) MNIST]{
   \includegraphics[width = 0.18\columnwidth, trim=15 0 30 10, clip]{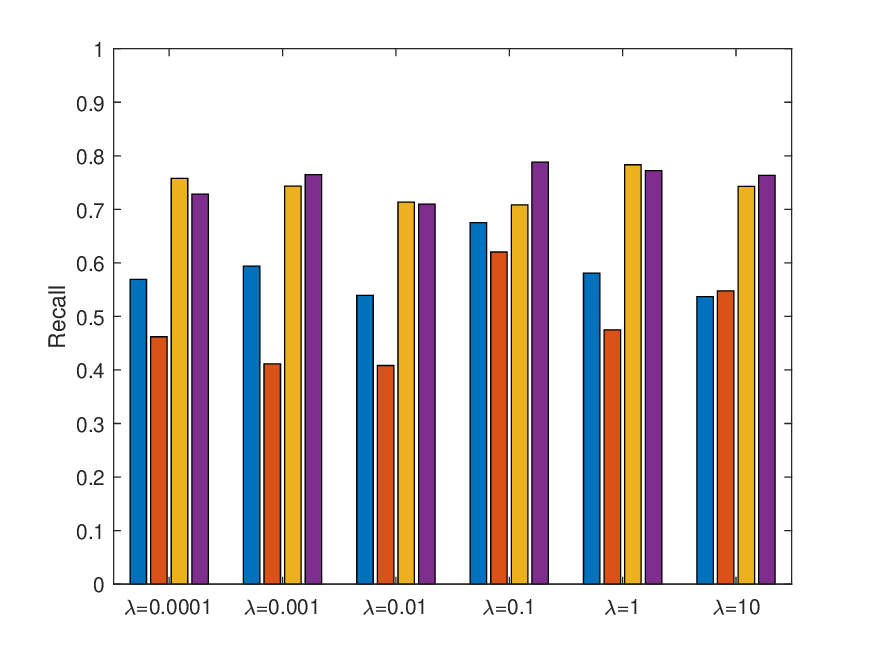}}
      \subfigure[(c) PROTEIN]{
   \includegraphics[width = 0.18\columnwidth, trim=15 0 30 10, clip]{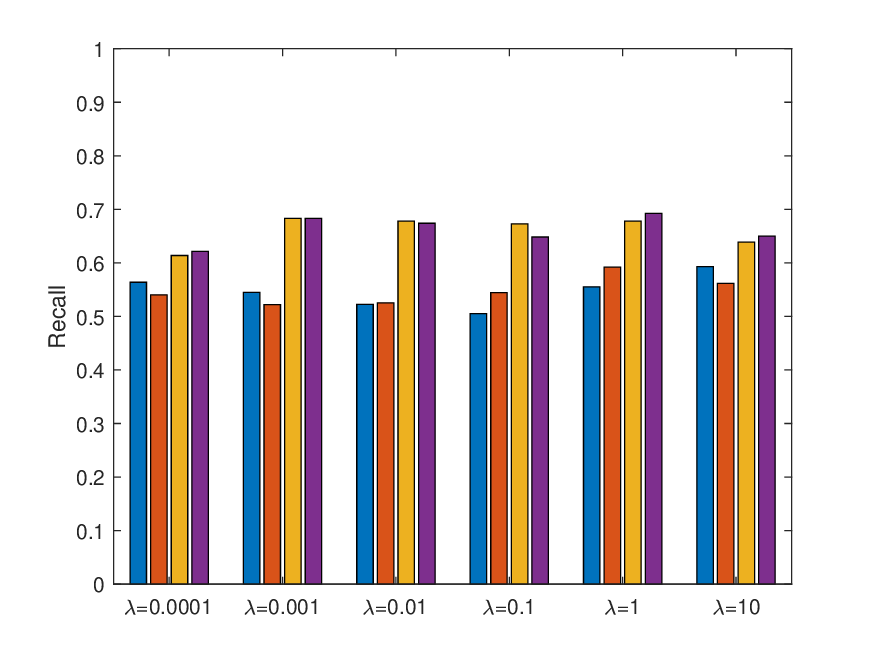}}
      \subfigure[(d) CIFAR]{ 
   \includegraphics[width = 0.18\columnwidth, trim=15 0 30 10, clip]{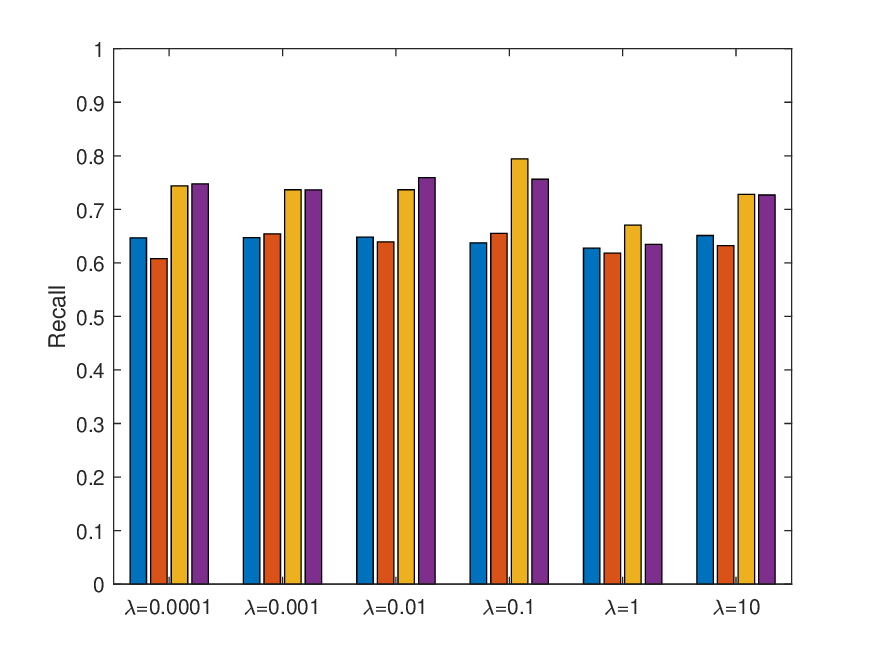}}
    \subfigure[(e) GoogleNews]{ 
   \includegraphics[width = 0.18\columnwidth, trim=15 0 30 10, clip]{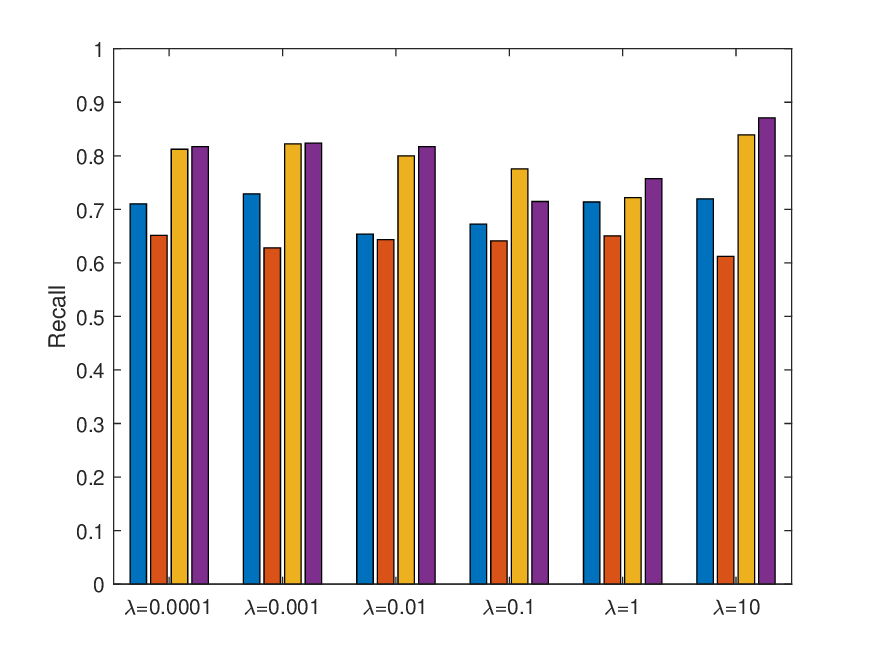}}
   \caption{RMSE/Recall@top20\% versus $\lambda$ on various dataset, with $m=1,000$ search candidates and $n=200$ query items, where missing ratio $\rho=0.7$, rank $r=100$, $\gamma = 0.001$, iterations $T=10,000$.}
    \label{fig:gamma:R07}
\end{figure*}

\begin{figure*}[!htb]
    \centering  
      \subfigure{
   \includegraphics[width = 0.18\columnwidth, trim=15 0 30 10, clip]{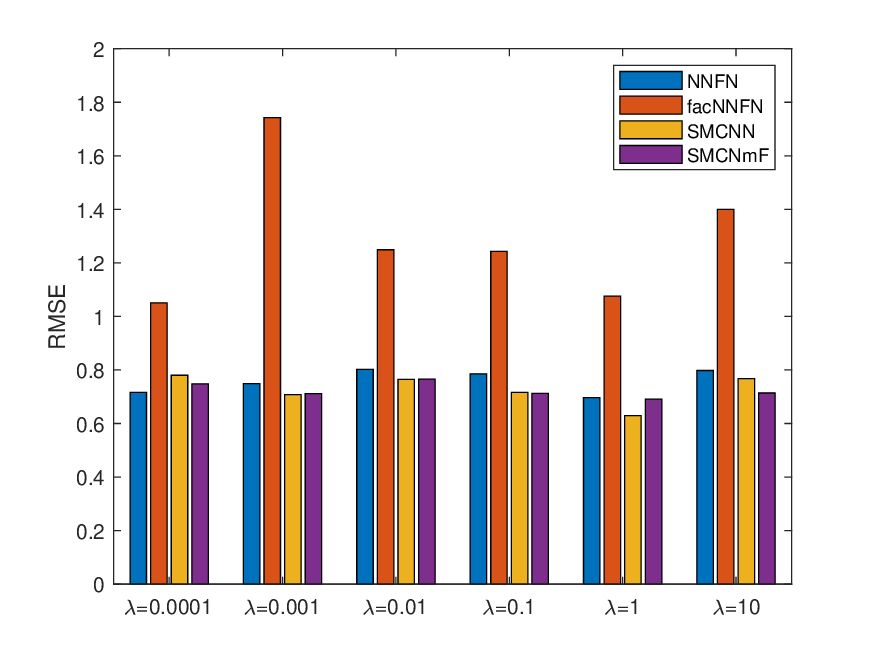}}
      \subfigure{
   \includegraphics[width = 0.18\columnwidth, trim=15 0 30 10, clip]{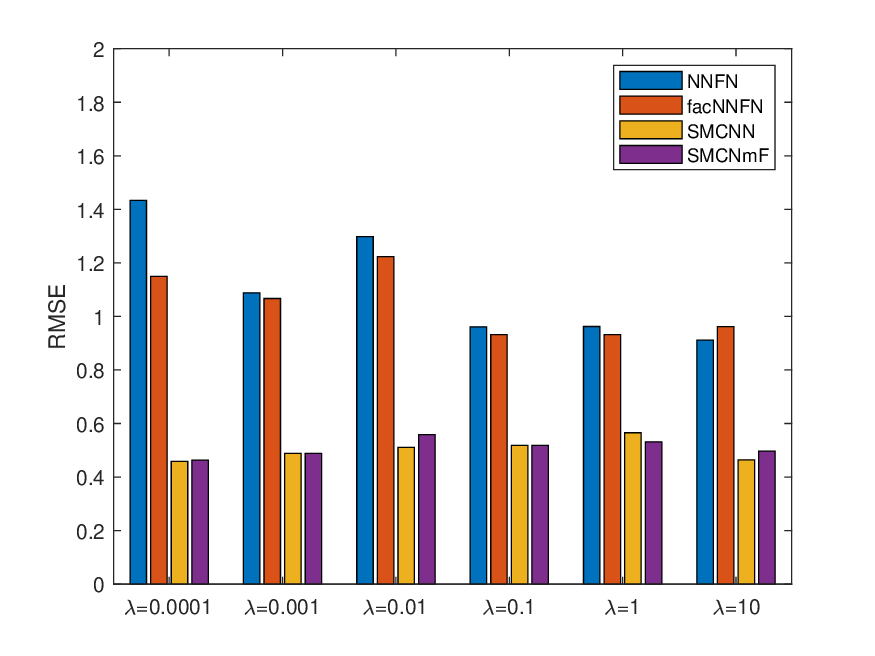}}
      \subfigure{
   \includegraphics[width = 0.18\columnwidth, trim=15 0 30 10, clip]{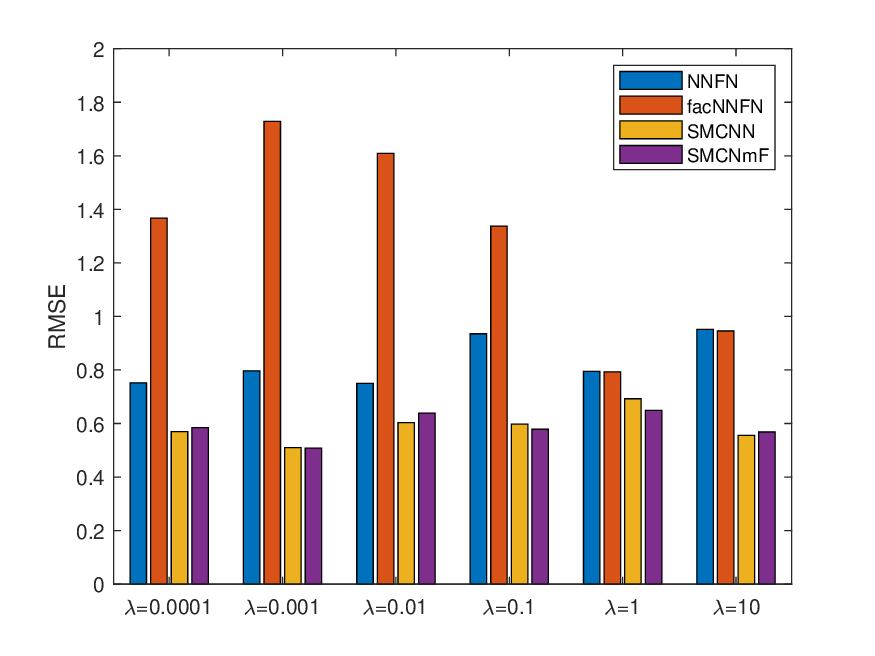}}
      \subfigure{ 
   \includegraphics[width = 0.18\columnwidth, trim=15 0 30 10, clip]{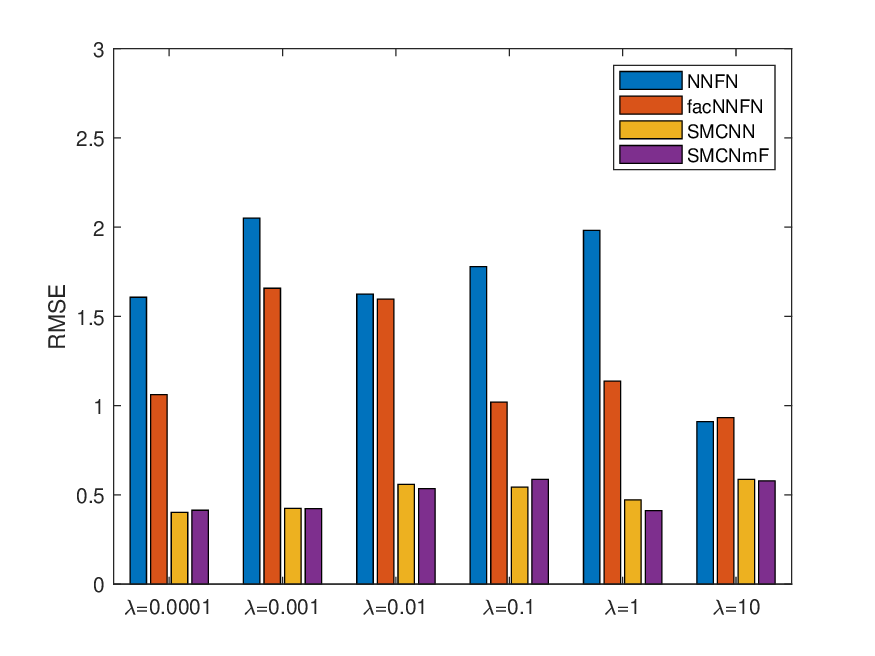}}
   \subfigure{ 
   \includegraphics[width = 0.18\columnwidth, trim=15 0 30 10, clip]{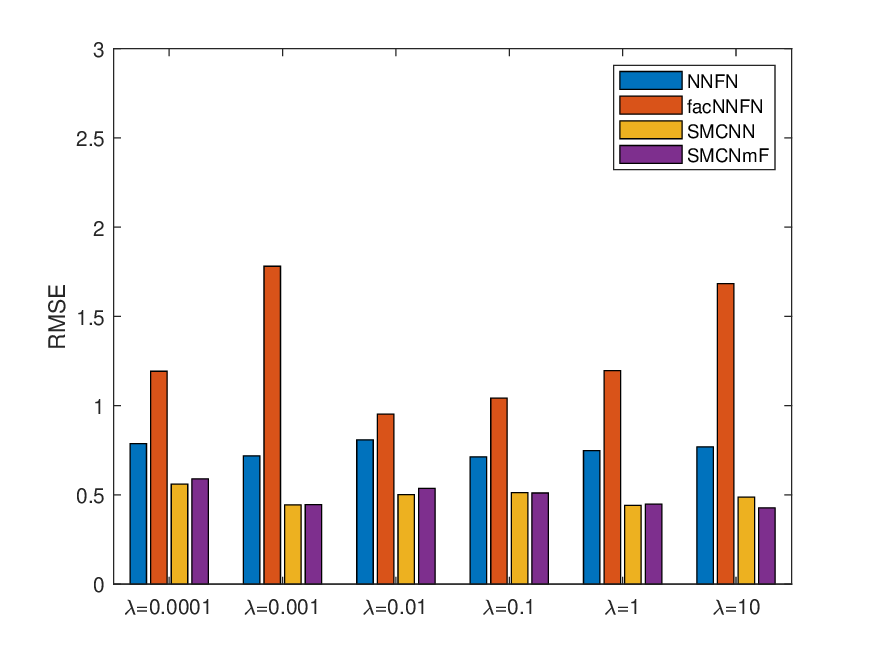}}
   \vspace{-0.5cm}
   
      \subfigure[(a) ImageNet]{
   \includegraphics[width = 0.18\columnwidth, trim=15 0 30 10, clip]{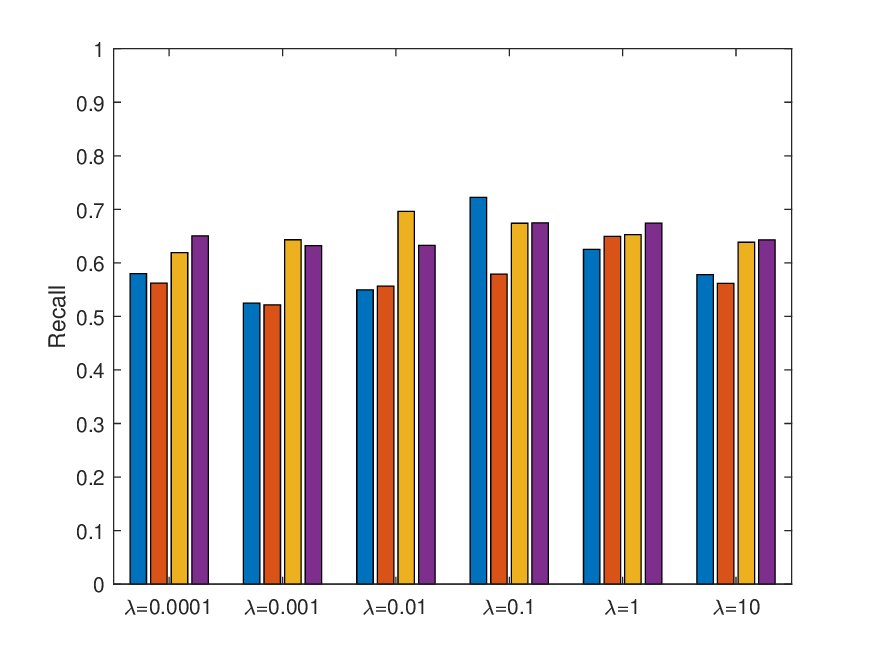}}
      \subfigure[(b) MNIST]{
   \includegraphics[width = 0.18\columnwidth, trim=15 0 30 10, clip]{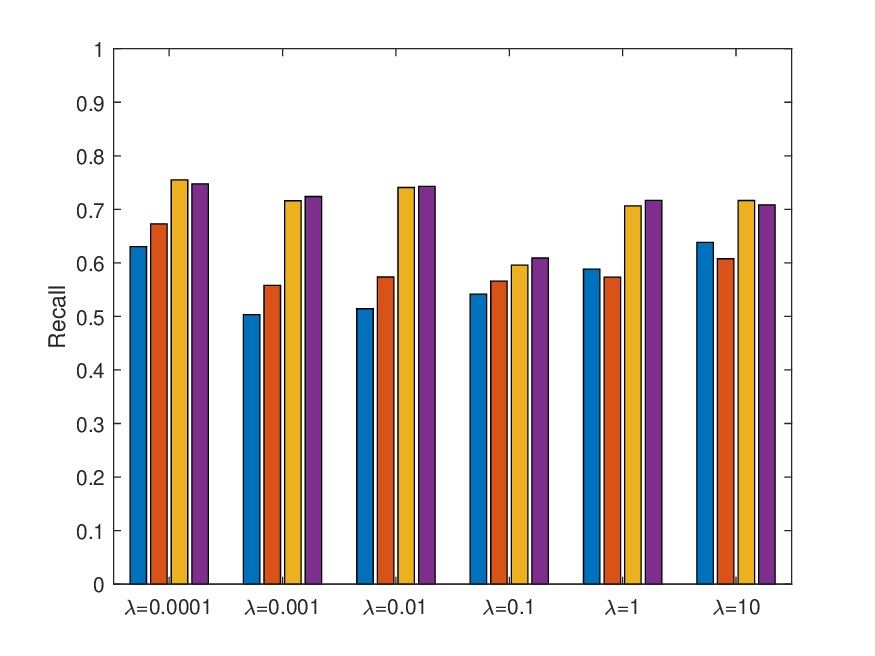}}
      \subfigure[(c) PROTEIN]{
   \includegraphics[width = 0.18\columnwidth, trim=15 0 30 10, clip]{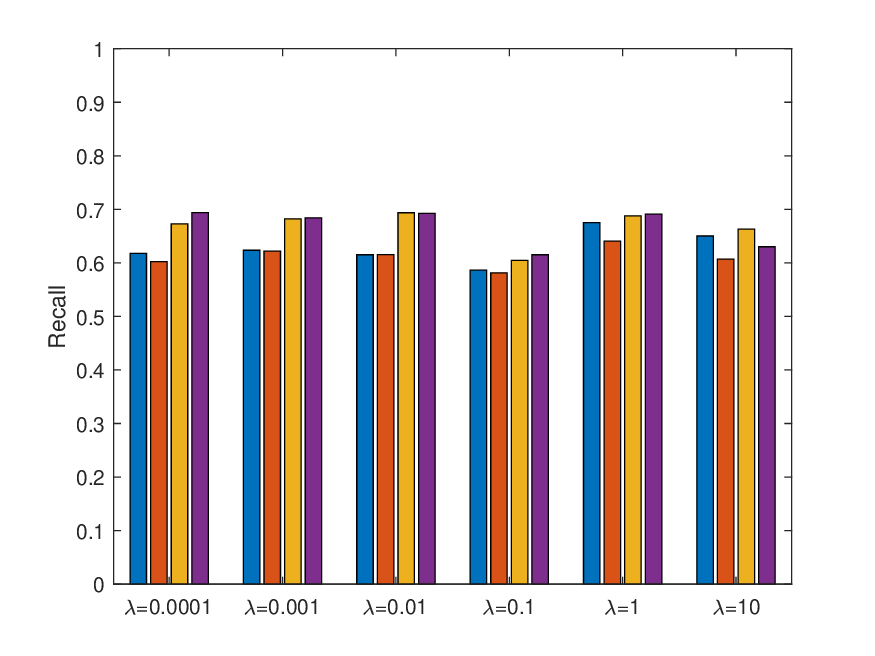}}
      \subfigure[(d) CIFAR]{ 
   \includegraphics[width = 0.18\columnwidth, trim=15 0 30 10, clip]{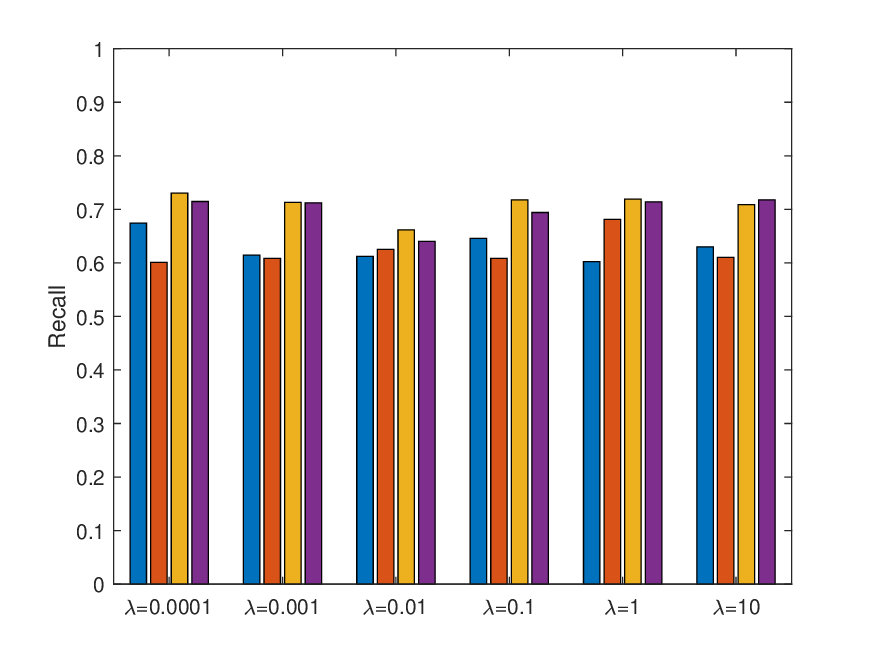}}
    \subfigure[(e) GoogleNews]{ 
   \includegraphics[width = 0.18\columnwidth, trim=15 0 30 10, clip]{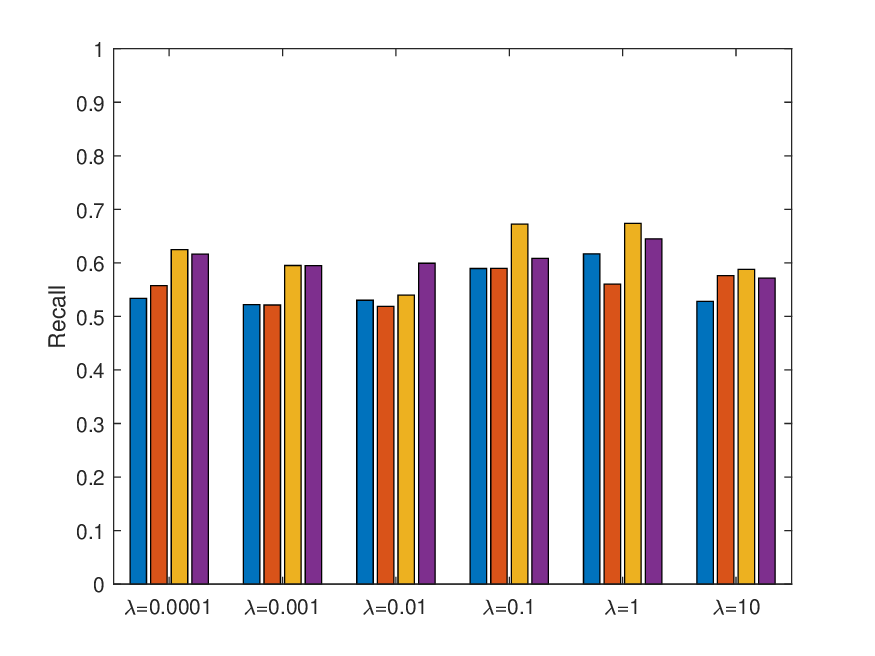}}
   \caption{RMSE/Recall@top20\% versus $\gamma$ on various dataset, with $m=1,000$ search candidates and $n=200$ query items, where missing ratio $\rho=0.8$, rank $r=100$, $\lambda = 0.001$, iterations $T=10,000$.}
    \label{fig:gamma:R08}
\end{figure*}

\begin{figure*}[!htb]
    \centering  
      \subfigure{
   \includegraphics[width = 0.18\columnwidth, trim=15 0 30 10, clip]{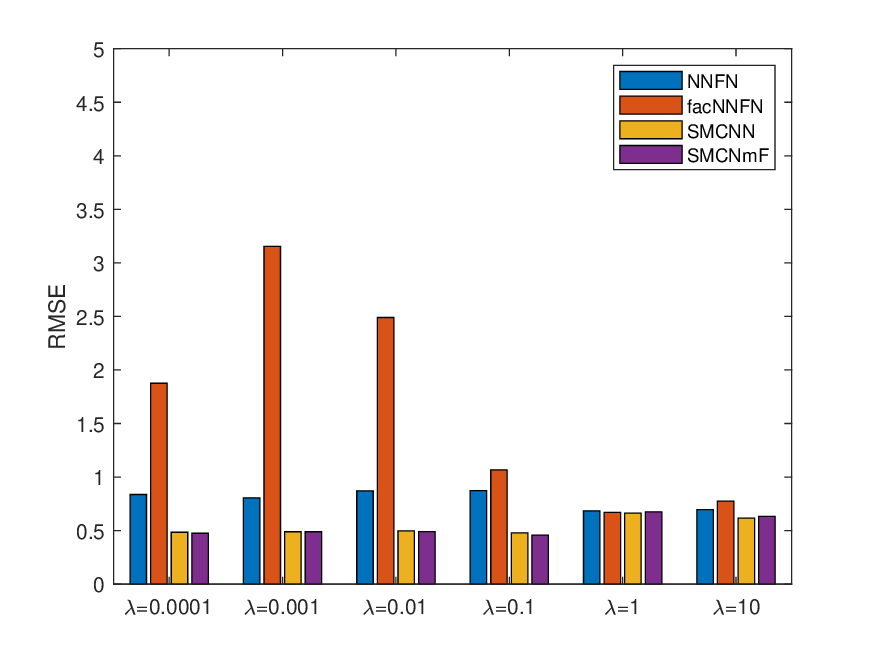}}
      \subfigure{
   \includegraphics[width = 0.18\columnwidth, trim=15 0 30 10, clip]{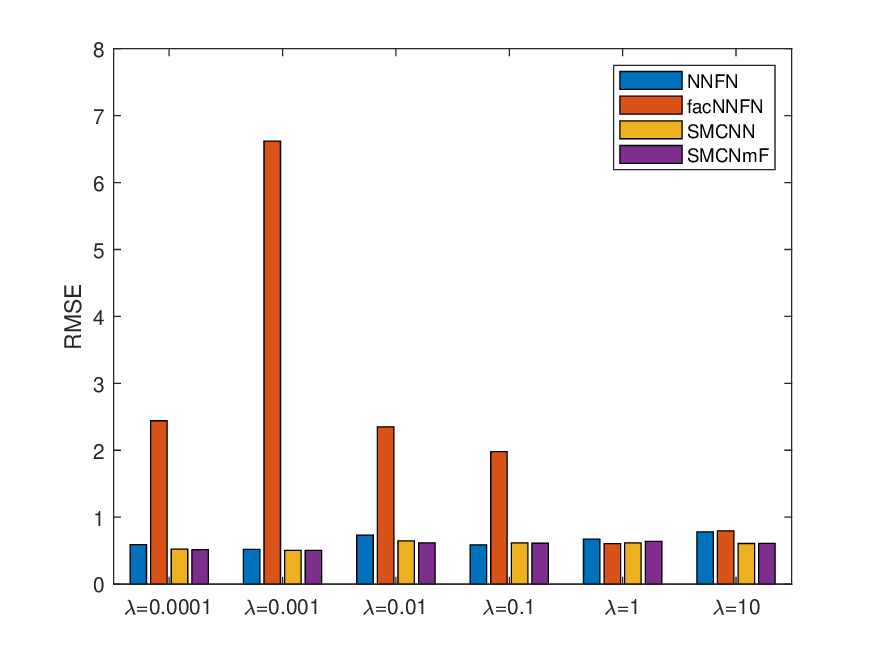}}
      \subfigure{
   \includegraphics[width = 0.18\columnwidth, trim=15 0 30 10, clip]{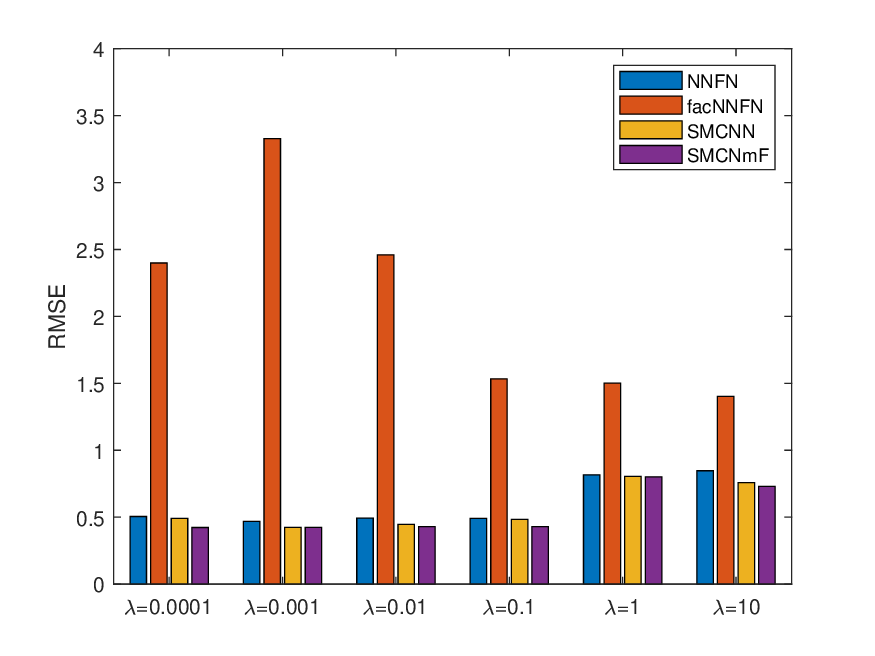}}
      \subfigure{ 
   \includegraphics[width = 0.18\columnwidth, trim=15 0 30 10, clip]{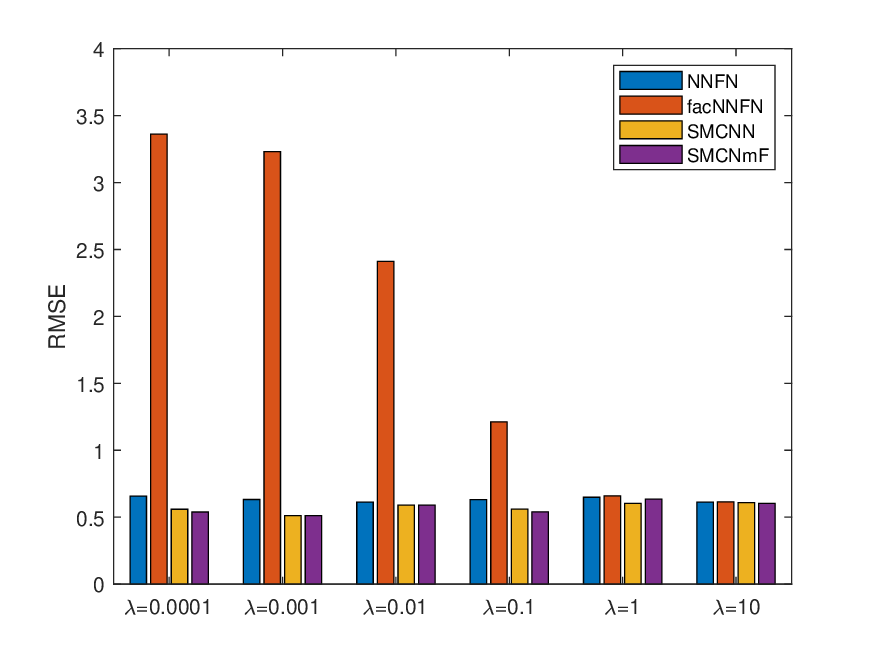}}
   \subfigure{ 
   \includegraphics[width = 0.18\columnwidth, trim=15 0 30 10, clip]{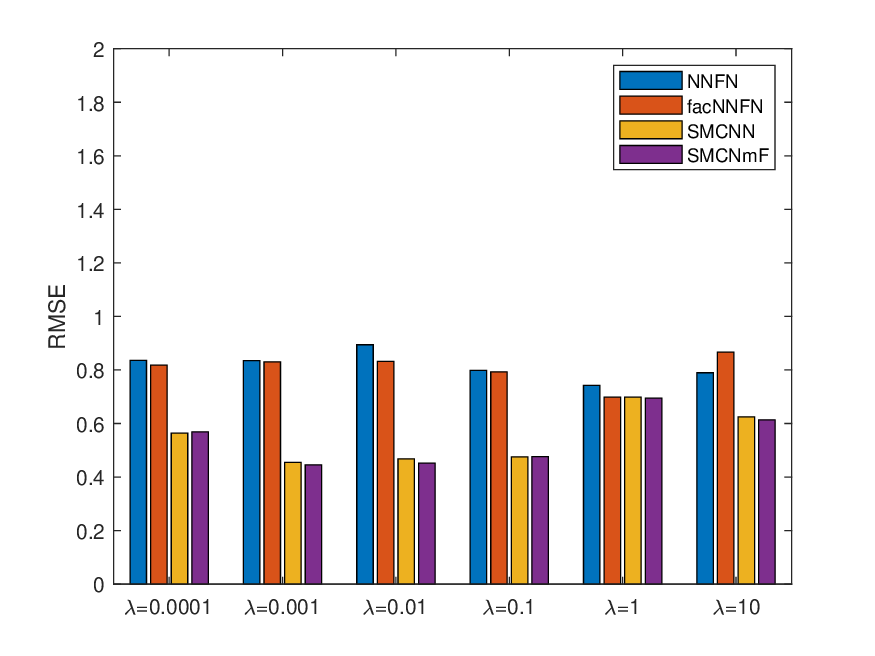}}
   \vspace{-0.5cm}
   
      \subfigure[(a) ImageNet]{
   \includegraphics[width = 0.18\columnwidth, trim=15 0 30 10, clip]{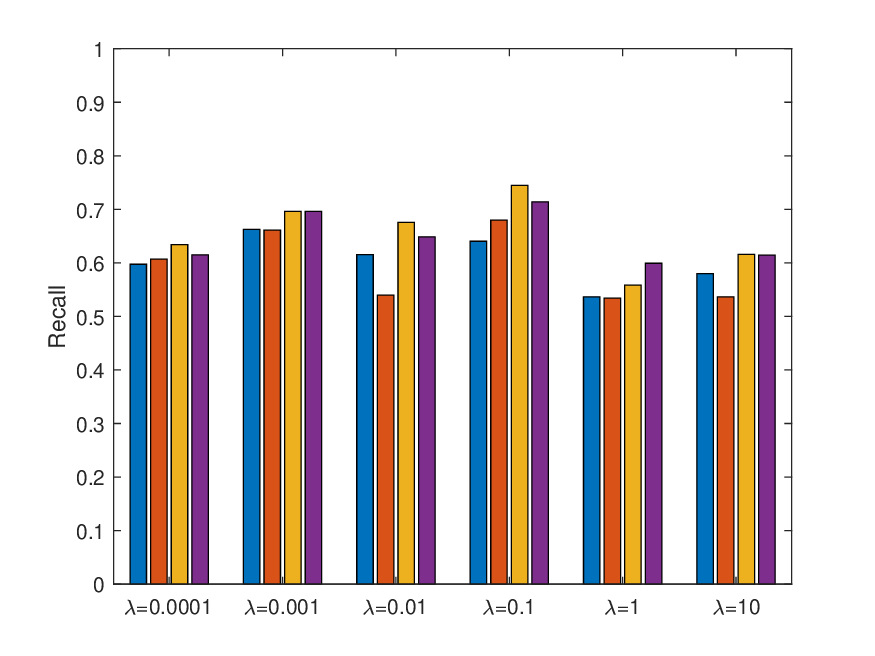}}
      \subfigure[(b) MNIST]{
   \includegraphics[width = 0.18\columnwidth, trim=15 0 30 10, clip]{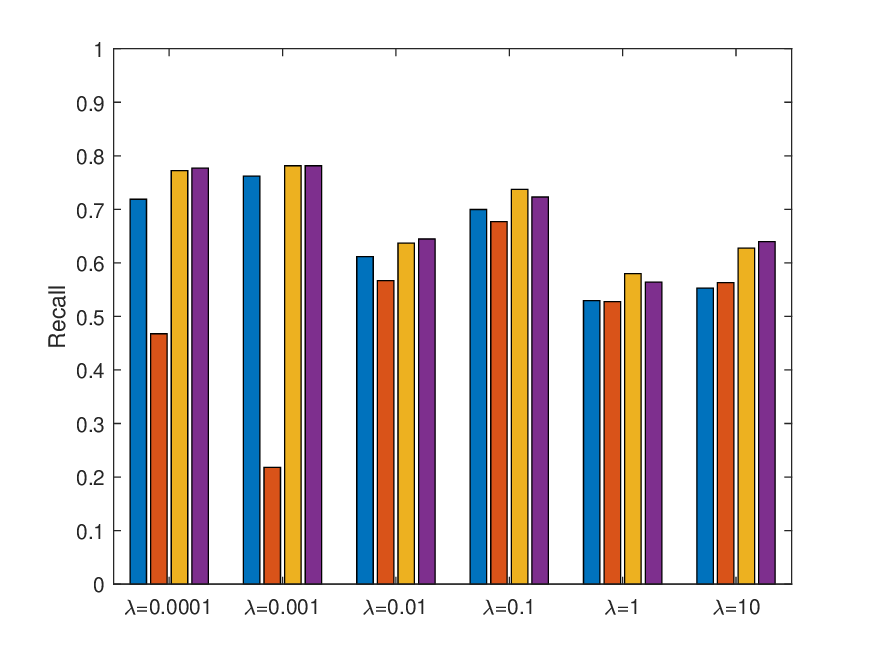}}
      \subfigure[(c) PROTEIN]{
   \includegraphics[width = 0.18\columnwidth, trim=15 0 30 10, clip]{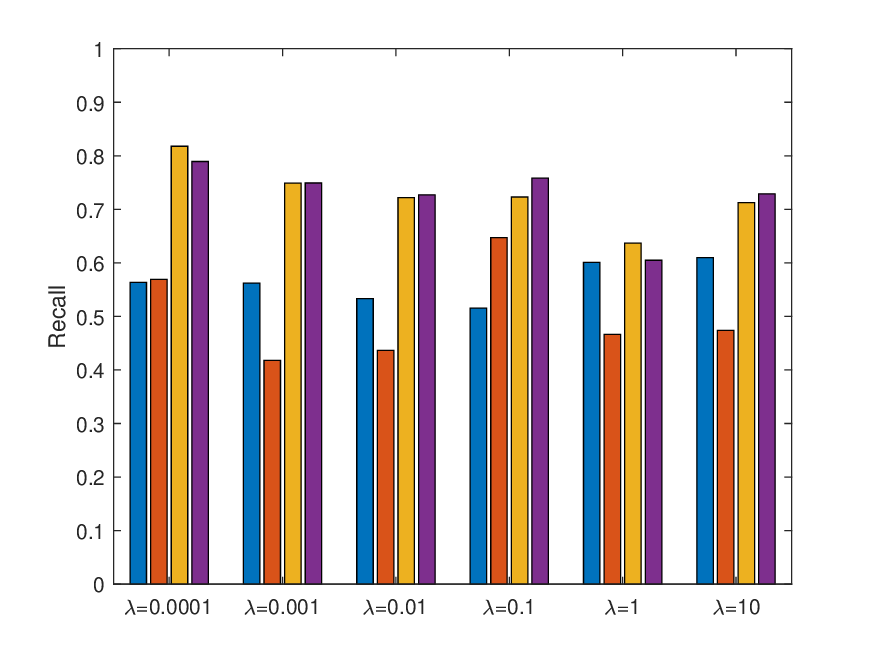}}
      \subfigure[(d) CIFAR]{ 
   \includegraphics[width = 0.18\columnwidth, trim=15 0 30 10, clip]{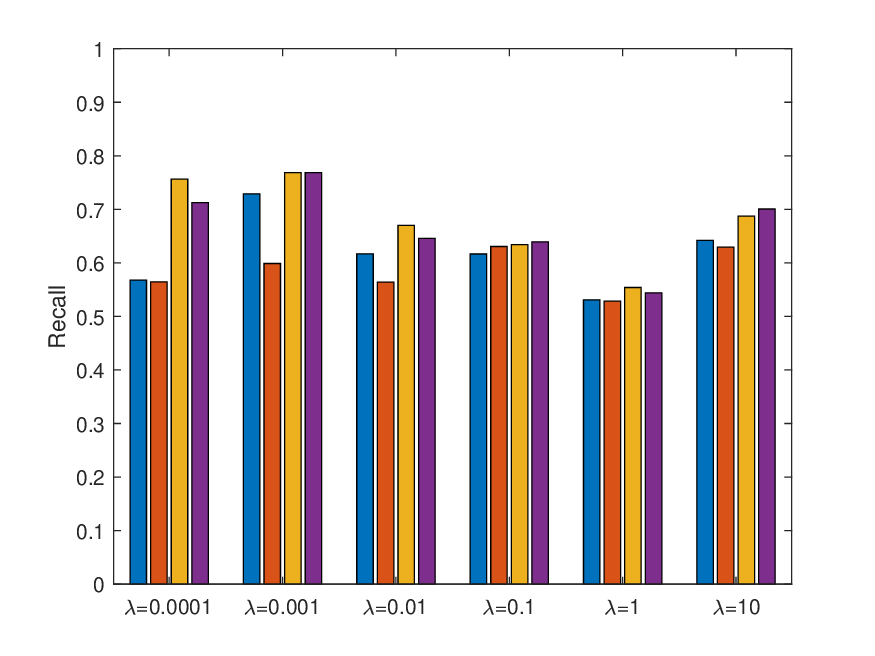}}
    \subfigure[(e) GoogleNews]{ 
   \includegraphics[width = 0.18\columnwidth, trim=15 0 30 10, clip]{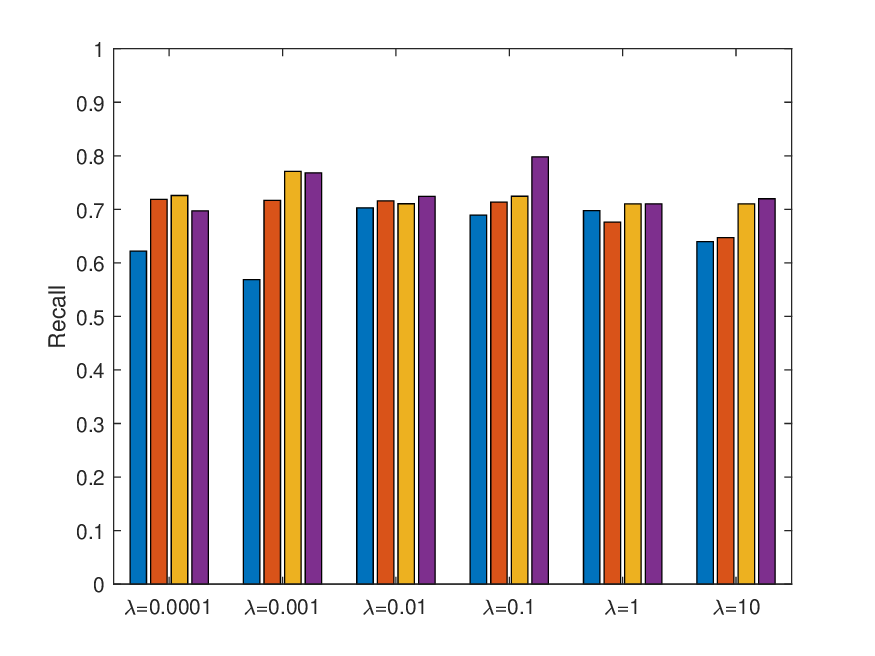}}
   \caption{RMSE/Recall@top20\% versus $\lambda$ on various dataset, with $m=1,000$ search candidates and $n=200$ query items, where missing ratio $\rho=0.9$, rank $r=100$, $\gamma = 0.001$, iterations $T=10,000$.}
    \label{fig:gamma:R09}
\end{figure*}

\clearpage 
%\input{gradient }

%\clearpage
%\input{rebuttal}

\end{document}